%% file: main.tex
\documentclass{article}
\usepackage[final]{colm2026_conference}

\usepackage{microtype}
\usepackage{hyperref}
\usepackage{url}
\usepackage{booktabs}
\usepackage{amsmath}
\usepackage{amssymb}
\usepackage{amsthm}
\usepackage{graphicx}
\usepackage{multirow}
\usepackage{natbib}
\usepackage{float}
\usepackage{xcolor}
\usepackage{etoc}
\usepackage{tikz}
\usepackage{enumitem}
\usepackage{titletoc}
\usepackage{etoc}

\usetikzlibrary{shapes,arrows,calc}

\newtheorem{proposition}{Proposition}
\newtheorem{remark}{Remark}
\usepackage{lineno}

\definecolor{darkblue}{rgb}{0, 0, 0.35}
\hypersetup{colorlinks=true, citecolor=darkblue, linkcolor=darkblue, urlcolor=darkblue}

\definecolor{rosy}{RGB}{220,20,120 }
\definecolor{NavyBlue}{RGB}{68,139,180}
\definecolor{ForestGreen}{RGB}{74,136,91}
\hypersetup{colorlinks=true, allcolors=NavyBlue}

\title{Matching Accuracy, Different Geometry: \\Evolution Strategies vs GRPO in LLM Post-Training}

\author{
  William Hoy \textsuperscript{1} \qquad Binxu Wang \textsuperscript{2}\thanks{Equal advising.} \qquad Xu Pan \textsuperscript{3}\footnotemark[1]\\[0.5em]
  \textsuperscript{1} Department of Industrial and Systems Engineering, University of Miami \\ \textsuperscript{2} Kempner Institute, Harvard University \qquad \textsuperscript{3} Harvard University \\[0.3em]
  \textsuperscript{1} \texttt{wjh58@miami.edu} \qquad 
  \textsuperscript{2}\texttt{binxu\_wang@hms.harvard.edu}\qquad \textsuperscript{3}\texttt{xupan@fas.harvard.edu}
}

\begin{document}

\ifcolmsubmission
\linenumbers
\fi

\maketitle

\ifcolmfinal
\lhead{Published as a conference paper at COLM 2026}
\thispagestyle{fancy}
\fi

\begin{abstract}
Evolution Strategies (ES) have emerged as a scalable gradient-free alternative to reinforcement learning based LLM fine-tuning, but it remains unclear whether comparable task performance implies comparable solutions in parameter space. We compare ES and Group Relative Policy Optimization (GRPO) across four tasks in both single-task and sequential continual-learning settings. ES matches or exceeds GRPO in single-task accuracy and remains competitive sequentially when its iteration budget is controlled. Despite this similarity in task performance, the two methods produce markedly different model updates: ES makes much larger changes and induces broader off-task KL drift, whereas GRPO makes smaller, more localized updates. Strikingly, the ES and GRPO solutions are linearly connected with no loss barrier, even though their update directions are nearly orthogonal. 
We develop an analytical theory of ES that explains all these phenomena within a unified framework, showing how ES can accumulate large off-task movement on weakly informative directions while still making enough progress on the task to match gradient-based RL in downstream accuracy. 
These results show that gradient-free and gradient-based fine-tuning can reach similarly accurate yet geometrically distinct solutions, with important consequences for forgetting and knowledge preservation. The source code is publicly available.\footnote{\url{https://github.com/Bhoy1/ESvsGRPO}}

\end{abstract}

\section{Introduction}


Fine-tuning large language models (LLMs) is essential for adapting them to new tasks and for maintaining performance in continual learning settings. When explicit supervised targets are unavailable, optimization must proceed from evaluative feedback rather than token-level labels. In this regime, gradient-based reinforcement learning methods such as Group Relative Policy Optimization (GRPO) have become a standard component of LLM post-training \citep{shao2024grpo, deepseek2025r1}, while Evolution Strategies (ES) have recently emerged as a scalable gradient-free alternative \citep{qiu2025es, korotyshova2025essa, sarkar2025hyperscale}. ES replaces backpropagation with population-level perturbation and reward aggregation, offering potential advantages in memory usage and parallelism.


Despite growing interest in ES for LLM alignment and fine-tuning, its behavior relative to gradient-based methods remains poorly understood. Prior works present a mixed picture. Some studies \citep{qiu2025es} suggest that ES can achieve competitive performance while retaining low KL divergence to the base model. 
This raises the possibility that it may be well suited for continual learning, where preserving prior capabilities is important. Other work \citep{abdi2026es_forgetting}, however, finds that ES can induce substantial forgetting, attributing this to larger and less sparse parameter updates. Taken together, these results leave open a basic question: when ES and gradient-based RL achieve similar task performance, do they induce similar off-target effects (e.g. forgetting), and do the solutions have similar geometric properties (e.g. norm and dimensionality)?




In this study, we compare ES and RL in a setup where an LLM is trained sequentially across four diverse tasks and extend the analysis with a geometric characterization of the solutions each method finds. Our analysis reveals that despite equivalent accuracy across four diverse tasks, ES and GRPO navigate the loss landscape in different ways. GRPO identifies a narrow, high gradient direction, making precise and targeted updates that stay close to the base model. ES navigates a broad, low curvature plateau, accumulating large diffuse updates that travel orders of magnitude further from initialization. Strikingly, despite these near orthogonal trajectories and orders of magnitude norm differences, no loss barrier exists between the two solutions. Moreover, when we probe the loss landscape along each method’s learned update direction, the ES direction behaves much more like a random direction on held-out tasks, indicating that a substantial fraction of its displacement lies in task-irrelevant dimensions rather than in a sharply task-aligned subspace. 

Finally, we derive a theoretical account of these geometric differences. 
ES weight updates decompose into two components: an on-manifold (task relevant) component, and an off-manifold (task-invariant) component. In the off-manifold subspace, ES is equivalent to a random walk, which exhibits a characteristic scaling: the squared norm of the weight change grows as $\alpha^2 dT/N$, proportional to parameter dimension $d$ and step count $T$, and inversely proportional to population size $N$. In a high-dimensional landscape with many flat dimensions — this off-manifold diffusion dominates the weight change. Empirically, LLM parameters (particularly weight matrices) precisely exhibit this predicted random walk behavior, with the theoretical scaling matching observations at $R^2 > 0.99$. 
This single mechanism accounts for all the geometric phenomena we observe: the large update norm, the near-orthogonality between ES and GRPO update directions, the flatter loss curvature along ES weight changes, and the linear mode connectivity despite geometric separation.

Our contributions are:
\begin{enumerate}
\item Empirical evidence that ES can match or exceed GRPO in task accuracy across both single-task and sequential fine-tuning settings.
\item A geometric characterization of ES versus GRPO solutions, revealing that equivalent task performance can be achieved with orders of magnitude norm differences.
\item A theoretical account of the random-walk-like behavior of ES in high-dimensional weakly informative subspaces, explaining its large update norms and diffuse-like behavior relative to GRPO.
\end{enumerate}

\section{Related Work}



\paragraph{Evolution strategies for LLMs.} 
\citet{salimans2017es} demonstrated ES as a scalable alternative to reinforcement learning via high parallelizability. Recent work has brought ES to LLM scale: \citet{qiu2025es} showed ES achieving comparable performance to GRPO on reasoning tasks, attributing this to LLMs' low intrinsic dimensionality~\citep{aghajanyan2021intrinsic} and ES's smoothing of jagged reward landscapes via Gaussian convolution in parameter space~\citep{lehman2018es}. \citet{korotyshova2025essa} and \citet{sarkar2025hyperscale} further showed that ES can be made practical for large-scale alignment through low-rank parameterizations that improve efficiency and scalability.  \citet{liang2026blessing} showed that fine-tuning loss landscapes are low-dimensional in curvature, helping explain why ES succeeds with small populations. However, \citet{abdi2026es_forgetting} demonstrates that ES causes significant forgetting compared to GRPO on smaller models (Qwen2.5-1.5B and Llama-3.2-1B), attributing it to larger update norms and lower sparsity.

\paragraph{Gradient-based RL for LLMs.}
\citet{shao2024grpo} introduced GRPO, which eliminates the critic model by estimating advantages from group statistics, reducing memory overhead compared to PPO. It became a standard post-training method following its use in DeepSeek-R1~\citep{deepseek2025r1}. Several concurrent works have investigated why GRPO preserves prior capabilities during fine-tuning. \citet{shenfeld2025rl} showed that forgetting correlates with KL divergence from the base model, while \citet{chen2025onpolicy} showed that GRPO's robustness may stem more from its use of on-policy data rather than KL regularization or advantage estimation.


\section{Methodology}


We compare two fundamentally different approaches to fine tune large language 
models: ES, a gradient free method, and GRPO, a gradient based reinforcement learning algorithm.

\subsection{Evolution strategies}

ES optimizes model parameters through population based search without computing 
gradients. At each iteration, we sample $N$ perturbation vectors $\epsilon_i \sim 
\mathcal{N}(0, I)$, evaluate the perturbed models, and update parameters using 
z score normalized rewards:
\begin{equation}\label{eq:ES_equation}
\theta_{t+1} = \theta_t + \alpha \cdot \frac{1}{N} \sum_{i=1}^{N} Z_i \epsilon_i
\end{equation}
where $\alpha$ is the learning rate, $Z_i = \frac{R_i - \mu_R}{\sigma_R}$ is the 
z score normalized reward for the $i$-th perturbation, and $R_i = R(\theta_t + 
\sigma \epsilon_i)$ is the reward of the perturbed model with noise scale $\sigma$. 
Following~\citet{qiu2025es}, we maintain $\alpha = \sigma / 2$ throughout and 
apply ES updates to the complete model weights.

\subsection{Group relative policy optimization}

GRPO~\citep{shao2024grpo} normalizes advantages within groups of sampled responses, 
eliminating the need for a critic model. For a prompt $x$, we sample $K$ responses 
$\{y_1, \ldots, y_K\}$ from the current policy $\pi_\theta$ and compute 
group-normalized advantages: \(\hat{A}_i = \frac{R(y_i) - \mu_R}{\sigma_R + \epsilon}\),
where $\mu_R$ and $\sigma_R$ are the within group mean and standard deviation of 
rewards. The policy is updated by maximizing the clipped surrogate objective: \(\mathcal{L}_{\text{GRPO}}(\theta) = \mathbb{E}\left[\frac{1}{K}\sum_{i=1}^{K} 
\min\left(\rho_i(\theta)\hat{A}_i, 
\text{clip}(\rho_i(\theta), 1-\epsilon, 1+\epsilon)\hat{A}_i\right)\right]\),
where $\rho_i(\theta) = \frac{\pi_\theta(y_i|x)}{\pi_{\theta_{\text{old}}}(y_i|x)}$ 
is the probability ratio between the current and old policy. We disable the KL 
penalty ($\beta_{\text{KL}} = 0$) during training to induce a clean comparison with ES without explicit regularization \citep{shenfeld2025rl}, but track KL divergence to 
the reference policy for analysis in both methods.

\subsection{Experimental setup}

All experiments use Qwen3-4B-Instruct-2507 as the base model. We evaluate on four diverse tasks spanning reasoning, mathematics, scientific knowledge, and reading comprehension: Countdown (arithmetic reasoning), Math (mathematical problem solving), SciKnowEval-Chemistry (chemistry domain knowledge), and BoolQ (boolean reading comprehension). Each task uses 200 training samples and 2,000 test samples (500 for Math). In the continual learning experiments, tasks are presented sequentially. Hyperparameter sweeps were conducted for both methods (see Appendix~\ref{sec:hyperparameter-sweeps}); we selected $N=30$, $\sigma{=}0.0015$, $\alpha{=}0.00075$ for ES and $\text{lr}{=}1{\times}10^{-5}$ for GRPO. An additional model, Llama-3.2-3B-Instruct, can be found in the Appendix.

\section{Empirical Results}
\subsection{Task performance}

Table \ref{tab:single-task} compares the task accuracy of the base model against ES and GRPO in a single task training setup. Running ES for 300 iterations consistently yields the highest peak accuracy across all four tasks. For instance, on the Chemistry task, ES (300) achieves 76.5\% accuracy, significantly outperforming both ES (100) at 68.1\% and GRPO at 74.9\%.

However, when the model is trained sequentially across the tasks (Figure \ref{fig:sequential-full}), utilizing 300 ES iterations causes pronounced degradation in earlier tasks as the training progresses. Limiting ES to 100 iterations offers a balance. While it sacrifices a small degree of peak single task accuracy compared to the 300 iteration configuration, it maintains sequential training stability. Both ES (100) and GRPO remain comparatively stable throughout the entire sequential training pipeline, avoiding the sharp drop-offs seen with ES (300). Because of this, we use an iteration count of 100 for ES in all of the following sequential experiments.
\begin{table}[H]
\vspace{-14pt}
\centering
\footnotesize
\caption{Accuracy (\%) for single task training.}
\label{tab:single-task}
\begin{tabular}{lcccc}
\toprule
\textbf{Task} & \textbf{Base} & \textbf{ES (iter=100)} & \textbf{ES (iter=300)} & \textbf{GRPO} \\
\midrule
Countdown & 20.5 & 75.1 & \textbf{78.6} & 74.5 \\
Math      & 61.8 & 72.4 & \textbf{74.0} & 68.8 \\
Chemistry & 63.4 & 68.1 & \textbf{76.5} & 74.9 \\
BoolQ     & 83.7 & 87.2 & \textbf{88.1} & 86.6 \\
\bottomrule
\end{tabular}
\vspace{-5pt}
\end{table}
As a side note, ES and GRPO show different transfer dynamics in the sequential task training. ES exhibits notable zero-shot forward transfer, i.e., learning a task increases performance on an unseen task. Chemistry accuracy rises from 63\% to approximately 69\% before its training stage begins, and Math accuracy increases by roughly 3 points after Countdown training alone. This forward transfer is not universal and depends on the tasks, as BoolQ accuracy dips slightly below the base model before its training stage. GRPO shows little forward transfer by comparison.

\begin{figure}[!ht]
\vspace{-10pt}
    \centering
    \includegraphics[width=0.82\linewidth]{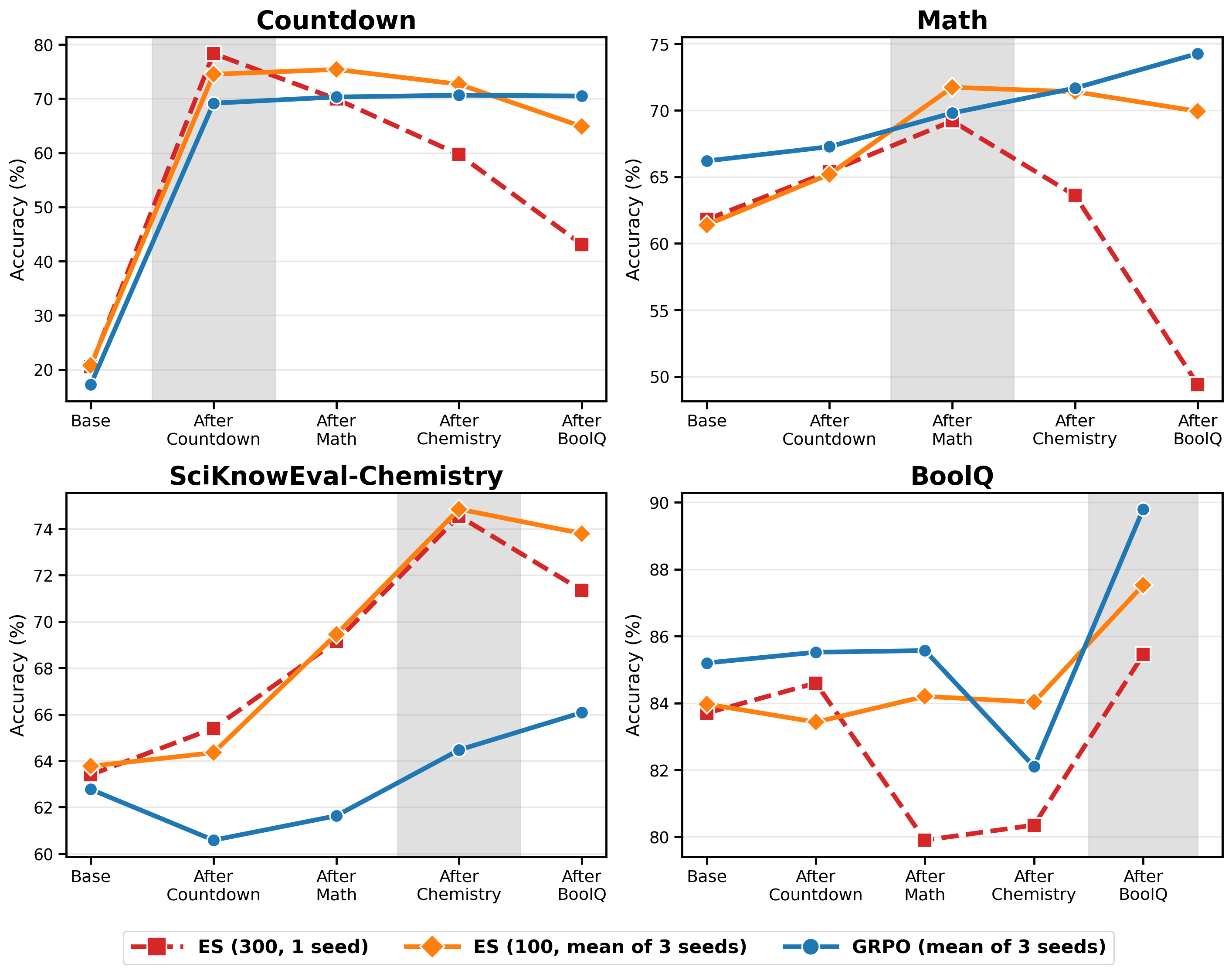}
    \caption{Accuracy (\%) on each task throughout sequential training. The 
    shaded region marks the stage at which that task was trained. ES (300) 
    shows clear degradation on earlier tasks as training progresses, while 
    ES (100) and GRPO remain comparatively stable.}
    \label{fig:sequential-full}
\vspace{-10pt}
\end{figure}

\subsection{Forgetting}


To test whether fine-tuning affects general capabilities beyond training tasks, we evaluate both methods on two held-out benchmarks after each training stage: MMLU for general knowledge, and IFEval (strict) for instruction-following ability. The two methods diverge on MMLU. ES accuracy drops monotonically from 77.5\% to 73.8\% across the four training stages ($-$3.7\%), with the largest single-stage drop during Chemistry training ($-$1.6\%). GRPO, by contrast, remains stable or slightly improves throughout, ending at 78.3\% ($+$0.8\%). On IFEval, the pattern reverses: ES gains $+$2.2\%, peaking after Math training, while GRPO rises through Chemistry but drops after BoolQ training, ending at $-$1.3\% below baseline.

\begin{table}[H]
\vspace{-10pt}
\centering
\footnotesize
\caption{Holdout benchmark accuracy (\%) throughout sequential training.}
\label{tab:holdout-forgetting}
\begin{tabular}{lcccc}
\toprule
\textbf{Stage} & \textbf{ES (MMLU)} & \textbf{GRPO (MMLU)} & \textbf{ES (IFEval)} & \textbf{GRPO (IFEval)} \\
\midrule
Base model      & 77.5 & 77.5 & 79.7 & 79.7 \\
After Task 1 (Countdown) & 77.0 & 77.0 & 80.0 & 80.6 \\
After Task 2 (Math)      & 76.2 & 77.5 & 82.6 & 80.6 \\
After Task 3 (Chemistry) & 74.6 & 78.0 & 82.3 & 81.7 \\
After Task 4 (BoolQ)     & 73.8 & 78.3 & 81.9 & 78.4 \\
\midrule
$\Delta$ (base $\to$ final) & $-$3.7 & $+$0.8 & $+$2.2 & $-$1.3 \\
\bottomrule
\end{tabular}
\end{table}

\begin{figure}[H]
\vspace{-10pt}
    \centering
    \includegraphics[width=0.95\linewidth]{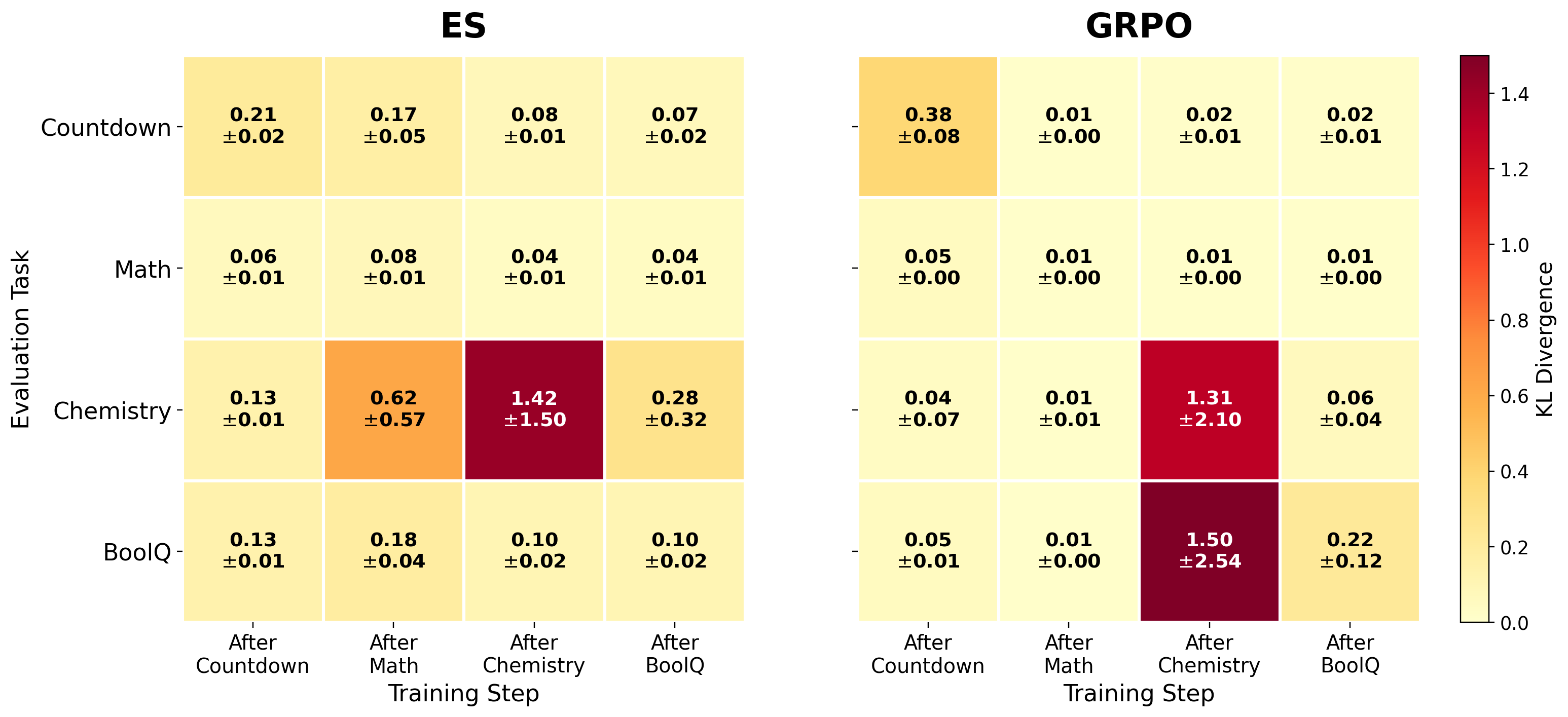}
    \caption{Incremental KL divergence per training step. Each cell shows $\text{KL}(\pi_{\text{after}} \| \pi_{\text{before}})$ for a given evaluation task (row) after a given training step (column). Both panels share the same color scale. Across three seeds there is large variance for some blocks. We have included median, min and max in Appendix~\ref{sec:kl-median-detail} so the complete picture can be seen}
    \label{fig:kl-drift}
\end{figure}


We measure the incremental KL divergence $\text{KL}(\pi_{\text{after}} \| \pi_{\text{before}})$ on each task's held-out set after every training step, capturing how much each step shifts the model's output distribution on every task (Figure~\ref{fig:kl-drift}). \citet{qiu2025es} find that ES produces substantially lower KL divergence than GRPO on a single-task conciseness benchmark. The first column of our heatmap (T0, Countdown only) is broadly consistent: ES's on-task KL (0.21 nats) is slightly lower than GRPO's (0.38 nats). In the sequential setting, however, ES produces larger off-diagonal KL shifts---training on any task causes non-trivial drift on all other tasks (e.g., training on Math shifts Chemistry by 0.62 nats for ES vs.\ 0.01 for GRPO). ES's on-diagonal drift is also more aggressive: Chemistry training induces 1.42 nats for ES versus 1.31 for GRPO. There is a single example where GRPO induces larger off diagonal than ES but if you compared medians in Appendix~\ref{sec:kl-median-detail} than it goes away. This is consistent with our weight-space findings: ES's isotropic perturbations spread change broadly across the output distribution, while GRPO's directed updates concentrate it on the trained task. The discrepancy with \citet{qiu2025es} in the sequential setting may reflect differences in task structure or model, but it suggests the KL advantage of ES is not universal. 

To quantify off-diagonal drift directly, we compute the ratio of off-task to
on-task incremental KL separately for each seed $s$:
\[
\rho_s =
\frac{\sum_t \kappa_{t,s}^{\mathrm{off}}}
     {\sum_t \kappa_{t,s}^{\mathrm{on}}},
\]
where $\kappa_{t,s}^{\mathrm{on}}$ is the test-set incremental KL on the task
trained at transition $t$, and $\kappa_{t,s}^{\mathrm{off}}$ is the mean
test-set incremental KL over the other three tasks at the same transition.
Thus, $\rho_s \to 0$ indicates that drift is concentrated on-task, whereas
$\rho_s \approx 1$ indicates that average off-task drift is comparable to
on-task drift. For Qwen3-4B-Instruct-2507, we compute $\rho_s$ separately for
seeds 42, 52, and 62 and report the mean and sample standard deviation across
seeds. ES yields $\bar{\rho}_{\mathrm{ES}}=0.533\pm0.407$, compared with
$\bar{\rho}_{\mathrm{GRPO}}=0.228\pm0.108$ for GRPO. The ES mean is therefore
$2.34\times$ the GRPO mean, indicating greater off-task drift relative to
on-task change under ES. The large variation across seeds further indicates
substantial seed-dependent heterogeneity in the magnitude of this effect. 

We additionally report cumulative test-set KL divergence from the original
base model throughout training,
$D_{\mathrm{KL}}(\pi_{\theta_t}\,\|\,\pi_{\theta_0})$
(Appendix~\ref{sec:cumulative-kl}).

\section{Weight Space Geometry Analysis}
\label{sec:weight-space}

Despite both ES and GRPO achieving similar accuracy when trained with a single task, their difference in the sequential task experiments suggests that the weight changes of both could be distinct. We characterize the geometry of their weight changes, seeking a deeper understanding of the difference between the two methods.

\subsection{ES induces larger weight changes}

We measure the total $\ell_2$ norm of each method's weight difference from the base model:
\begin{equation}
\|\Delta\theta\|_2 = \|\theta_{\text{trained}} - \theta_{\text{base}}\|_2
\end{equation}
Table~\ref{tab:norms} shows this at each sequential checkpoint. ES updates grow from 87.28 to 173.00 across the four tasks, while GRPO updates remain two orders of magnitude smaller (1.00 to 1.84), yielding norm ratios of 87 to 107$\times$.

\begin{table}[H]
\vspace{-10pt}
\centering
\footnotesize
\caption{Weight update norms ($\|\Delta\theta\|_2$) at each sequential checkpoint, measured from the base model.}
\label{tab:norms}
\begin{tabular}{lccc}
\toprule
\textbf{Checkpoint} & \textbf{ES norm} & \textbf{GRPO norm} & \textbf{Ratio} \\
\midrule
After Task 1 (Countdown) & 87.28  & 1.00 & 87$\times$ \\
After Task 2 (Math)      & 122.76 & 1.15 & 107$\times$ \\
After Task 3 (Chemistry) & 149.99 & 1.65 & 91$\times$ \\
After Task 4 (BoolQ)     & 173.00 & 1.84 & 94$\times$ \\
\bottomrule
\end{tabular}
\vspace{-10pt}
\end{table}

\subsection{ES updates are denser and higher rank}
\label{sec:sparsity-rank}

The $\ell_2$ norm alone conflates several ways of geometric change; a large norm can arise from large changes along a few coordinates or from small changes spread broadly across many parameters. To disambiguate these, we measure two other properties of each method's weight update at every sequential checkpoint: \emph{sparsity}, the fraction of update entries $\Delta\theta_i$ with $|\Delta\theta_i| \le \varepsilon$ (using $\varepsilon = 1\times10^{-5}$), and \emph{stable rank}, $\|\Delta W\|_F^2 / \|\Delta W\|_2^2$ (Frobenius norm over spectral/operator norm), which quantifies the effective dimensionality of the update for matrix-shaped parameters. Table~\ref{tab:sparsity-rank} reports both measures. GRPO's updates are substantially sparser (87.7--88.0\% near-zero entries, versus 2.1--3.8\% for ES) and lower rank (60--103, versus 1078--1100 for ES). This is consistent with our picture (Sec.~\ref{sec:theory}): ES's isotropic, high-dimensional diffusion produces dense, high-rank updates, while GRPO's targeted, gradient-aligned updates are concentrated in a sparse, low-rank subspace.

\begin{table}[H]
\vspace{-14pt}
\centering
\footnotesize
\caption{Sparsity (fraction of entries with $|\Delta\theta_i| \le \varepsilon = 1\times10^{-5}$) and stable rank ($\|\Delta W\|_F^2/\|\Delta W\|_2^2$, averaged across attention QKVO projections, and MLP up/down/gate projections) of ES and GRPO weight updates at each sequential checkpoint.}
\label{tab:sparsity-rank}
\begin{tabular}{lcccc}
\toprule
\textbf{Stage} & \textbf{Sparsity ES} & \textbf{Sparsity GRPO} & \textbf{Stable Rank ES} & \textbf{Stable Rank GRPO} \\
\midrule
T0 Countdown & 3.8\% & 88.0\% & 1100 & 60  \\
T1 Math      & 2.9\% & 88.0\% & 1090 & 70  \\
T2 Chemistry & 2.4\% & 87.7\% & 1084 & 89  \\
T3 BoolQ     & 2.1\% & 87.7\% & 1078 & 103 \\
\bottomrule
\end{tabular}
\vspace{-14pt}
\end{table}

\subsection{ES and GRPO solutions are linearly mode-connected}

We evaluate whether ES and GRPO solutions occupy the same loss basin by testing 
linear interpolations (i.e. linear mode connectivity) at each sequential checkpoint:
\begin{equation}
\theta(\alpha) = (1 - \alpha)\,\theta_{\text{ES}} + \alpha\,\theta_{\text{GRPO}}, 
\quad \alpha \in \{0, 0.25, 0.5, 0.75, 1\}
\end{equation}
A loss barrier would manifest as degraded performance at intermediate $\alpha$ 
values relative to both endpoints. 
Figure~\ref{fig:connectivity} shows average 
accuracy across tasks at each interpolation point. No significant loss barrier exists at 
any checkpoint: performance transitions smoothly, and intermediate points 
frequently match or outperform the endpoints, demonstrating that ES and GRPO 
remain in the same loss basin throughout sequential training.


\begin{figure}[!ht]
    \centering
        \includegraphics[width=0.75\linewidth]{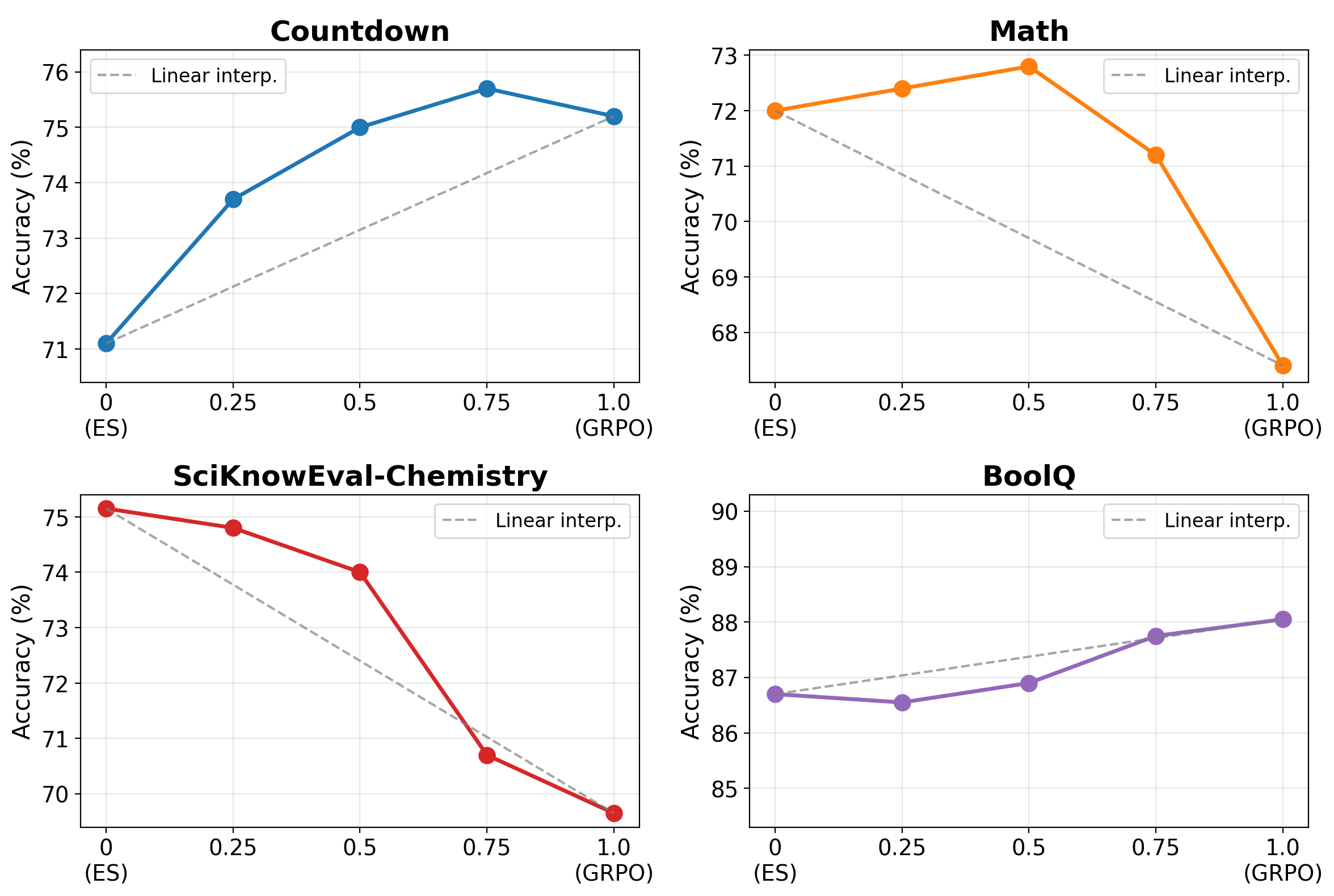}
        \caption{Per-task accuracy along the linear interpolation path between the
        final ES and GRPO checkpoints. The dashed gray line shows the linear
        interpolation between endpoint accuracies. Across all tasks, the interpolated
        model remains close to or above the linear baseline with no catastrophic
        accuracy drop, indicating that the two solutions lie in the same basin
        of the loss landscape.}
        \label{fig:connectivity}
\end{figure}

\subsection{ES solution direction behaves similarly as random directions}

To understand why ES and GRPO produce such different weight geometries despite similar task accuracy, we probe the loss landscape along each method's learned update direction. Starting from the base model, we move along the normalized ES or GRPO weight delta at increasing magnitudes:
\begin{equation}
\theta_{\text{perturbed}} = \theta_{\text{base}} + m \cdot \frac{\Delta\theta}{\|\Delta\theta\|_2}
\end{equation}
where $m$ ranges from 0 to $\|\Delta\theta\|_2$ (the trained checkpoint) and beyond. Random directions, averaged over three seeds, serve as a control. By construction, $m = \|\Delta\theta\|_2$ exactly recovers the trained checkpoint.

Figure~\ref{fig:directional-perturbation} shows that along the GRPO direction, task accuracy rises rapidly from the base-model level to the checkpoint level within a very small displacement, then degrades quickly once the perturbation magnitude exceeds the trained solution. Along the ES direction, accuracy improves more gradually over a much larger displacement scale. Both methods reach comparable final accuracy, but GRPO concentrates its useful signal into a much smaller norm, while ES spreads it across a larger perturbation, consistent with ES containing a larger component weakly coupled to the target task. 

\begin{figure}[!ht]
\vspace{-12pt}
    \centering
    \includegraphics[width=0.8\linewidth]{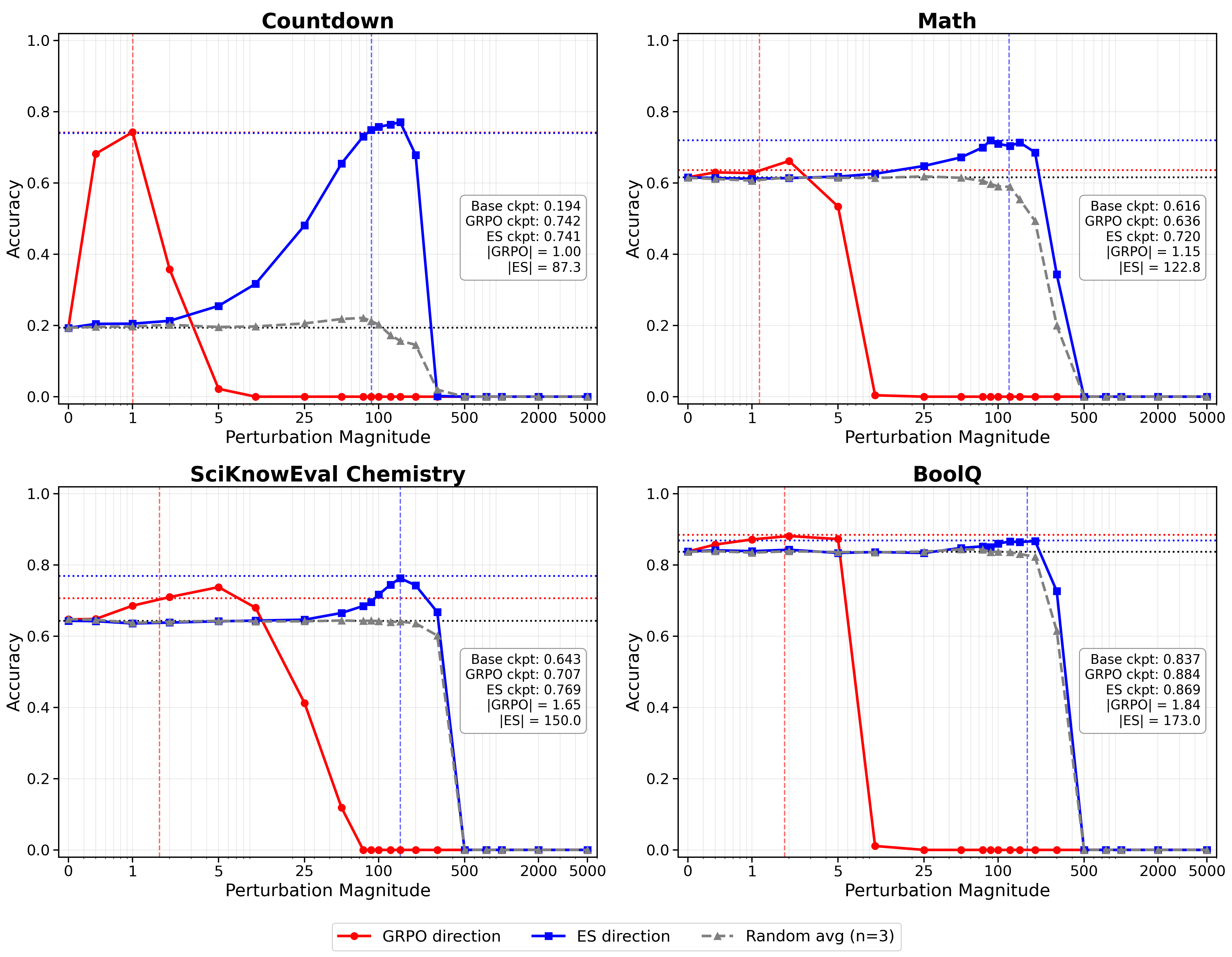}
    \caption{Task accuracy when moving from the base model along ES (blue), GRPO (red), and random (gray) directions at increasing magnitudes. Dashed vertical lines mark $\|\Delta_{\text{GRPO}}\|$ (red) and $\|\Delta_{\text{ES}}\|$ (blue). Dotted horizontal lines show base and checkpoint accuracies. }
    \label{fig:directional-perturbation}
\vspace{-12pt}
\end{figure}

Evaluating perturbed models on holdout tasks reinforces this picture (Fig.~\ref{fig:holdout-perturbation}). The ES and random directions produce similar holdout degradation profiles, consistent with the ES update being a non-targeted diffused shift. GRPO's direction is more effective per unit norm on the target task, but degrades general capabilities faster than either ES or random directions, consistent with its updates being concentrated in a narrow task-specific subspace.
\begin{figure}[!ht  ]
\vspace{-12pt}
    \centering
    \includegraphics[width=0.8\linewidth]{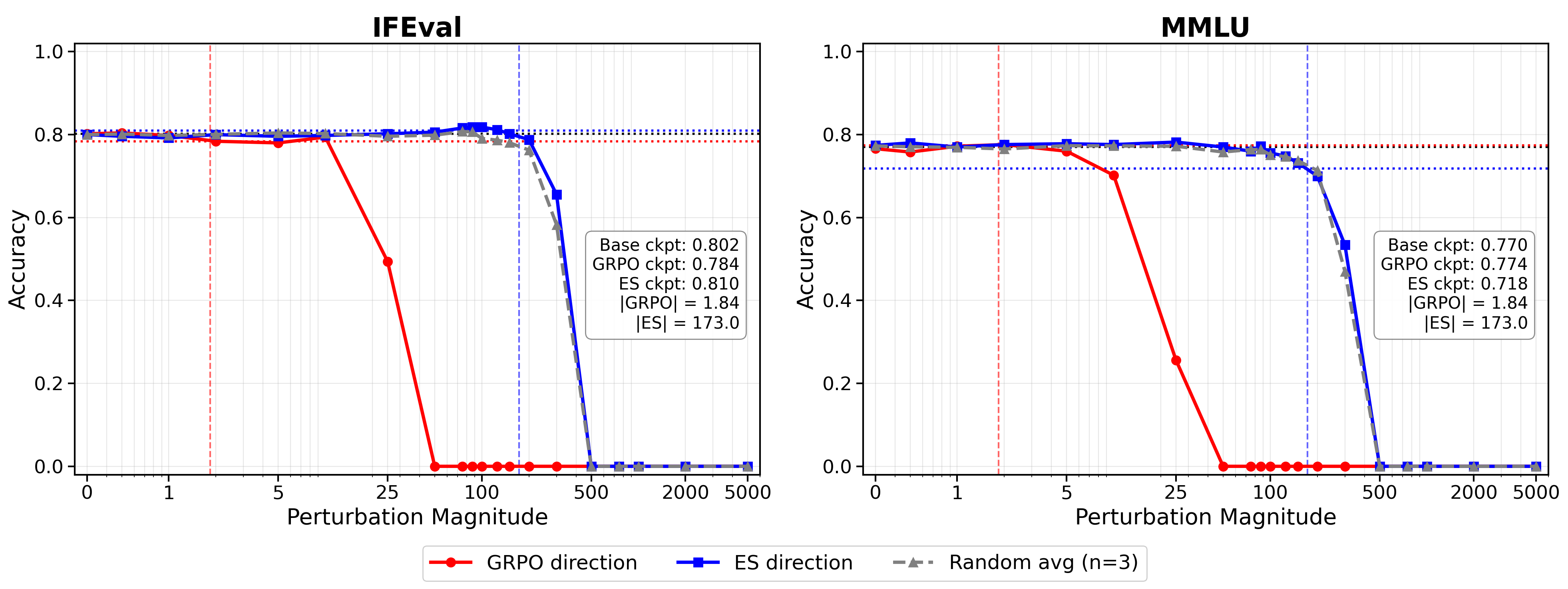}
    \caption{Holdout task accuracy (IFEval, MMLU) when perturbing along ES (blue) and GRPO (red) directions from the final BoolQ checkpoint.}
    \label{fig:holdout-perturbation}
\vspace{-12pt}
\end{figure}

\section{Theoretical Interpretation}
\label{sec:theory}
The experiments show a consistent pattern: ES and GRPO reach comparable task accuracy, yet ES accumulates much larger weight changes, resembles random directions on holdout probes, and remains linearly connected to GRPO solution in accuracy with no loss barrier. We now develop a theoretical account that explains these phenomena within a unified framework (with formal statements in App.~\ref{app:theory-derivation}).

\paragraph{Signal–diffusion decomposition.}
Consider a quadratic potential centered at $\mathbf{0}$ with curvature $Q$, $\mathcal{L}(\theta)=\tfrac12\theta^\top Q\theta$. Idealizing ES as a fixed step-size Ornstein--Uhlenbeck process surrogate, the displacement after $T$ steps decomposes as (Prop.~\ref{prop:quad_landscape_dynamics}):
\vspace{-5pt}
\begin{equation}\label{eq:es_decomposition}
    \theta_{\mathrm{ES}} - \theta_0
    = \underbrace{(H^T - I)\,\theta_0}_{\text{signal (GD-like)}}
    \;+\; \underbrace{\sum_{t=1}^{T}H^{t-1}\,\eta_{T-t}}_{\text{diffusion}},
\end{equation}
with $H := I - \frac{\alpha\sigma}{\sigma_R}Q$ and i.i.d.\ noise $\eta_t \sim \mathcal{N}(0, \frac{\alpha^2}{N}I_d)$.
The first term drives parameters toward the optimum and resides in the column space of $Q$, resembling gradient descent with $\frac{\alpha\sigma}{\sigma_R}$ playing the role of learning rate; the second term is diffusion in $d$ dimensions. Along flat directions ($\mathbf{u}$ where $Q\mathbf{u}=0$), the signal vanishes while the diffusion persists. Noiseless gradient descent -- an idealized view of GRPO, produces only the first term (Prop.~\ref{prop:gd_displacement}).

\paragraph{Why diffusion dominates}
On a flat landscape, the ES update is a pure random walk (Prop.~\ref{prop:flat_landscape_scaling}):
$\Delta\theta_t \sim \mathcal{N}(0,\, \frac{\alpha^2}{N} I_d)$, and the displacement concentrates at its expectation, 
\begin{equation}
\mathbb{E}[\|\theta_T - \theta_0\|^{2}] = {\alpha^2 T d}/{N}.
\end{equation}
This scaling---linear in $T$ and $d$, inverse in $N$---persists in the off-manifold (i.e. reward irrelevant) subspace even when the landscape is non-flat. On linear and quadratic landscapes, we derive the full single-step statistics analytically (Prop.~\ref{prop:linear_landscape_step},~\ref{prop:quad_landscape_step}). 
On a linear landscape, the fraction of gradient-aligned squared displacement per step is (App.~\ref{appsec:linear_landscape_step}):
\begin{equation}\label{eq:rho}
    \chi
    = {(1 + (N{+}1)s)}\,/\,{(d + (N{+}1)s)},
\end{equation}
where $s := \sigma^{2}\|\mathbf{v}\|^{2}/\sigma_{R}^{2} \in [0,1]$ is the \textit{signal fraction}, with $\|\mathbf{v}\|$ the gradient norm. Since $N \ll d$ in practice ($N = 30$, $d \approx 4 \times 10^9$), $\chi$ is negligible, i.e. most update is not gradient aligned. Further, on a quadratic landscape, with rank-$r$ curvature, over $T$ steps (Prop.~\ref{prop:es_displacement}):
\begin{equation}\label{eq:norm_hierarchy}
    \mathbb{E}\|\theta_{\mathrm{ES}} - \theta_0\|^2
    = \underbrace{\mathcal{O}(r)}_{\text{on-manifold}}
    + \underbrace{{\alpha^2 T(d{-}r)}/{N}}_{\text{off-manifold}}
    \approx {\alpha^2 T d}/{N}.
\end{equation}

\paragraph{Geometric consequences}

The decomposition yields a norm hierarchy (Propositions~\ref{prop:gd_displacement}--\ref{prop:es_gd_difference}):
\begin{align}
    \|\theta_{\mathrm{GD}} - \theta_0\|^2 = \mathcal{O}(r), \qquad
    \mathbb{E}\|\theta_{\mathrm{ES}} - \theta_0\|^2 = \mathcal{O}(r) + \mathcal{O}(d), \qquad
    \mathbb{E}\|\theta_{\mathrm{ES}} - \theta_{\mathrm{GD}}\|^2 = \mathcal{O}(d), \label{eq:hierarchy}
\end{align}
and the expected cosine similarity scales as $\mathcal{O}(\sqrt{r/d})$ (Remark~\ref{rmk:cosine_similarity}): nearly orthogonal despite similar task performance. The key difference is on flat directions: GD freezes at initialization while ES diffuses with variance $\alpha^2 t/N$ (Propositions~\ref{prop:quad_landscape_dynamics},~\ref{prop:gd_quad_dynamics}).

\subsection{Empirical validation and explanatory power}


\textbf{Random walk scaling.}\;
For the full parameter vector, $\|\Delta\theta\|^2$ grows linearly with step count in both models ($r > 0.9999$ for ES against $r \approx 0.87$ for GRPO; Fig.~\ref{fig:landscape_schematics}\textbf{A}). 
Inverting the slope to infer the effective dimension yields $d_{\mathrm{eff}}/d_{\mathrm{total}} = 96.4\%$ (Qwen) and $102.4\%$ (Llama)---nearly all dimensions behave as flat, the two models bracketing the parameter-free prediction within $4\%$. 
Per-parameter analysis localizes the departure: large weight matrices follow the random walk scaling almost exactly ($0.968 \pm 0.014$ and $1.024 \pm 0.031$, $R^2 = 0.9999$), while gains of normalization layers deviate strongly and differently across architectures ($1.75 \pm 0.87$ in Qwen, $0.40 \pm 0.40$ in Llama; Fig.~\ref{fig:landscape_schematics}\textbf{B,C}; see App.~\ref{appsec:weight_drift_RW_stats}). 
In summary, ES dynamics is dominated by a random walk through an almost flat landscape, with a $<4\%$ correction from the task-relevant directions---where curvature compresses the
walk, and a gradient-like component inflates it.

\textbf{Explaining each observation.}\;
Our framework explains:
\textbf{\textit{Large ES update norms}} (Table~\ref{tab:norms}): ES is dominated by the $\mathcal{O}(d)$ off-manifold term while GRPO stays at $\mathcal{O}(r)$, producing the ${\sim}100\times$ ratio.\;
\textbf{\textit{Near orthogonal updates}} (App.~\ref{sec:cosine-verification}): the cosine similarity of GD and ES weight updates is around $\sim 10^{-4}$ (Tab.~\ref{tab:cosine-verification}), consistent with the predicted scale $\mathcal{O}(\sqrt{r/d})$.\;
\textbf{\textit{Flat mode connectivity}} (Fig.~\ref{fig:connectivity}): after convergence, the $\theta_{\mathrm{ES}}$--$\theta_{\mathrm{GD}}$ line lies in the off-manifold subspace where loss is flat.\;
\textbf{\textit{ES resembles random}} (Fig.~\ref{fig:directional-perturbation}): dominant fraction of ES update comes from off-manifold random walk, explaining their similarity and smaller directional curvature along ES updates.\; 
\textbf{\textit{Holdout degradation}} (Fig.~\ref{fig:holdout-perturbation}): the fine-tuning landscape sits atop a broad plateau of base-task performance; off-manifold perturbations beyond a certain radius degrade general capabilities regardless of task relevance, explaining the matched ES--random degradation profiles.
This picture is summarized in Figure~\ref{fig:landscape_schematics}\textbf{D}. 
 \begin{figure}[!ht]
    \centering
    \vspace{-4pt}
    \includegraphics[width=0.99\linewidth]{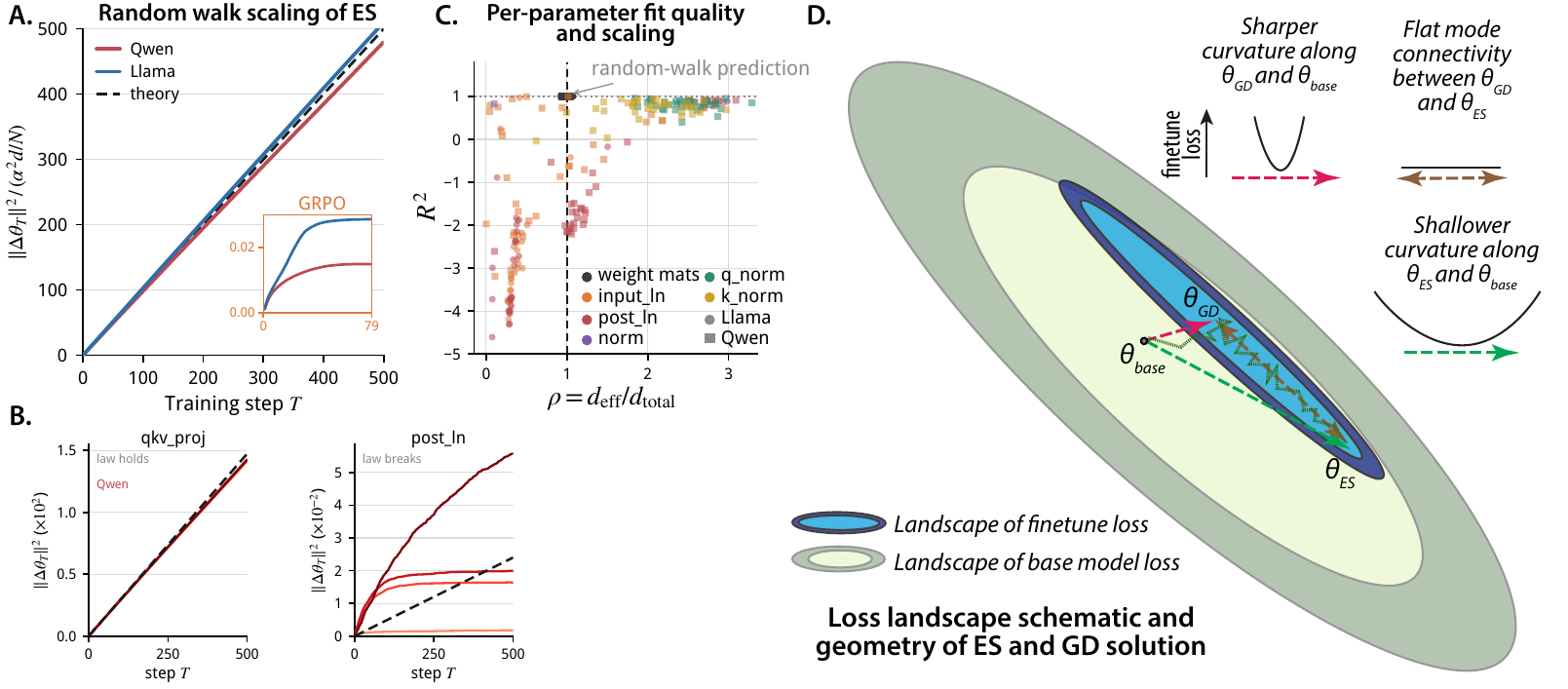}
    \vspace{-10pt}
    \caption{\textbf{ES fine-tuning drifts as a random walk through the parameter space; GRPO does not.} 
\textbf{A.} Total drift $\|\Delta\theta_T\|^{2}$ for Qwen and Llama, each divided by its own random-walk scaling $\alpha^{2} d / N$ so that the theory becomes the identity (dashed). Inset: GRPO with the same normalisation. 
\textbf{B.} Individual Qwen tensors against the same scaling (dashed): \texttt{qkv\_proj} follows it for the whole run, \texttt{post\_ln} saturates away from it.
\textbf{C.} Effective dimension ratio $\rho = d_{\mathrm{eff}}/d_{\mathrm{total}}$ against fit quality $R^{2}$ for every tensor. All weight matrices collapse onto the ideal scaling $R^2=1,\rho=1$, the norm vectors do not. 
\textbf{D.} GRPO moves along the low-dimensional task-relevant subspace (red). ES accumulates both a task-relevant component and a large off-manifold random walk (green), yielding a much larger displacement nearly orthogonal to GRPO. The two solutions are linearly connected because the off-manifold component is loss-invariant.
}
    \label{fig:landscape_schematics}
\end{figure}

\begin{remark}
High-dimensional OU processes are dominated by their flattest dimensions, thus resembling random walk \citep{antognini2018pca_highdim_random_walk}. 
Similar random-walk-like dynamics has been observed in diverse context: SGD-trained neural network weights, evolution of Influenza virus (H3N2) and ES-optimized visual stimuli \citep{Moore2018HighEvolution,wang2022highperf_ES,wang2022tuninglandscape}, suggesting high degeneracy and flat subspaces are generic features of high-dimensional landscapes.
\end{remark} 

\section{Discussion}

Our results reveal that ES and GRPO find geometrically distinct solutions despite equivalent task performance: near-orthogonal update directions and orders-of-magnitude norm differences. The random walk interpretation from Section~\ref{sec:theory} provides a unifying explanation: on the high-dimensional subspace of task-irrelevant directions, ES accumulates isotropic noise whose squared norm grows linearly with steps, while the task-relevant component remains small but sufficient for learning. These geometric differences imply that task-matched solutions found by different optimizers need not be equivalent in parameter space or in how they alter the model more broadly. More generally, our results suggest that weight-space geometry is itself an important object of study in post-training, because similar downstream accuracy can mask substantial differences in the structure and scale of the underlying update.

\clearpage
\section*{Acknowledgments}
We gratefully acknowledge support from the Kempner Research Fellowship (B.W.) and the Schwartz Fellowship (X.P.), as well as computing resources provided by the Kempner Institute cluster. 
We thank Core Francisco Park for valuable feedback on earlier versions of this work. We thank \href{https://intelligencecubed.io/}{Intelligence Cubed} and its fellowship program for providing funding to conduct the experiments.



\newpage
\bibliographystyle{colm2026_conference}
\bibliography{refs}

\appendix
\newpage

\startcontents[appendices]
\printcontents[appendices]{}{1}{\section*{Contents of the Appendices}}

\clearpage

\section{Author contributions}

W.H. designed the initial experiments, performed most of the experiments and analyses, and wrote the first draft of the manuscript. B.W. contributed to the theoretical framework and appendix derivations. X.P. coordinated the study and contributed broadly to the design and framing of the study. All authors contributed to revising and polishing the final manuscript.

\section{Disclosure of LLM Use}

In the study, LLM is used for coding the experiments, analyzing the results, assisting mathematical derivation, and polishing the writing.

\section{Experimental Configurations}
\label{sec:configs}
 
All experiments use Qwen3-4B-Instruct-2507 with full-parameter fine-tuning, seed 42, 52, 62, and a maximum of 1{,}024 generated tokens.
Each task uses a fixed training set of 200 prompts. An additional model of Llama-3.2-3B-Instruct was used to repeat the experiments with one seed, 42.

\subsection{Sequential Training Results (Qwen3-4B-Instruct)}
\label{sec:qwen_multiseed}

\begin{table}[H]
\centering
\caption{Qwen3-4B-Instruct: accuracy (\%, mean $\pm$ std over 3 seeds) throughout sequential training with ES (iter=100). Bold indicates the task just trained at that stage.}
\label{tab:qwen_es_sequential_seeds}
\begin{tabular}{lcccc}
\toprule
\textbf{Stage} & \textbf{Countdown} & \textbf{Math} & \textbf{SciK} & \textbf{BoolQ} \\
\midrule
Base                  & 20.78 $\pm$ 0.20          & 61.40 $\pm$ 0.35          & 63.77 $\pm$ 0.32          & 83.97 $\pm$ 0.23 \\
After Countdown       & \textbf{74.52 $\pm$ 0.70} & 65.20 $\pm$ 0.72          & 64.35 $\pm$ 0.80          & 83.43 $\pm$ 0.95 \\
After Math            & 75.40 $\pm$ 0.31          & \textbf{71.73 $\pm$ 0.42} & 69.45 $\pm$ 0.46          & 84.20 $\pm$ 1.18 \\
After Chemistry       & 72.67 $\pm$ 1.03          & 71.40 $\pm$ 1.40          & \textbf{74.85 $\pm$ 1.01} & 84.03 $\pm$ 2.29 \\
After BoolQ           & 64.87 $\pm$ 8.78          & 69.93 $\pm$ 1.01          & 73.80 $\pm$ 1.59          & \textbf{87.53 $\pm$ 0.60} \\
\bottomrule
\end{tabular}
\end{table}

\begin{table}[H]
\centering
\caption{Qwen3-4B-Instruct: accuracy (\%, mean $\pm$ std over 3 seeds) throughout sequential training with GRPO. Bold indicates the task just trained at that stage.}
\label{tab:qwen_grpo_sequential_seeds}
\begin{tabular}{lcccc}
\toprule
\textbf{Stage} & \textbf{Countdown} & \textbf{Math} & \textbf{SciK} & \textbf{BoolQ} \\
\midrule
Base                  & 17.22 $\pm$ 2.11          & 66.20 $\pm$ 4.01          & 62.77 $\pm$ 0.68          & 85.20 $\pm$ 1.39 \\
After Countdown       & \textbf{69.17 $\pm$ 4.86} & 67.27 $\pm$ 6.47          & 60.58 $\pm$ 4.85          & 85.52 $\pm$ 1.78 \\
After Math            & 70.32 $\pm$ 3.49          & \textbf{69.80 $\pm$ 5.89} & 61.63 $\pm$ 2.03          & 85.57 $\pm$ 1.69 \\
After Chemistry       & 70.65 $\pm$ 4.01          & 71.67 $\pm$ 5.03          & \textbf{64.47 $\pm$ 7.34} & 82.10 $\pm$ 8.00 \\
After BoolQ           & 70.50 $\pm$ 3.77          & 74.27 $\pm$ 5.80          & 66.08 $\pm$ 7.05          & \textbf{89.80 $\pm$ 1.31} \\
\bottomrule
\end{tabular}
\end{table}

\subsection{Replication on a Second Model Family (Llama-3.2-3B-Instruct)}
\label{sec:llama_replication}

\begin{table}[H]
\centering
\caption{Llama-3.2-3B-Instruct: accuracy (\%) throughout sequential training with ES (iter=100). Bold indicates the task just trained at that stage.}
\label{tab:llama_es_sequential}
\begin{tabular}{lcccc}
\toprule
\textbf{Stage} & \textbf{Countdown} & \textbf{Math} & \textbf{SciK} & \textbf{BoolQ} \\
\midrule
Base                  & 2.4           & 38.2          & 60.6          & 73.4 \\
After Countdown       & \textbf{37.4} & 41.2          & 60.9          & 75.8 \\
After Math            & 34.9          & \textbf{41.8} & 64.0          & 71.6 \\
After Chemistry       & 29.3          & 41.0          & \textbf{65.9} & 68.8 \\
After BoolQ           & 20.3          & 38.2          & 65.7          & \textbf{81.9} \\
\bottomrule
\end{tabular}
\end{table}

\begin{table}[H]
\centering
\caption{Llama-3.2-3B-Instruct: accuracy (\%) throughout sequential training with GRPO. Bold indicates the task just trained at that stage.}
\label{tab:llama_grpo_sequential}
\begin{tabular}{lcccc}
\toprule
\textbf{Stage} & \textbf{Countdown} & \textbf{Math} & \textbf{SciK} & \textbf{BoolQ} \\
\midrule
Base                  & 1.8           & 39.4          & 61.5          & 73.4 \\
After Countdown       & \textbf{37.8} & 42.2          & 62.6          & 75.3 \\
After Math            & 40.8          & \textbf{42.0} & 62.4          & 75.9 \\
After Chemistry       & 40.4          & 44.2          & \textbf{65.8} & 72.2 \\
After BoolQ           & 40.3          & 46.4          & 67.0          & \textbf{83.0} \\
\bottomrule
\end{tabular}
\end{table}
 
\subsection{Hyperparameter Sweeps}
\label{sec:hyperparameter-sweeps}
 
\paragraph{Evolution Strategies.}
We conducted a sweep over three configurations maintaining $\alpha = \sigma / 2$:
\begin{table}[H]
\centering
\caption{ES hyperparameter configurations. Config 3 selected.}
\label{tab:es_sweep}
\begin{tabular}{ccc}
\toprule
\textbf{Config} & $\sigma$ & $\alpha$ \\
\midrule
1 & 0.0005  & 0.00025  \\
2 & 0.001   & 0.0005   \\
3 & 0.0015  & 0.00075  \\
\bottomrule
\end{tabular}
\end{table}
 
\paragraph{GRPO.}
We swept learning rates with KL penalty disabled ($\beta_{\text{KL}} = 0$) and linear decay:
\begin{table}[H]
\centering
\caption{GRPO learning rate configurations. Config 3 selected.}
\label{tab:grpo_sweep}
\begin{tabular}{cc}
\toprule
\textbf{Config} & \textbf{Learning rate} \\
\midrule
1 & $1 \times 10^{-6}$ \\
2 & $3 \times 10^{-6}$ \\
3 & $1 \times 10^{-5}$ \\
4 & $3 \times 10^{-5}$ \\
\bottomrule
\end{tabular}
\end{table}

\paragraph{On disabling the KL penalty.} Although the original GRPO formulation \citep{shao2024grpo} includes a KL-divergence penalty against the reference policy, we set $\beta_{\text{KL}}=0$, following now-common practice in the field.  The widely used TRL implementation similarly defaults to $\beta=0$ for this reason.\footnote{\url{https://huggingface.co/docs/trl/main/en/grpo_trainer\#computing-the-loss}} 

\subsection{Seed Variance in Incremental KL Divergence}
\label{sec:kl-median-detail}

\begin{table}[H]
\centering
\footnotesize
\caption{ES: median [min, max] incremental test KL divergence across 3 seeds (42, 52, 62). Bold indicates the task just trained at that stage.}
\label{tab:kl-median-es}
\resizebox{\linewidth}{!}{%
\begin{tabular}{lcccc}
\toprule
\textbf{Eval.\ task} & \textbf{After Countdown} & \textbf{After Math} & \textbf{After Chemistry} & \textbf{After BoolQ} \\
\midrule
Countdown & \shortstack{\textbf{0.200}\\{[0.187, 0.232]}} & \shortstack{0.162\\{[0.124, 0.221]}} & \shortstack{0.075\\{[0.074, 0.083]}} & \shortstack{0.086\\{[0.047, 0.086]}} \\[1.2em]
Math      & \shortstack{0.062\\{[0.054, 0.071]}} & \shortstack{\textbf{0.080}\\{[0.067, 0.088]}} & \shortstack{0.041\\{[0.033, 0.042]}} & \shortstack{0.043\\{[0.036, 0.054]}} \\[1.2em]
Chemistry & \shortstack{0.137\\{[0.116, 0.143]}} & \shortstack{0.313\\{[0.272, 1.276]}} & \shortstack{\textbf{1.100}\\{[0.106, 3.062]}} & \shortstack{0.103\\{[0.090, 0.649]}} \\[1.2em]
BoolQ     & \shortstack{0.135\\{[0.126, 0.137]}} & \shortstack{0.170\\{[0.151, 0.227]}} & \shortstack{0.096\\{[0.086, 0.127]}} & \shortstack{\textbf{0.102}\\{[0.090, 0.122]}} \\
\bottomrule
\end{tabular}%
}
\end{table}

\begin{table}[H]
\centering
\footnotesize
\caption{GRPO: median [min, max] incremental test KL divergence across 3 seeds (42, 52, 62). Bold indicates the task just trained at that stage.}
\label{tab:kl-median-grpo}
\resizebox{\linewidth}{!}{%
\begin{tabular}{lcccc}
\toprule
\textbf{Eval.\ task} & \textbf{After Countdown} & \textbf{After Math} & \textbf{After Chemistry} & \textbf{After BoolQ} \\
\midrule
Countdown & \shortstack{\textbf{0.381}\\{[0.293, 0.454]}} & \shortstack{0.008\\{[0.006, 0.012]}} & \shortstack{0.012\\{[0.004, 0.030]}} & \shortstack{0.017\\{[0.005, 0.029]}} \\[1.2em]
Math      & \shortstack{0.050\\{[0.048, 0.056]}} & \shortstack{\textbf{0.009}\\{[0.003, 0.010]}} & \shortstack{0.010\\{[0.007, 0.016]}} & \shortstack{0.006\\{[0.005, 0.011]}} \\[1.2em]
Chemistry & \shortstack{0.012\\{[0.000, 0.115]}} & \shortstack{0.010\\{[0.000, 0.019]}} & \shortstack{\textbf{0.117}\\{[0.072, 3.738]}} & \shortstack{0.056\\{[0.017, 0.105]}} \\[1.2em]
BoolQ     & \shortstack{0.051\\{[0.035, 0.061]}} & \shortstack{0.010\\{[0.003, 0.010]}} & \shortstack{0.038\\{[0.028, 4.431]}} & \shortstack{\textbf{0.239}\\{[0.094, 0.322]}} \\
\bottomrule
\end{tabular}%
}
\end{table}

\subsection{Cumulative KL Divergence from the Base Model}
\label{sec:cumulative-kl}

\begin{table}[H]
\centering
\footnotesize
\caption{Qwen3-4B-Instruct: cumulative KL divergence from base model, ES (iter=100). Mean $\pm$ std over 3 seeds (42, 52, 62). Bold indicates the task just trained at that stage.}
\label{tab:cumkl-qwen-es}
\begin{tabular}{lcccc}
\toprule
\textbf{Stage} & \textbf{Countdown} & \textbf{Math} & \textbf{SciK} & \textbf{BoolQ} \\
\midrule
After Countdown & \textbf{0.206 $\pm$ 0.023} & 0.062 $\pm$ 0.008 & 0.132 $\pm$ 0.014 & 0.133 $\pm$ 0.006 \\
After Math      & 0.436 $\pm$ 0.032          & \textbf{0.178 $\pm$ 0.021} & 0.755 $\pm$ 0.477 & 0.355 $\pm$ 0.079 \\
After Chemistry & 0.536 $\pm$ 0.059          & 0.250 $\pm$ 0.033 & \textbf{2.541 $\pm$ 2.390} & 0.545 $\pm$ 0.211 \\
After BoolQ     & 0.551 $\pm$ 0.112          & 0.288 $\pm$ 0.028 & 2.013 $\pm$ 1.781 & \textbf{0.545 $\pm$ 0.158} \\
\bottomrule
\end{tabular}
\end{table}

\begin{table}[H]
\centering
\footnotesize
\caption{Qwen3-4B-Instruct: cumulative KL divergence from base model, GRPO. Mean $\pm$ std over 3 seeds (42, 52, 62). Bold indicates the task just trained at that stage.}
\label{tab:cumkl-qwen-grpo}
\begin{tabular}{lcccc}
\toprule
\textbf{Stage} & \textbf{Countdown} & \textbf{Math} & \textbf{SciK} & \textbf{BoolQ} \\
\midrule
After Countdown & \textbf{0.376 $\pm$ 0.081} & 0.052 $\pm$ 0.004 & 0.040 $\pm$ 0.066 & 0.049 $\pm$ 0.013 \\
After Math      & 0.420 $\pm$ 0.113          & \textbf{0.069 $\pm$ 0.018} & 0.087 $\pm$ 0.129 & 0.072 $\pm$ 0.025 \\
After Chemistry & 0.435 $\pm$ 0.143          & 0.093 $\pm$ 0.038 & \textbf{1.297 $\pm$ 2.120} & 1.904 $\pm$ 3.140 \\
After BoolQ     & 0.435 $\pm$ 0.115          & 0.091 $\pm$ 0.024 & 1.282 $\pm$ 2.062 & \textbf{0.194 $\pm$ 0.063} \\
\bottomrule
\end{tabular}
\end{table}

\begin{table}[H]
\centering
\footnotesize
\caption{Llama-3.2-3B-Instruct: cumulative KL divergence from base model, ES (iter=100). Single seed. Bold indicates the task just trained at that stage.}
\label{tab:cumkl-llama-es}
\begin{tabular}{lcccc}
\toprule
\textbf{Stage} & \textbf{Countdown} & \textbf{Math} & \textbf{SciK} & \textbf{BoolQ} \\
\midrule
After Countdown & \textbf{0.069} & 0.020 & 0.074 & 0.236 \\
After Math      & 0.071          & \textbf{0.038} & 0.127 & 0.536 \\
After Chemistry & 0.095          & 0.058 & \textbf{0.396} & 1.644 \\
After BoolQ     & 0.170          & 0.066 & 1.362 & \textbf{1.788} \\
\bottomrule
\end{tabular}
\end{table}

\begin{table}[H]
\centering
\footnotesize
\caption{Llama-3.2-3B-Instruct: cumulative KL divergence from base model, GRPO. Single seed. Bold indicates the task just trained at that stage.}
\label{tab:cumkl-llama-grpo}
\begin{tabular}{lcccc}
\toprule
\textbf{Stage} & \textbf{Countdown} & \textbf{Math} & \textbf{SciK} & \textbf{BoolQ} \\
\midrule
After Countdown & \textbf{0.072} & 0.011 & 0.094 & 0.312 \\
After Math      & 0.081          & \textbf{0.012} & 0.107 & 0.347 \\
After Chemistry & 0.091          & 0.022 & \textbf{2.306} & 1.411 \\
After BoolQ     & 0.099          & 0.023 & 2.330 & \textbf{1.717} \\
\bottomrule
\end{tabular}
\end{table}

\subsection{Single Task Training Curves}
\label{sec:single-task-curves}

Figure~\ref{fig:grpo_single} and Figure~\ref{fig:es_single} show train and test reward over the course of single-task training for GRPO and ES using the configurations selected in Tables~\ref{tab:es_sweep} and~\ref{tab:grpo_sweep} for Qwen3-4b-Instruct. Both methods converge on all four tasks without overfitting, confirming that these hyperparameters yield stable learning before we proceed to the sequential setting.

\begin{figure}[H]
\centering
\includegraphics[width=0.85\textwidth]
{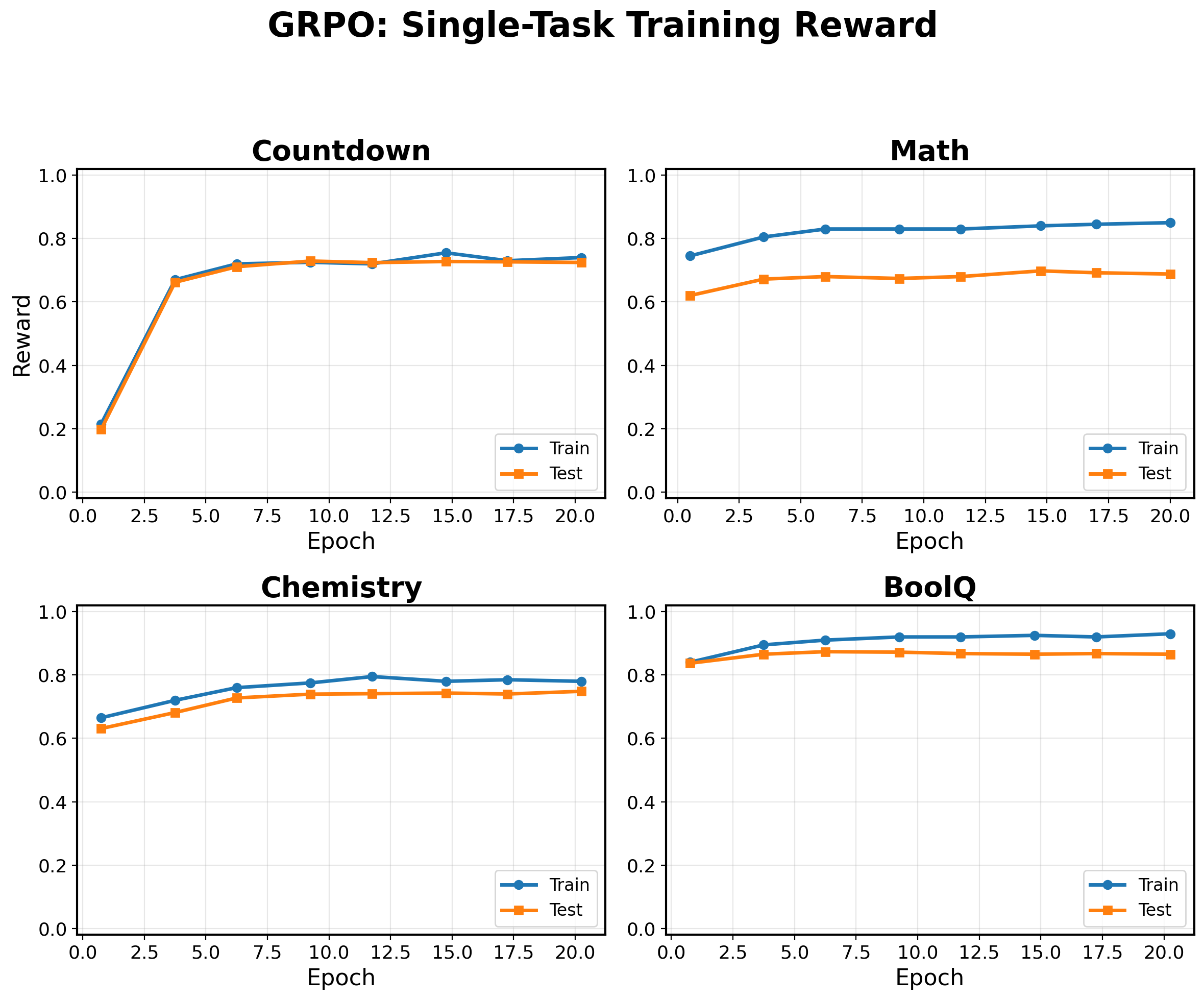}
\caption{GRPO single-task training reward across 20 epochs for each task.}
\label{fig:grpo_single}
\end{figure}

\begin{figure}[H]
\centering
\includegraphics[width=0.85\textwidth]{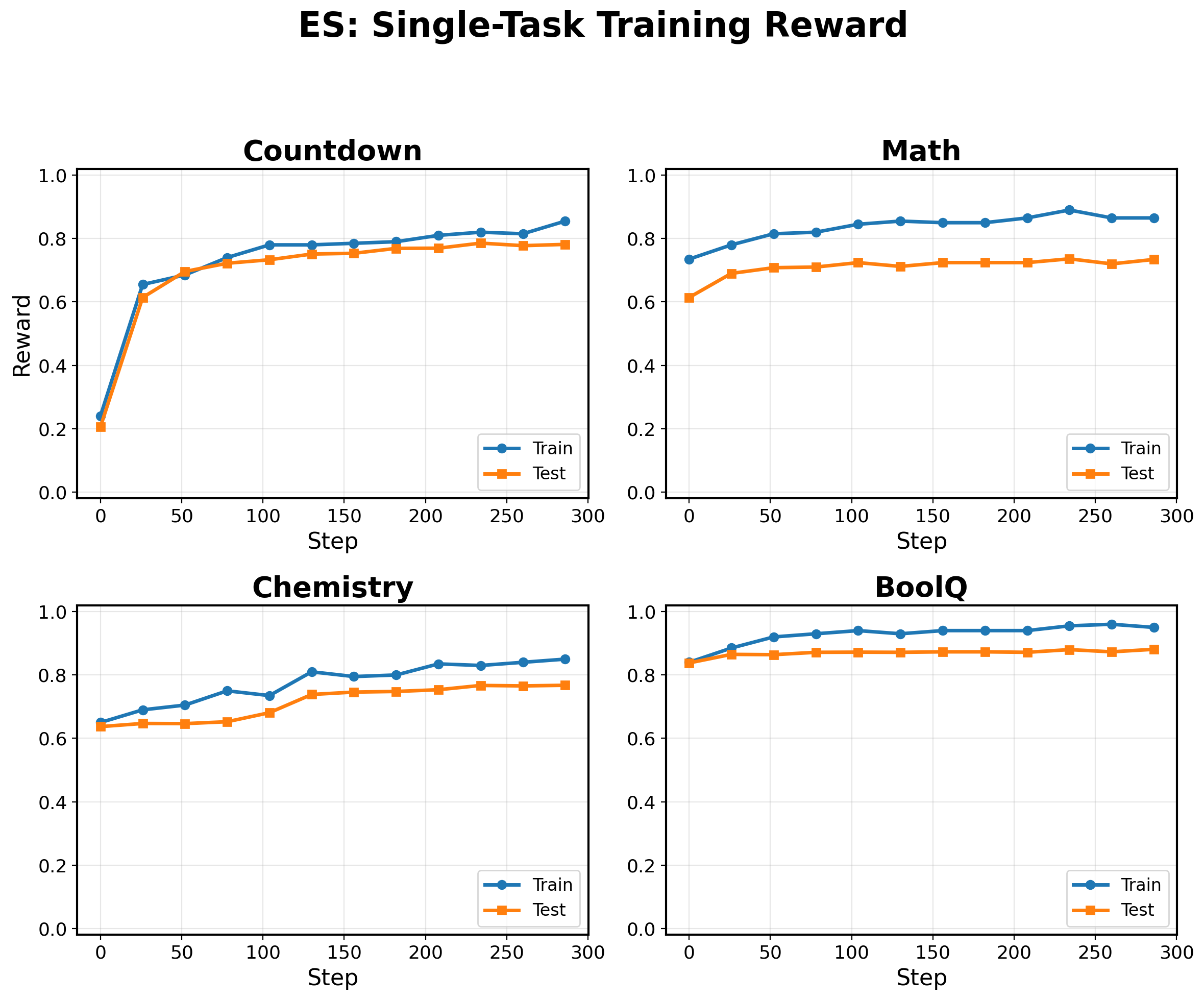}
\caption{ES single-task training reward across 300 steps for each task.}
\label{fig:es_single}
\end{figure}
 
\subsection{Selected Training Configurations}
 
The selected configurations are summarized in Table~\ref{tab:training_configs}.
In the single-task setting, ES trains for 300 steps and GRPO for 20 epochs.
In the sequential setting, each of the four tasks trains for 100 ES steps or 20 GRPO epochs with identical per-task hyperparameters.
 
\begin{table}[H]
\centering
\caption{Training hyperparameters for ES (left) and GRPO (right).}
\label{tab:training_configs}
\begin{minipage}[t]{0.42\textwidth}
\centering
\begin{tabular}{lc}
\toprule
\multicolumn{2}{c}{\textbf{ES}} \\
\midrule
Noise scale ($\sigma$)   & 0.0015 \\
Step size ($\alpha$)      & 0.00075 \\
Population size           & 30 \\
Batch size                & 200 \\
\bottomrule
\end{tabular}
\end{minipage}
\hfill
\begin{minipage}[t]{0.54\textwidth}
\centering
\begin{tabular}{lc}
\toprule
\multicolumn{2}{c}{\textbf{GRPO}} \\
\midrule
Learning rate                          & $1 \times 10^{-5}$ \\
Optimizer                              & AdamW \\
Scheduler                              & Linear decay \\
Max gradient norm                      & 1.0 \\
KL coefficient ($\beta_{\text{KL}}$)   & 0 \\
Rollouts per prompt ($K$)              & 8 \\
\bottomrule
\end{tabular}
\end{minipage}
\end{table}

\subsection{Compute Cost (FLOPs) Comparison}
\label{sec:flops}

For a model with $P$ parameters and sequence length $L$, we estimate a forward pass at $2PL$ FLOPs and a backward pass at $4PL$.

\paragraph{GRPO.} Each response requires a policy forward pass, a reference-model forward pass, and a policy backward pass:
\begin{equation}
\text{FLOPs}_{\text{GRPO}} = T_{\text{GRPO}} \cdot B \cdot G \cdot \underbrace{(2+2+4)}_{\text{fwd + ref + bwd}} \cdot PL = 8\, T_{\text{GRPO}} B G \, PL,
\end{equation}
where $T_{\text{GRPO}}$ is the number of optimizer steps, $B$ the question batch size, and $G$ the number of responses generated per question.

\paragraph{ES.} Each population member requires only a forward pass over the evaluation batch:
\begin{equation}
\text{FLOPs}_{\text{ES}} = T_{\text{ES}} \cdot N \cdot D \cdot \underbrace{2}_{\text{fwd}} \cdot PL = 2\, T_{\text{ES}} N D \, PL,
\end{equation}
where $T_{\text{ES}}$ is the number of iterations, $N$ the population size, and $D$ the number of prompts evaluated per iteration.

\paragraph{Applied to our setup.} From training logs, GRPO runs $T_{\text{GRPO}}=80$ steps per task with $B\approx63$ and $G=8$; ES uses $N=30$ and $D=200$ prompts per iteration, with $T_{\text{ES}} \in \{100, 300\}$. Table~\ref{tab:flops_compute} reports the resulting FLOP totals:

\begin{table}[H]
\centering
\caption{Compute per task, relative to GRPO.}
\label{tab:flops_compute}
\begin{tabular}{lcc}
\toprule
\textbf{Method} & \textbf{FLOPs} & \textbf{vs.\ GRPO} \\
\midrule
GRPO ($T=80$)          & $8(80)(63)(8)\,PL = 322{,}560\,PL$ & $1.00\times$ \\
ES (100 iter)           & $2(100)(30)(200)\,PL = 1{,}200{,}000\,PL$ & $3.72\times$ \\
ES (300 iter)           & $2(300)(30)(200)\,PL = 3{,}600{,}000\,PL$ & $11.16\times$ \\
\bottomrule
\end{tabular}
\end{table}

ES(100), the configuration used throughout Section~4, uses roughly $3.7\times$ the compute of GRPO, and ES(300) roughly $11\times$.

\subsection{Evaluation Protocol}
\label{sec:eval_protocol}
 
All evaluations use greedy decoding (temperature $= 0$) with a maximum of 1{,}024 generated tokens and binary rewards (1.0 correct, 0.0 otherwise).
Table~\ref{tab:eval_protocol} summarizes the answer extraction and scoring for each task.
Countdown problems are sourced from \texttt{Jiayi-Pan/Countdown-Tasks-3to4}, MATH from \texttt{nlile/hendrycks-MATH-benchmark}, SciKnowEval Chemistry from \texttt{hicai-zju/SciKnowEval} v2, BoolQ from \texttt{google/boolq}, MMLU from \texttt{cais/mmlu}, and IFEval from \texttt{google/IFEval}.
 
\begin{table}[H]
\centering
\caption{Evaluation protocol per task. All tasks use greedy decoding (temperature $= 0$), 1{,}024 max generated tokens, and binary scoring.}
\label{tab:eval_protocol}
\small
\begin{tabular}{lcrp{6.2cm}}
\toprule
\textbf{Task} & \textbf{Answer format} & $|\mathcal{D}_{\text{test}}|$ & \textbf{Answer extraction and scoring} \\
\midrule
Countdown & \verb|\boxed{expr}| & 2{,}000 & Extract last \verb|\boxed{}| via regex. Verify the expression contains only digits and $+,-,*,/$; each input number is used exactly once; and the evaluated result equals the target (tolerance $< 10^{-5}$). \\[4pt]
MATH & \verb|\boxed{ans}| & 500 & Extract last \verb|\boxed{}| with nested-brace handling. Normalize both prediction and gold by stripping LaTeX wrappers (\verb|\text|, \verb|\dfrac|, currency/degree symbols). Compare as strings; on mismatch, convert fractions and decimals to floats and compare numerically (tolerance $< 10^{-6}$). \\[4pt]
SciKnowEval Chem. & \texttt{"answer":"X"} & 2{,}000 & Extract answer letter via regex cascade: (1) JSON-style \texttt{"answer":"X"}, (2) unquoted \texttt{answer: X}, (3) last standalone A/B/C/D token. Use the last match found. \\[4pt]
BoolQ & \verb|\boxed{yes/no}| & 2{,}000 & Extract last \verb|\boxed{}|, normalize to lowercase. Accept \{yes, true, 1\} as positive and \{no, false, 0\} as negative. \\[4pt]
\midrule
\multicolumn{4}{l}{\textit{Holdout tasks (not trained on; used to measure forgetting)}} \\
\midrule
MMLU & \texttt{"answer":"X"} & 14{,}042 & Same regex cascade as SciKnowEval Chemistry. \\[4pt]
IFEval & Free-form & 541 & Check each instruction constraint programmatically (keyword presence, formatting rules, length bounds, etc.). Score 1.0 only if all constraints for a prompt are satisfied (strict prompt-level accuracy). \\
\bottomrule
\end{tabular}
\end{table}
 
MMLU and IFEval are evaluated after training to measure whether general capabilities are preserved.

\clearpage
\newpage
\section{Theoretical Results}\label{app:theory-derivation}

\input{theory_derivation}

\end{document}

%% file: theory_derivation.tex
\subsection{Evolution Strategy on a Flat Landscape}
\label{prop:flat}
\begin{proposition}[Evolution Strategy on a Flat Landscape] 
\label{prop:flat_landscape_scaling}
  Let $R(\theta) = \text{const}$, so all reward variation arises from
  observation noise. Under the ES update
  $\Delta\theta = \frac{\alpha}{N}\sum_{i=1}^N Z_i \epsilon_i$
  with $\epsilon_i \sim \mathcal{N}(0, I_d)$ and z-scored rewards
  $Z_i$, the weight update is a pure isotropic random walk:
  \[
      \Delta\theta \sim \mathcal{N}\!\left(0,\, \frac{\alpha^2}{N} I_d\right),
      \qquad
      \mathbb{E}\bigl[\|\Delta\theta\|^2\bigr] = \frac{\alpha^2 d}{N}.
  \]
  After $T$ steps of ES with population size $N$, step size $\alpha$, and parameter dimension $d$, the cumulative deviation satisfies
  \[
      \mathbb{E}\bigl[\|\theta_T - \theta_0\|^2\bigr]
      = \frac{\alpha^2 T d}{N}
      = \frac{\sigma^2 T d}{4N},
  \]
  where the last equality substitutes $\alpha = \sigma/2$.
  In particular, the expected drift grows linearly in $T$ and $d$,
  and shrinks inversely in the batch size $N$. Proof in \ref{appsec:flat_RW_scaling}.
  \end{proposition}

  \subsection{Evolution Strategy on a Linear Landscape}
  \begin{proposition}[Evolution Strategy on a Linear Landscape] \label{prop:linear_landscape_step}
    Let $R(\theta) = -\mathbf{v}^\top\theta$ with observation noise
    $\sigma_\xi$. Define $\sigma_R^2 := \sigma^2\|\mathbf{v}\|^2 + \sigma_\xi^2$.
    The mean and covariance of the ES parameter update are 
    \begin{align}
        \mathbb{E}[\Delta\theta]
        &= \frac{-\alpha\sigma\,\mathbf{v}}{\sigma_R} \label{eq:es_linear_mean} \\
        \operatorname{Cov}[\Delta\theta]
        &= \frac{\alpha^2}{N}\,I_d
          + \frac{\alpha^2\sigma^2}{N\sigma_R^2}\,\mathbf{v}\mathbf{v}^\top \label{eq:es_linear_cov}
    \end{align}
    The mean update is purely aligned with the gradient direction $\mathbf{v}$,
    and the covariance consists of the same isotropic baseline as the flat landscape
    plus a rank-1 boost along $\mathbf{v}$.
    Relative to the quadratic case (Proposition~\ref{prop:quad_landscape_step}),
    the linear landscape produces identical expressions with $Q^2 = 0$:
    the curvature-dependent inflation of $\sigma_R^2$ and the anisotropic
    noise term $2\sigma^4 Q^2$ are both absent. Proof in \ref{appsec:linear_landscape_step}.
    \end{proposition}
    
    \begin{remark}[On-manifold alignment of ES update]
    \label{remark:2}
    Define the signal fraction
    $s := \sigma^2\|\mathbf{v}\|^2/\sigma_R^2 \in [0,1]$.
    The on-manifold fraction of total squared displacement is
    \[
        \rho :=
        \frac{\mathbb{E}\|P^{\parallel}\Delta\theta\|^2}
             {\mathbb{E}\|\Delta\theta\|^2}
        =
        \frac{1 + (N+1)\,s}
             {d + (N+1)\,s}.
    \]
    Hence: (i) without signal ($s=0$), $\rho = 1/d$, i.e.\ the update is
    purely diffusive; (ii) in the high-SNR limit ($s\to 1$), $\rho \to
    (N+2)/(d+N+1)$; (iii) $\rho$ increases monotonically in $s$ and $N$,
    and decreases in $d$. Since $N \ll d$ in practice, on-manifold movement
    per step is negligible relative to off-manifold diffusion.
    \end{remark}

  
  \subsection{Evolution Strategy on a Quadratic Landscape}
  
  \begin{proposition}[Evolution Strategy on a Quadratic Landscape] \label{prop:quad_landscape_step}
  Let $R(\theta) = -\frac{1}{2}\theta^\top Q\theta$,  
  current iterate $\theta_0$, and $\mathbf{v} := Q\theta_0$
  (the negative gradient). Define
  \begin{equation} \label{eq:es-quadratic-sigmaR}
      \sigma_R^2 := \sigma^2\|\mathbf{v}\|^2
                  + \tfrac{\sigma^4}{2}\operatorname{Tr}[Q^2]
                  + \sigma_\xi^2.
  \end{equation}
  The mean and covariance of the ES parameter update are
  \begin{align}
      \mathbb{E}[\Delta\theta]
      &= \frac{-\alpha\sigma\,\mathbf{v}}{\sigma_R}, \label{eq:es-quad-mean} \\
      \operatorname{Cov}[\Delta\theta]
      &= \frac{\alpha^2}{N\sigma_R^2}
        \left(
          \sigma^2\,\mathbf{v}\mathbf{v}^\top
          + 2\sigma^4 Q^2
          + \sigma_R^2\,I_d
        \right). \label{eq:es-quad-cov}
  \end{align}
  Relative to the linear case (Proposition~\ref{prop:linear_landscape_step}), the quadratic term $Q$ introduces two
  additional contributions: (i) a curvature-dependent inflation of
  $\sigma_R^2$ through $\operatorname{Tr}[Q^2]$, which \emph{attenuates}
  the mean gradient step; and (ii) an anisotropic noise term
  $2\sigma^4 Q^2/\sigma_R^2$ in the covariance, which increases
  update variance along the principal high curvature directions
  of $Q$. Proof in \ref{appsec:quad_landscape_step}.
  \end{proposition}
  
  \begin{remark}[On-manifold alignment on quadratic landscape]
    The on-manifold fraction of total expected squared displacement is
    \[
        \rho
        =
        \frac{(N+1)\,\sigma^2\|\mathbf{v}\|^2
              + 2\sigma^4\,\hat{\mathbf{v}}^\top Q^2\hat{\mathbf{v}}
              + \sigma_R^2}
             {(N+1)\,\sigma^2\|\mathbf{v}\|^2
              + 2\sigma^4\,\operatorname{Tr}[Q^2]
              + d\,\sigma_R^2},
    \]
    where $\hat{\mathbf{v}} := \mathbf{v}/\|\mathbf{v}\|$.
    Setting $Q^2 = 0$ and $\sigma_R^2 = \sigma^2\|\mathbf{v}\|^2 + \sigma_\xi^2$
    recovers the linear-landscape ratio $\rho = (1+(N{+}1)s)/(d+(N{+}1)s)$.
    The quadratic landscape modifies $\rho$ in two ways:
    the curvature inflates $\sigma_R^2$, which grows the isotropic term
    $d\,\sigma_R^2$ in the denominator;
    and the $2\sigma^4 Q^2$ noise splits between numerator
    (the projection $\hat{\mathbf{v}}^\top Q^2\hat{\mathbf{v}}$) and
    denominator (the full trace $\operatorname{Tr}[Q^2]$).
    When the gradient aligns with the high-curvature eigenvectors,
    much of the curvature-induced noise is on-manifold and $\rho$ is
    comparable to the linear case; when the curvature is predominantly
    orthogonal to $\hat{\mathbf{v}}$, the denominator grows while the
    numerator does not, suppressing $\rho$.
    \end{remark}

  \subsection{Dynamics of ES on Quadratic Landscape}
  
  \begin{proposition}[Dynamics of ES on Quadratic Landscape] \label{prop:quad_landscape_dynamics}
    Consider the simplified evolution
    \[
    \theta_{t+1} = H\theta_t + \eta_t, \qquad \eta_t \sim \mathcal{N}\!\left(0,\tfrac{\alpha^2}{N}I_d\right),
    \]
    where $H = I - \frac{\alpha\sigma}{\sigma_R}Q = \sum_k \left(1 - \frac{\alpha\sigma}{\sigma_R}\lambda_k\right)\mathbf{u}_k\mathbf{u}_k^\top$. 
    \footnote{$\sigma_R$ here is a fixed constant, not changing across iterations as in the full ES update (Eq. \ref{eq:es-quadratic-sigmaR}).}
    Let $\gamma_k := 1 - \frac{\alpha\sigma}{\sigma_R}\lambda_k$ denote the contraction factor along eigendirection $\mathbf{u}_k$. Then the projected mean and variance evolve as
    \[
    \mathbb{E}[\mathbf{u}_k^\top\theta_t] = \gamma_k^t\,\mathbf{u}_k^\top\theta_0,
    \]
    \[
    \mathrm{Var}[\mathbf{u}_k^\top\theta_t] = \frac{\alpha^2}{N}\cdot\begin{cases} \dfrac{1 - \gamma_k^{2t}}{1-\gamma_k^2}, & \gamma_k^2 \neq 1, \\[6pt] t, & \gamma_k^2 = 1. \end{cases}
    \]
    The process converges in mean and variance along direction $\mathbf{u}_k$ if and only if $|\gamma_k| < 1$, i.e.,
    \[
    0 < \lambda_k < \frac{2\sigma_R}{\alpha\sigma}.
    \]
    In this regime, the asymptotic variance is
    \[
    \lim_{t\to\infty}\mathrm{Var}[\mathbf{u}_k^\top\theta_t] = \frac{\alpha^2}{N}\cdot\frac{1}{1-\gamma_k^2},
    \]
    and the convergence timescale is $\tau_k = -\log 2\,/\log\gamma_k^2$. Proof in \ref{appsec:quad_landscape_dynamics}.
    \end{proposition}
    
    \begin{remark}[Stability conditions]
    The three failure modes each correspond to a violation of $|\gamma_k|<1$. Flat directions ($\lambda_k = 0$) satisfy $\gamma_k = 1$, so both the mean and variance grow as a pure random walk: the optimizer drifts diffusively and never localizes. Concave directions ($\lambda_k < 0$) give $\gamma_k > 1$, causing exponential divergence of both mean and variance --- the landscape actively repels the iterates. Overly sharp convex directions ($\lambda_k \geq 2\sigma_R/\alpha\sigma$) give $|\gamma_k| \geq 1$ for a different reason: the ES step overshoots, analogous to gradient descent with too large a learning rate. Note that because the ES learning rate is $\alpha = \sigma/2$, this threshold becomes $\lambda_k = 4\sigma_R/\sigma^2$.
    \end{remark}
    
    \begin{remark}[Optimal curvature]
    Within the stable regime, the convergence speed and asymptotic variance are jointly controlled by $\gamma_k^2$. Directions with $\lambda_k \approx \sigma_R/\alpha\sigma$ achieve $\gamma_k \approx 0$, which simultaneously minimizes the convergence timescale ($\tau_k \to 1$ iteration) and the asymptotic variance ($\approx \alpha^2/N$). Directions near the stability boundary $\lambda_k \to 0^+$ or $\lambda_k \to (2\sigma_R/\alpha\sigma)^-$ suffer slow convergence and inflated variance. Thus the landscape curvature relative to the ratio $\sigma_R/\alpha\sigma$ plays the role of an effective condition number for the ES dynamics.
    \end{remark}

    \begin{remark}[Role of noise scale]
    The asymptotic variance $\frac{\alpha^2}{N(1-\gamma_k^2)}$ reveals the expected tension between exploration and precision: increasing the population size $N$ suppresses variance linearly, while the landscape geometry (through $\gamma_k^2$) sets a floor on the minimum achievable uncertainty. In well-conditioned directions, $1 - \gamma_k^2 \approx 2\frac{\alpha\sigma}{\sigma_R}\lambda_k$, so the asymptotic variance scales as $\frac{\alpha\sigma_R}{2N\sigma\lambda_k}$, inversely proportional to both curvature and the signal-to-noise ratio $\sigma/\sigma_R$.
    \end{remark}
\subsection{Dynamics of GD on Quadratic Landscape}
  \begin{proposition}[Dynamics of GD on Quadratic Landscape] \label{prop:gd_quad_dynamics}
Let $R(\theta) = -\frac{1}{2}\theta^\top Q\theta$ with $Q \succeq 0$,
and let gradient descent with learning rate $\beta$ iterate as
$\theta_{t+1} = (I - \beta Q)\,\theta_t$.
Define $H_{\mathrm{GD}} := I - \beta Q$ and let
$\gamma_k^{\mathrm{GD}} := 1 - \beta\lambda_k$. The iterate is
\[
    \theta_t = H_{\mathrm{GD}}^t\,\theta_0
    = \sum_{k=1}^{d}(\gamma_k^{\mathrm{GD}})^t\,(\mathbf{u}_k^\top\theta_0)\,\mathbf{u}_k.
\]
Direction $\mathbf{u}_k$ converges if and only if $|\gamma_k^{\mathrm{GD}}| < 1$, i.e.,
$0 < \lambda_k < 2/\beta$.
In this regime the convergence timescale is
$\tau_k^{\mathrm{GD}} = -\log 2\,/\,\log(\gamma_k^{\mathrm{GD}})^2$,
and the residual error is exactly zero: $\lim_{t\to\infty}\mathbf{u}_k^\top\theta_t = 0$. Proof in \ref{appsec:gd_quad_landscape}.
\end{proposition}

\begin{remark}[Stability conditions for GD]
The three failure modes mirror those of the ES dynamics
(Proposition~\ref{prop:quad_landscape_dynamics}):
flat directions ($\lambda_k = 0$) give $\gamma_k^{\mathrm{GD}} = 1$, so the iterate
is frozen at $\mathbf{u}_k^\top\theta_0$ for all $t$ --- unlike ES, where the iterate
drifts diffusively.
Concave directions ($\lambda_k < 0$) give $\gamma_k^{\mathrm{GD}} > 1$, causing
exponential divergence.
Overly sharp convex directions ($\lambda_k \geq 2/\beta$) give
$|\gamma_k^{\mathrm{GD}}| \geq 1$, and the iterate overshoots with growing or
non-decaying oscillations.
\end{remark}

\subsection{GD displacement is low-dimensional}
\begin{proposition}[GD displacement is low-dimensional] \label{prop:gd_displacement}
Starting from $\theta_0$ on a rank-$r$ quadratic landscape ($r \ll d$), the GD displacement
\[
    \theta_{\mathrm{GD}} - \theta_0
    = \sum_{k:\,\lambda_k > 0}\bigl[(1-\beta\lambda_k)^T - 1\bigr](\mathbf{u}_k^\top\theta_0)\,\mathbf{u}_k
\]
lies entirely within the $r$-dimensional column space of $Q$, with
$\|\theta_{\mathrm{GD}} - \theta_0\|^2 = \mathcal{O}(r)$.
\end{proposition}

\subsection{ES displacement decomposes into signal and diffusion}
\begin{proposition}[ES displacement decomposes into signal and diffusion] \label{prop:es_displacement}
The ES displacement from initialization decomposes as
\[
    \theta_{\mathrm{ES}} - \theta_0
    = \underbrace{(H_{\mathrm{ES}}^T - I)\theta_0}_{\text{signal (on-manifold)}}
    + \underbrace{\sum_{t=1}^{T}H_{\mathrm{ES}}^{t-1}\,\eta_{T-t}}_{\text{noise (all directions)}}.
\]
The signal term lies within the column space of $Q$, identical in structure to the
GD displacement with $\beta \leftrightarrow \alpha\sigma/\sigma_R$.
The noise term has support in all $d$ dimensions. For $r \ll d$ and large $T$,
\[
    \mathbb{E}\|\theta_{\mathrm{ES}} - \theta_0\|^2
    = \underbrace{\mathcal{O}(r)}_{\text{on-manifold}}
    + \underbrace{\frac{\alpha^2 T(d-r)}{N}}_{\text{off-manifold diffusion}}
    \approx \frac{\alpha^2 T d}{N}.
\]
\end{proposition}

\subsection{ES--GD difference is off-manifold}
\begin{proposition}[ES--GD difference is off-manifold] \label{prop:es_gd_difference}
When both methods have converged on the active directions
($(\gamma_k^{\mathrm{ES}})^T \approx (\gamma_k^{\mathrm{GD}})^T \approx 0$),
the expected squared difference between the two solutions is dominated by
off-manifold diffusion:
\[
    \mathbb{E}\|\theta_{\mathrm{ES}} - \theta_{\mathrm{GD}}\|^2
    \approx \frac{\alpha^2 T(d-r)}{N}
    = \mathcal{O}(d).
\]
\end{proposition}

\begin{remark}[Geometric summary] \label{rmk:geometric_summary}
The three pairwise squared distances satisfy the hierarchy (for $r \ll d$, large $T$):
\begin{align*}
    \|\theta_{\mathrm{GD}} - \theta_0\|^2 &= \mathcal{O}(r), \\
    \mathbb{E}\|\theta_{\mathrm{ES}} - \theta_0\|^2 &= \mathcal{O}(r) + \mathcal{O}(d), \\
    \mathbb{E}\|\theta_{\mathrm{ES}} - \theta_{\mathrm{GD}}\|^2 &= \mathcal{O}(d).
\end{align*}
The GD displacement is low-dimensional and task-aligned.
The ES displacement superposes a comparable on-manifold component with a much larger
isotropic off-manifold diffusion.
The ES--GD difference is dominated by the off-manifold term alone, since both methods
converge to the same optimum within the active subspace.
This accounts for the empirical observation that ES achieves similar task performance
to GD (shared on-manifold signal) while being far from it in parameter space
(off-manifold diffusion).
\end{remark}

\begin{remark}[Cosine similarity between displacements] \label{rmk:cosine_similarity}
Since $\theta_{\mathrm{GD}} - \theta_0$ lies within the column space of $Q$
and the ES noise is zero-mean, the expected inner product picks up only the
on-manifold signal of ES. For large $T$ with both converged on active directions,
the expected cosine similarity satisfies
\[
    \mathbb{E}\left[\cos\angle(\theta_{\mathrm{ES}} - \theta_0,\;
    \theta_{\mathrm{GD}} - \theta_0)\right]
    \sim \mathcal{O}\!\left(\sqrt{r/d}\right),
\]
which is small for $r \ll d$: the two displacement vectors are nearly orthogonal
despite leading to similar task performance.
\end{remark}

\clearpage
\section{Detailed Validation of Theoretical Predictions} 
\label{app:detailed-results}
\input{appendix_weight_drift_RW_stats}

\clearpage
\subsection{Testing the Cosine-Similarity Prediction}
\label{sec:cosine-verification}

Remark \ref{rmk:cosine_similarity} predicts that the expected cosine similarity between the ES and GRPO update directions scales as $\mathbb{E}[\cos\angle(\theta_{\text{ES}}-\theta_0,\, \theta_{\text{GRPO}}-\theta_0)] \sim \mathcal{O}(\sqrt{r/d})$, where $r$ is the task-relevant (GD-aligned) rank and $d$ is the full parameter dimension. This prediction was previously verified only indirectly, through its consequences for update norms and mode connectivity. Here we report the direct empirical measurement.

For each sequential checkpoint, we compute the cosine similarity between the full parameter-space update vectors $\Delta\theta_{\text{ES}} = \theta_{\text{ES}} - \theta_{\text{base}}$ and $\Delta\theta_{\text{GRPO}} = \theta_{\text{GRPO}} - \theta_{\text{base}}$, using float64 key-wise accumulation over the full parameter vector ($d \approx 3\text{--}4 \times 10^9$ for Llama-3.2-3B and Qwen3-4B respectively). Table~\ref{tab:cosine-verification} reports the results for both model families.

\begin{table}[H]
\centering
\caption{Empirical cosine similarity between ES and GRPO update directions at each sequential checkpoint.}
\label{tab:cosine-verification}

\begin{tabular}{lcc}
\toprule
\textbf{Stage} & \textbf{Qwen3-4B $\cos$} & \textbf{Llama-3.2-3B $\cos$} \\
\midrule
T0 Countdown & $1.0 \times 10^{-4}$ & $0.8 \times 10^{-4}$ \\
T1 Math      & $1.6 \times 10^{-4}$ & $1.0 \times 10^{-4}$ \\
T2 Chemistry & $2.2 \times 10^{-4}$ & $1.0 \times 10^{-4}$ \\
T3 BoolQ     & $2.7 \times 10^{-4}$ & $1.3 \times 10^{-4}$ \\
\bottomrule
\end{tabular}
\end{table}

All measured cosines are $\sim 10^{-4}$, corresponding to almost orthogonal updates. 
We noticed that, when plugging a curvature rank estimate $r\approx 100$ with parameter dimension $d \approx 3\text{--}4 \times 10^9$ into Remark \ref{rmk:cosine_similarity}, 
the predicted scale of the cosine $\sqrt{r/d}\approx 1.6 \times 10^{-4}$ will match the scale of the observed cosine values. 
A more principled way to estimate the rank $r$ of the curvature of the landscape is left for future work. 

  \clearpage
  \section{Detailed derivations of theoretical results}
  
  \subsection{Z-score ES Update Setup}
  We detail the ES update rule as implemented in LLM fine-tuning. The update at step $t$ is
  \begin{equation}
    \theta_{t+1} = \theta_{t} + \alpha \cdot \frac{1}{N} \sum_{i=1}^{N} Z_{i} \epsilon_{i}
  \end{equation}
  where each $\epsilon_{i} \sim \mathcal{N}(0, I_{d})$ is an independent Gaussian perturbation, and $Z_{i} = \frac{R_{i} - \mu_{R}}{\sigma_{R}}$ is the z-scored return for the $i$-th sample. Here, $R_{i}$ is the reward, with $\mu_{R}$ and $\sigma_{R}$ denoting the mean and standard deviation of the $R_{i}$ across the population.
  With slight abuse of notation we let $R(\theta_{t} + \sigma \epsilon_{i})$ be the true or expected reward at the perturbed location, and $R_i$ the observed reward with certain observation noise.
  The step size is set to $\alpha = \sigma/2$, and a population size of $N = 30$ is typical. Defining the update increment as
  \begin{equation}
    \Delta\theta = \alpha \cdot \frac{1}{N} \sum_{i=1}^{N} Z_{i} \epsilon_{i}
  \end{equation}
  the algorithm iterates as $\theta_{t+1} = \theta_{t} + \Delta\theta$.
  
  \subsection{Flat Landscape Analysis}\label{appsec:flat_RW_scaling}

  \begin{proof}[Proof of Proposition~\ref{prop:flat_landscape_scaling}]
  The key assumption is that $Z_{i}$ is independent of $\epsilon_{i}$,
  i.e., $Z_{i}\perp\epsilon_{i}$. Then, conditioned on a set of $Z_{i}$,
  \begin{equation}
    \mathrm{Var}\left[\Delta\theta \,\middle|\, Z_{i}\right] = \left( \frac{\alpha^{2}}{N^{2}}\sum_{i=1}^N Z_i^2 \right) I_d
  \end{equation}
  and the distribution is isotropic Gaussian:
  \begin{equation}
  \Delta\theta|Z_{i} \sim \mathcal{N}\left(0, \left(\frac{\alpha^{2}}{N^{2}}\sum_{i=1}^{N}Z_{i}^{2}\right)I_{d}\right)
  \end{equation}

  By definition,\footnote{Note that this constraint depends on the implementation of the z-score
  $Z_{i} = \frac{R_{i}-\mu_{R}}{\sigma_{R}}$. In our code, we compute
  the standard deviation as $\sigma_{R}^{2} = \frac{1}{N}\sum_{i=1}^{N}(R_{i}-\mu_{R})^{2}$.
  In the unbiased variance estimator (denominator $N-1$), the identity reads $\sum_{i=1}^{N}Z_{i}^{2} = N-1$.}
  \begin{align}
  \sum_{i=1}^{N}Z_{i}^{2}
  &= \sum_{i=1}^{N} \left(\frac{R_{i}-\mu_{R}}{\sigma_{R}}\right)^{2} \notag\\
  &= \frac{1}{\sigma_{R}^{2}}\sum_{i=1}^{N}(R_{i} - \mu_{R})^{2} \notag\\
  &= N\quad\text{(or $N-1$ with unbiased estimator)}
  \end{align}

  Therefore, in general the scaling is:
  \begin{align}
    \mathrm{Var}[\Delta\theta|Z_{i}] &= \frac{\alpha^2}{N^2} N\, I_d \notag \\
    &= \frac{\alpha^2}{N} I_d
  \end{align}
  Since this is independent of $Z_i$, it follows after marginalizing:
  \begin{equation}
    \Delta\theta \sim \mathcal{N} \left( 0, \frac{\alpha^2}{N} I_d \right)
  \end{equation}
  \begin{equation}
    \mathrm{Var}[\Delta\theta] = \frac{\alpha^2}{N} I_d
  \end{equation}

  On a flat loss, $\Delta\theta$ is independent between steps, so after $T$ steps,
  the combined weight change is
  \begin{equation}
    \bar{\Delta\theta}_T = \theta_T - \theta_0 = \sum_{t=1}^T \Delta\theta_t
  \end{equation}
  and the total variance is
  \begin{equation}
    \mathrm{Var}[\bar{\Delta\theta}_T] = \frac{\alpha^2 T}{N} I_d
  \end{equation}
  The expected squared norm of total deviation is
  \begin{equation}
    \mathbb{E}\|\bar{\Delta\theta}_T\|^2 = \frac{\alpha^2 T d}{N}
  \end{equation}
  Finally, substituting $\alpha = \frac{\sigma}{2}$ gives
  \begin{equation}
    \frac{\alpha^2}{N} = \frac{\sigma^2}{4N}
  \end{equation}
  and
  \begin{equation}
    \mathbb{E}\|\bar{\Delta\theta}_T\|^2 = \frac{\sigma^2 T d}{4N}
  \end{equation}
  \end{proof}

  \subsection{Linear Landscape Analysis}\label{appsec:linear_landscape_step}
  \begin{proof}[Proof of Proposition~\ref{prop:linear_landscape_step}]
  Consider a linear landscape with i.i.d.\ observational noise $\xi\sim\mathcal{N}(0,1)$,
  assumed independent of $\epsilon$, i.e., $\xi\perp\epsilon$.
  The linear reward is
  \begin{equation}
    R(\theta) = -\mathbf{v}^\top\theta
  \end{equation}
  and the observed reward for each sample is
  \begin{align}
    R_i &= R(\theta_0+\sigma\epsilon_i) + \sigma_\xi \xi_i \notag \\
        &= -\mathbf{v}^\top(\theta_0+\sigma \epsilon_i) + \sigma_\xi \xi_i \notag \\
        &\equiv -\sigma\,\mathbf{v}^\top\epsilon_i + \sigma_\xi \xi_i
  \end{align}
  
  Decompose the exploration as
  \begin{align}
    \epsilon_i = \epsilon_i^\perp + \epsilon_i^\parallel
  \end{align}
  where
  \begin{align}
    P^\perp &:= I - \frac{\mathbf{v}\mathbf{v}^\top}{\|\mathbf{v}\|^2} \\
    P^\parallel &:= \frac{\mathbf{v}\mathbf{v}^\top}{\|\mathbf{v}\|^2}
  \end{align}
  so that
  $\epsilon_i^\perp := P^\perp \epsilon_i$, $\epsilon_i^\parallel := P^\parallel \epsilon_i$.
  Then $\epsilon_i^\perp \sim \mathcal{N}(0, P^\perp)$, $\epsilon_i^\parallel \sim \mathcal{N}(0, P^\parallel)$,
  and $\epsilon_i^\perp\perp\epsilon_i^\parallel$.

  The reward is
  \begin{align}
    R_i &= -\sigma \mathbf{v}^\top \epsilon_i^\parallel + \sigma_\xi \xi_i
  \end{align}
  and thus $R_i \perp \epsilon_i^\perp$.

  The update is:
  \begin{equation}
    \Delta\theta = \alpha \cdot \frac{1}{N} \sum_{i=1}^N Z_i \epsilon_i
  \end{equation}
  and can be split as
  \begin{equation}
    \Delta\theta = P^\perp \Delta\theta + P^\parallel \Delta\theta
  \end{equation}

  \paragraph{Off-manifold component}
  \begin{align}
    P^\perp \Delta\theta &= \alpha \cdot \frac{1}{N} \sum_{i=1}^N Z_i P^\perp \epsilon_i \notag \\
    &= \alpha \cdot \frac{1}{N} \sum_{i=1}^N Z_i \epsilon_i^\perp
  \end{align}
  
  Conditioned on the observed normalized reward $Z_{i}$, $P^{\perp}\Delta\theta$
  is a weighted sum of Gaussian random vectors, thus Gaussian:
  \begin{equation}
    P^\perp\Delta\theta \mid \{Z_i\}_{i=1}^N \sim \mathcal{N}\left(
      0, \frac{\alpha^2}{N^2} \sum_{i=1}^N Z_i^2 P^\perp
    \right)
  \end{equation}
  
  Since $\{Z_{i}\}_{i=1}^{N}$ are z-scored $\{R_{i}\}_{i=1}^{N}$, the sum satisfies the following identity.%
  \begin{equation}
    P^\perp\Delta\theta \mid \{Z_i\}_{i=1}^N \sim \mathcal{N}\left(
      0, \frac{\alpha^2}{N} P^\perp
    \right)
  \end{equation}
  
  Because of the following equations:
  \begin{align}
  \sigma_{R}^{2} &= \frac{1}{N}\sum_{i=1}^{N}(R_{i}-\mu_{R})^{2}\\
  Z_{i} &= \frac{R_{i}-\mu_{R}}{\sigma_{R}}
  \end{align}
  
  Thus, the conditional distribution is
  \begin{equation}
  P^{\perp}\Delta\theta\,\big|\;\{Z_{i}\}_{i=1}^{N} \sim \mathcal{N}\left(0, \frac{\alpha^{2}}{N^{2}} N P^{\perp}\right)
  \end{equation}
  
  Since the conditional distribution $P^{\perp}\Delta\theta\,\big|\;\{Z_{i}\}_{i=1}^{N}$
  is independent of the observed reward, marginalizing over $Z_{i}$ we have
  \begin{equation}
  P^{\perp}\Delta\theta \sim \mathcal{N}\left(0,\frac{\alpha^{2}}{N}P^{\perp}\right)
  \end{equation}
  
  The norm of the off-manifold component is
  \begin{align}
    \mathbb{E}\|P^\perp\Delta\theta\|^2 &= \mathrm{Tr}\left[ \frac{\alpha^2}{N} P^\perp \right] \notag \\
    &= \frac{\alpha^2}{N} (d-1)
  \end{align}

  \paragraph{On-manifold component}
  
  \begin{align}
  P^{\parallel}\Delta\theta &= \alpha\cdot\frac{1}{N}\sum_{i=1}^{N}Z_{i}P^{\parallel}\epsilon_{i} \notag\\
  &= \alpha\cdot\frac{1}{N}\sum_{i=1}^{N}Z_{i}\epsilon_{i}^{\parallel}
  \end{align}
  
  The on-manifold variable can be expressed as $\epsilon_{i}^{\parallel} = \frac{\mathbf{v}}{\|\mathbf{v}\|} n_{i}$,
  with $n_{i} \sim \mathcal{N}(0,1)$, thus the reward can be written as 
  \begin{align}
  R_{i} &= -\sigma\mathbf{v}^{\top} \epsilon_{i}^{\parallel} + \sigma_{\xi}\xi_{i} \notag\\
  &= -\sigma \frac{\mathbf{v}^{\top}\mathbf{v}}{\|\mathbf{v}\|} n_{i} + \sigma_{\xi}\xi_{i} \notag\\
  &= -\sigma\|\mathbf{v}\| n_{i} + \sigma_{\xi}\xi_{i}
  \end{align}
  
  Recall the definition of observed reward and z-score, 
  \begin{equation}
  Z_{i} = \frac{R_{i} - \mu_{R}}{\sigma_{R}}
  \end{equation}
  
  To simplify derivation, we use the population expectation and standard
  deviation here: 
  \begin{align}
  \mu_{R} &\approx \mathbb{E}[R_{i}] = 0\\
  \sigma_{R}^{2} &\approx \mathrm{Var}[R_{i}] = \sigma^{2}\|\mathbf{v}\|^{2} + \sigma_{\xi}^{2}
  \end{align}
  
  Under these simplifications,
  \begin{align}
  P^{\parallel}\Delta\theta &= \frac{\alpha}{N} \sum_{i=1}^{N} Z_{i} \epsilon_{i}^{\parallel} \notag\\
  &= \frac{\alpha}{N} \sum_{i=1}^{N} \frac{-\sigma\|\mathbf{v}\| n_{i} + \sigma_{\xi}\xi_{i} - \mu_{R}}{\sigma_{R}} n_{i} \frac{\mathbf{v}}{\|\mathbf{v}\|} \notag\\
  &= \frac{\alpha}{N\sigma_{R}}\sum_{i=1}^{N} (-\sigma\|\mathbf{v}\| n_{i} + \sigma_{\xi}\xi_{i}) n_{i} \frac{\mathbf{v}}{\|\mathbf{v}\|}
  \end{align}
  
  The expectation reads
  \begin{align}
  \mathbb{E}[P^{\parallel}\Delta\theta] =&~ \frac{\alpha}{N\sigma_{R}}\sum_{i=1}^{N} \mathbb{E}\left[ (-\sigma\|\mathbf{v}\| n_{i} + \sigma_{\xi}\xi_{i}) n_{i} \right] \frac{\mathbf{v}}{\|\mathbf{v}\|} \notag\\
  =&~\frac{\alpha}{N\sigma_{R}}\mathbb{E}[ -\sigma\|\mathbf{v}\| \sum_{i=1}^{N} n_{i}^{2} + \sigma_{\xi} \sum_{i=1}^{N}\xi_{i} n_{i} ]\frac{\mathbf{v}}{\|\mathbf{v}\|} \notag\\
  =&~\frac{\alpha}{N\sigma_{R}}(-\sigma\|\mathbf{v}\|N)\frac{\mathbf{v}}{\|\mathbf{v}\|} \notag\\
  =&~\frac{-\sigma\alpha}{\sigma_{R}}\mathbf{v} \notag\\
  =&~\frac{-\sigma\alpha\mathbf{v}}{\sqrt{\sigma^{2}\|\mathbf{v}\|^{2} + \sigma_{\xi}^{2}}}
  \end{align}
  where we used independence, $\mathbb{E}[\xi_{i} n_{i}]=0$, and $\mathbb{E}[\sum_{i=1}^{N} n_{i}^{2}] = N$. 
  
  Next, consider the variance. Consider the individual term, $X_{i} = -\sigma\|\mathbf{v}\| n_{i}^{2} + \sigma_{\xi}\xi_{i} n_{i}$,
  and its variance:
  \begin{align}
  \mathrm{Var}[X_{i}] =&~\mathbb{E}[(-\sigma\|\mathbf{v}\| n_{i}^{2} + \sigma_{\xi}\xi_{i} n_{i})^{2}]
  - \mathbb{E}[-\sigma\|\mathbf{v}\| n_{i}^{2} + \sigma_{\xi} \xi_{i} n_{i}]^{2} \notag\\
  =&~\mathbb{E}[ \sigma^{2}\|\mathbf{v}\|^{2} \underbrace{n_{i}^{4}}_{3}
  - 2\sigma\|\mathbf{v}\|\sigma_{\xi} \underbrace{\xi_{i}}_{0} \underbrace{n_{i}^{3}}_{0}
  + \sigma_{\xi}^{2} \underbrace{\xi_{i}^{2}}_{1} \underbrace{n_{i}^{2}}_{1}
  ] - (\sigma\|\mathbf{v}\|)^{2} \notag\\
  =&~3\sigma^{2}\|\mathbf{v}\|^{2} + \sigma_{\xi}^{2} - (\sigma\|\mathbf{v}\|)^{2} \notag\\
  =&~2\sigma^{2}\|\mathbf{v}\|^{2} + \sigma_{\xi}^{2}
  \end{align}
  
  Therefore,
  \begin{align}
  \mathrm{Var}[P^{\parallel}\Delta\theta] &= \frac{\alpha^{2}}{N^{2}\sigma_{R}^{2}}N\,\mathrm{Var}[X_{i}] \frac{\mathbf{v}\mathbf{v}^{\top}}{\|\mathbf{v}\|^{2}}
  \notag\\
  &= \frac{ \alpha^{2} (2\sigma^{2}\|\mathbf{v}\|^{2} + \sigma_{\xi}^{2}) }{ N \sigma_{R}^{2} } P^{\parallel} \notag\\
  &= \frac{ \alpha^{2} (2\sigma^{2}\|\mathbf{v}\|^{2} + \sigma_{\xi}^{2}) }{ N (\sigma^{2}\|\mathbf{v}\|^{2} + \sigma_{\xi}^{2}) } P^{\parallel}
  \end{align}
  
  \begin{align}
  \mathbb{E}[\|P^{\parallel}\Delta\theta\|^{2}] =&~ \mathrm{Tr}\left[\mathrm{Var}[P^{\parallel}\Delta\theta]\right] + \mathbb{E}[P^{\parallel}\Delta\theta]^{\top} \mathbb{E}[P^{\parallel}\Delta\theta] \notag\\
  =&~\frac{ \alpha^{2} (2\sigma^{2}\|\mathbf{v}\|^{2} + \sigma_{\xi}^{2}) }{ N ( \sigma^{2}\|\mathbf{v}\|^{2} + \sigma_{\xi}^{2} ) } + \frac{ (\sigma\alpha)^{2} \|\mathbf{v}\|^{2} }{ \sigma^{2}\|\mathbf{v}\|^{2} + \sigma_{\xi}^{2} } \notag\\
  =&~\frac{\alpha^{2}}{N} + \left(1 + \frac{1}{N}\right) \frac{ \sigma^{2}\|\mathbf{v}\|^{2} \alpha^{2} }{ \sigma^{2}\|\mathbf{v}\|^{2} + \sigma_{\xi}^{2} }
  \end{align}

  \paragraph{Overall distribution of the update}
Since the off-manifold component $P^{\perp}\Delta\theta$ is independent of the rewards
and hence of $P^{\parallel}\Delta\theta$, we can combine the on- and off-manifold results
to obtain the full single-step statistics. Writing $\sigma_{R}^{2} := \sigma^{2}\|\mathbf{v}\|^{2}+\sigma_{\xi}^{2}$,
\begin{align}
  \mathbb{E}[\Delta\theta] &= \frac{-\sigma\alpha\,\mathbf{v}}{\sigma_{R}}, \\
  \mathrm{Cov}[\Delta\theta]
  &= \frac{\alpha^{2}}{N}\,I_{d}
     + \frac{\alpha^{2}\sigma^{2}}{N\sigma_{R}^{2}}\,\mathbf{v}\mathbf{v}^\top.
\end{align}
The covariance has the same isotropic baseline $\frac{\alpha^2}{N}I_d$ as the flat landscape,
plus a rank-1 boost along the gradient direction. 
The total expected squared norm is
\begin{align}
\mathbb{E}\|\Delta\theta\|^{2}
&= \frac{\sigma^{2}\alpha^{2}\|\mathbf{v}\|^{2}}{\sigma_{R}^{2}}
 + \frac{\alpha^{2}d}{N}
 + \frac{\alpha^{2}\sigma^{2}\|\mathbf{v}\|^{2}}{N\,\sigma_{R}^{2}}.
\end{align}

\end{proof}
  
  \paragraph{Manifold Alignment of Update}
  
  Consider the ratio of the on- and off-manifold movement:
  \begin{equation}
  \mathbb{E}\left[ \|P^{\parallel}\Delta\theta\|^{2} \right] = \frac{\alpha^{2}}{N} + \left( 1 + \frac{1}{N} \right) \frac{ \sigma^{2}\|\mathbf{v}\|^{2}\alpha^{2} }{ \sigma^{2}\|\mathbf{v}\|^{2} + \sigma_{\xi}^{2} }
  \end{equation}
  \begin{equation}
  \mathbb{E}\|P^{\perp}\Delta\theta\|^{2} = \frac{\alpha^{2}(d-1)}{N}
  \end{equation}
  
  The fraction of on-manifold squared deviation in the total deviation is:
  \begin{align}
  \frac{ \mathbb{E}\left[\|P^{\parallel}\Delta\theta\|^2\right] }{ \mathbb{E}\left[\|\Delta\theta\|^2\right] }
  &= \frac{ \mathbb{E}\left[\|P^{\parallel}\Delta\theta\|^{2}\right] }{ \mathbb{E}\left[\|P^{\parallel}\Delta\theta\|^{2}\right] + \mathbb{E}\|P^{\perp}\Delta\theta\|^{2} } \notag\\
  &= \frac{ \frac{\alpha^{2}}{N} + \left(1 + \frac{1}{N}\right) \frac{ \sigma^{2} \|\mathbf{v}\|^{2} \alpha^{2} }{ \sigma^{2} \|\mathbf{v}\|^{2} + \sigma_{\xi}^{2} } }
    { \frac{ \alpha^{2}(d-1)}{N} + \frac{\alpha^{2}}{N} + \left(1 + \frac{1}{N}\right) \frac{ \sigma^{2}\|\mathbf{v}\|^{2} \alpha^{2} }{ \sigma^{2}\|\mathbf{v}\|^{2} + \sigma_{\xi}^{2} } } \notag\\
  &= \frac{ \frac{1}{N} + \frac{N+1}{N} \frac{ \sigma^{2} \|\mathbf{v}\|^{2} }{ \sigma^{2} \|\mathbf{v}\|^{2} + \sigma_{\xi}^{2} } }
    { \frac{d}{N} + \frac{N+1}{N} \frac{ \sigma^{2} \|\mathbf{v}\|^{2} }{ \sigma^{2}\|\mathbf{v}\|^{2} + \sigma_{\xi}^{2} } } \notag\\
  &= \frac{ 1 + (N+1) \frac{ \sigma^{2}\|\mathbf{v}\|^{2} }{ \sigma^{2}\|\mathbf{v}\|^{2} + \sigma_{\xi}^{2} } }
           { d + (N+1)\frac{ \sigma^{2}\|\mathbf{v}\|^{2} }{ \sigma^{2}\|\mathbf{v}\|^{2} + \sigma_{\xi}^{2} } }
  \end{align}
  
  where $s:=\frac{\sigma^{2}\|\mathbf{v}\|^{2}}{\sigma^{2}\|\mathbf{v}\|^{2}+\sigma_{\xi}^{2}}$. 
  
  This formulation reveals the key scaling relation of on and off manifold
  movement 
  \begin{itemize}
  \item Signal to noise ratio $s:=\frac{\sigma^{2}\|\mathbf{v}\|^{2}}{\sigma^{2}\|\mathbf{v}\|^{2}+\sigma_{\xi}^{2}}$
  , which range from $s\in[0,1]$
  \begin{itemize}
  \item When no gradient exists, $\|\mathbf{v}\|^{2}=0$, then independent
  of number of sample $N$, $\rho=1/d$. i.e. the norm on one axis among
  $d$ axis is $1/d$.
  \item Similarly, when $\sigma_{\xi}^{2}\gg\sigma^{2}\|\mathbf{v}\|^{2}$
  i.e. the reward noise variance is much higher than the variation of
  signal under exploration range $\sigma$, then it's effectively no
  gradient and no guidance. 
  \item In the other extreme, when very little reward noise, $\sigma_{\xi}^{2}\ll\sigma^{2}\|\mathbf{v}\|^{2}$,
  $s\approx1$, then $\rho=\frac{N+2}{d+N+1}$. 
  \item Via differentiation, one can show, $\rho$ monotonically depends on
  signal to noise ratio $s$
  \end{itemize}
  \item Full dimension $d$. Increase the full space dim $d$ will monotonically
  decrease the on manifold fraction $\rho$. 
  \item Sample size per generation $N$. Increase the sample size $N$ will
  monotonically increase the on manifold fraction $\rho$. 
  \begin{itemize}
  \item Note this scaling is almost linear, so to get half of the variance
  on manifold, one need $(N+1)s\approx d$. In reality, $d\sim10^{9}$
  , which makes this impossible. 
  \item In most realistic regimes $N\ll d$, which makes this on manifold
  ratio per step nearly negligible. 
  \end{itemize}
  \end{itemize}
  
  \clearpage
  \subsection{Quadratic landscape analysis} \label{appsec:quad_landscape_step}
  \begin{proof}[Proof of Proposition~\ref{prop:quad_landscape_step}]
  Consider a quadratic potential centered, i.e. peaked at the origin $\theta^{*}=0$.
  \begin{equation}
  R(\theta)=-\frac{1}{2}\theta^{\top}Q\theta
  \end{equation}
  At current state $\theta_{0}$, the reward at perturbed locations
  with observation noise are, where $\xi_{i}\sim\mathcal{N}(0,1)$,
  $\epsilon_{i}\sim\mathcal{N}(0,I_{d})$
  \begin{align}
  R_{i} & =R(\theta_{0}+\sigma\epsilon_{i})+\sigma_{\xi}\xi_{i}\notag\\
   & =-\frac{1}{2}(\theta_{0}+\sigma\epsilon_{i})^{\top}Q(\theta_{0}+\sigma\epsilon_{i})+\sigma_{\xi}\xi_{i}\notag\\
   & =-\frac{1}{2}\theta_{0}^{\top}Q\theta_{0}-\sigma\,\theta_{0}^{\top}Q\epsilon_{i}-\frac{\sigma^{2}}{2}\epsilon_{i}^{\top}Q\epsilon_{i}+\sigma_{\xi}\xi_{i}
  \end{align}
  Expectation of reward reads, which is the population mean of reward
  samples 
  \begin{equation}
  \mathbb{E}[R_{i}]=-\frac{1}{2}\theta_{0}^{\top}Q\theta_{0}-\frac{\sigma^{2}}{2}\mathrm{Tr}Q
  \end{equation}
  
  Center the rewards with population mean, we define 
  \begin{equation}
  \tilde{R}_{i}:=R_{i}-\mathbb{E}[R_{i}]
  \end{equation}
  
  Centered reward has three terms, linear $A$, quadratic \textbf{$B$},
  and noise $C$ 
  \begin{equation}
  \tilde{R}_{i}=\underbrace{-\sigma\,\theta_{0}^{\top}Q\epsilon_{i}}_{\text{linear }A}\underbrace{-\frac{\sigma^{2}}{2}(\epsilon_{i}^{\top}Q\epsilon_{i}-\mathrm{Tr}Q)}_{\text{quad }B}+\underbrace{\sigma_{\xi}\xi_{i}}_{\text{scalar noise }C}
  \end{equation}
  
  One can abbreviate $\mathbf{v}:=Q^{\top}\theta_{0}$ as the negative
  gradient. 
  
  The weight update reads, 
  \begin{align}
  \Delta\theta & =\alpha\cdot\frac{1}{N}\sum_{i=1}^{N}Z_{i}\epsilon_{i}\notag\\
   & =\frac{\alpha}{N\sigma_{R}}\sum_{i=1}^{N}\tilde{R}_{i}\epsilon_{i}
  \end{align}
  
  Since perturbations $\tilde{R}_{i}\epsilon_{i}$ are i.i.d. sampled
  \begin{equation}
  \mathbb{E}[\Delta\theta]=\frac{\alpha}{\sigma_{R}}\mathbb{E}[\tilde{R}_{i}\epsilon_{i}]
  \end{equation}
  and 
  \begin{equation}
  \mathrm{Cov}[\Delta\theta]=\frac{\alpha^{2}}{N\sigma_{R}^{2}}\mathrm{Cov}[\tilde{R}_{i}\epsilon_{i}]
  \end{equation}
  
  One observe that both $\tilde{R}_{i}$ and $\tilde{R}_{i}\epsilon_{i}$
  are \textbf{polynomial functions of gaussian random variables} $\epsilon_{i}$
  and $\xi_{i}$ (up to 3rd order), thus we can compute \textbf{their
  moment of any order} via Wick-Isserlis theorem \citep{wick1950evaluation,isserlis1918formula},
  including mean and covariance of them.
  
  Here are the specific calculations:
  
  \paragraph{Reward variance}
  
  Note that all three terms $A,B,C$ are centered 
  \begin{align}
  \mathbb{E}[A] & =0\notag\\
  \mathbb{E}[B] & =0\notag\\
  \mathbb{E}[C] & =0
  \end{align}
  
  Further $A\perp C$, $B\perp C$ due to independence of noise, and
  \begin{equation}
  \mathbb{E}[AB]=0
  \end{equation}
  
  as it only has polynomial of $\epsilon_{i}$ of odd order 1 and 3,
  thus its expectation is 0. Thus, $A,B,C$ are not correlated with
  one another. Thus, 
  \begin{align}
  \mathrm{Var}[\tilde{R}_{i}] & =\mathbb{E}[\tilde{R}_{i}^{2}]\notag\\
   & =\mathbb{E}[A^{2}]+\mathbb{E}[B^{2}]+\mathbb{E}[C^{2}]
  \end{align}
  
  all cross terms vanish due to no correlation.
  \begin{align}
  \mathbb{E}[A^{2}] & =\mathbb{E}[(-\sigma\,\theta_{0}^{\top}Q\epsilon_{i})\cdot(-\sigma\,\theta_{0}^{\top}Q\epsilon_{i})]\notag\\
   & =\sigma^{2}\theta_{0}^{\top}QQ^{\top}\theta_{0}\notag\\
   & =\sigma^{2}\|\mathbf{v}\|^{2}
  \end{align}
  \begin{equation}
  \mathbb{E}[C^{2}]=\sigma_{\xi}^{2}
  \end{equation}
  \begin{align}
  \mathbb{E}[B^{2}] & =\mathbb{E}[(\frac{\sigma^{2}}{2}(\epsilon_{i}^{\top}Q\epsilon_{i}-\mathrm{Tr}Q))^{2}]\notag\\
   & =\frac{\sigma^{4}}{4}\mathbb{E}[\epsilon_{i}^{\top}Q\epsilon_{i}\cdot(\epsilon_{i}^{\top}Q\epsilon_{i}-\mathrm{Tr}Q)-\mathrm{Tr}Q\cdot(\epsilon_{i}^{\top}Q\epsilon_{i}-\mathrm{Tr}Q)]\notag\\
   & =\frac{\sigma^{4}}{4}\Big(\mathbb{E}[\epsilon_{i}^{\top}Q\epsilon_{i}\cdot\epsilon_{i}^{\top}Q\epsilon_{i}]-(\mathrm{Tr}Q)^{2}\Big)
  \end{align}
  
  The first quintic polynomial of $\epsilon_{i}$ can be tackled by
  Wick-Isserelis theorem. With slight mis-use of notation, here we let
  $\epsilon_{a}$ denote the $a$th element of $\epsilon_{i}$ vector.
  
  Then 
  \begin{equation}
  \mathbb{E}[\epsilon_{a}\epsilon_{b}\epsilon_{c}\epsilon_{d}]=I_{ab}I_{cd}+I_{ad}I_{bc}+I_{ac}I_{bd}
  \end{equation}
  
  thus 
  \begin{align}
  \mathbb{E}[\epsilon_{i}^{\top}Q\epsilon_{i}\cdot\epsilon_{i}^{\top}Q\epsilon_{i}] & =\mathbb{E}[\sum_{a,b,c,d}Q_{ab}Q_{cd}\epsilon_{a}\epsilon_{b}\epsilon_{c}\epsilon_{d}]\notag\\
   & =\sum_{a,b,c,d}Q_{ab}Q_{cd}\mathbb{E}[\epsilon_{a}\epsilon_{b}\epsilon_{c}\epsilon_{d}]\notag\\
   & =\sum_{a,b,c,d}Q_{ab}Q_{cd}(\delta_{ab}\delta_{cd}+\delta_{ac}\delta_{bd}+\delta_{ad}\delta_{bc})\notag\\
   & =\mathrm{Tr}Q\mathrm{Tr}Q+\mathrm{Tr}[Q^{\top}Q]+\mathrm{Tr}[QQ]
  \end{align}
  
  Thus 
  \begin{align}
  \mathbb{E}[B^{2}] & =\frac{\sigma^{4}}{4}\Big(\mathbb{E}[\epsilon_{i}^{\top}Q\epsilon_{i}\cdot\epsilon_{i}^{\top}Q\epsilon_{i}]-(\mathrm{Tr}Q)^{2}\Big)\notag\\
   & =\frac{\sigma^{4}}{4}\Big((\mathrm{Tr}Q)^{2}+2\mathrm{Tr}[Q^{2}]-(\mathrm{Tr}Q)^{2}\Big)\notag\\
   & =\frac{\sigma^{4}}{2}\mathrm{Tr}[Q^{2}]
  \end{align}
  
  In sum, the population variance of reward reads, 
  \begin{equation}
  \mathrm{Var}[\tilde{R}_{i}]=\sigma^{2}\|\mathbf{v}\|^{2}+\frac{\sigma^{4}}{2}\mathrm{Tr}[Q^{2}]+\sigma_{\xi}^{2}
  \end{equation}
  
  and the population standard deviation is 
  \begin{equation}
  \sigma_{R}=\sqrt{\sigma^{2}\|\mathbf{v}\|^{2}+\frac{\sigma^{4}}{2}\mathrm{Tr}[Q^{2}]+\sigma_{\xi}^{2}}
  \end{equation}

  \paragraph{Expectation of parameter update}
  
  Next let's examine expectation of parameter change, where the key
  is $\mathbb{E}[\tilde{R}_{i}\epsilon_{i}]$, decomposing into three
  terms
  \begin{equation}
  \mathbb{E}[\tilde{R}_{i}\epsilon_{i}]=\mathbb{E}[A\epsilon_{i}]+\mathbb{E}[B\epsilon_{i}]+\mathbb{E}[C\epsilon_{i}]
  \end{equation}
  
  Second term vanish due to the odd order of Gaussian variable, 
  \begin{equation}
  \mathbb{E}[B\epsilon_{i}]=0
  \end{equation}
  
  Third term vanish due to independence $\epsilon_{i}\perp\xi_{i}$
  \begin{equation}
  \mathbb{E}[C\epsilon_{i}]=0
  \end{equation}
  
  First term is easy to calculate 
  \begin{align}
  \mathbb{E}[A\epsilon_{i}] & =\mathbb{E}[(-\sigma\,\theta_{0}^{\top}Q\epsilon_{i})\epsilon_{i}]\notag\\
   & =-\sigma\mathbb{E}[\epsilon_{i}(\epsilon_{i}^{\top}\mathbf{v})]\notag\\
   & =-\sigma\mathbf{v}
  \end{align}
  
  So in sum, 
  \begin{equation}
  \mathbb{E}[\tilde{R}_{i}\epsilon_{i}]=-\sigma\mathbf{v}
  \end{equation}
  
  \begin{align}
  \mathbb{E}[\Delta\theta] & =\frac{\alpha}{\sigma_{R}}\mathbb{E}[\tilde{R}_{i}\epsilon_{i}]\notag\\
   & =\frac{-\alpha\sigma\mathbf{v}}{\sigma_{R}}\notag\\
   & =\frac{-\alpha\sigma\mathbf{v}}{\sqrt{\sigma^{2}\|\mathbf{v}\|^{2}+\frac{\sigma^{4}}{2}\mathrm{Tr}[Q^{2}]+\sigma_{\xi}^{2}}}
  \end{align}
  
  which is on expectation aligned with local gradient direction, but
  scaled by signal to noise ratio. 
  
  \paragraph{Variance of parameter update}
  
  Finally, let's examine variance of parameter change, where the key
  is $\mathrm{Cov}[\tilde{R}_{i}\epsilon_{i}]$,
  \begin{align}
  \mathrm{Cov}[\tilde{R}_{i}\epsilon_{i}] & =\mathbb{E}[\tilde{R}_{i}\epsilon_{i}(\tilde{R}_{i}\epsilon_{i})^{\top}]-\mathbb{E}[\tilde{R}_{i}\epsilon_{i}]\mathbb{E}[\tilde{R}_{i}\epsilon_{i}]^{\top}\notag\\
   & =\mathbb{E}[\tilde{R}_{i}^{2}\epsilon_{i}\epsilon_{i}^{\top}]-\sigma^{2}\mathbf{v}\mathbf{v}^{\top}
  \end{align}
  
  Tackling the first term, using the same logic (odd order and independence)
  to eliminate terms, 
  \begin{align}
  \mathbb{E}[\tilde{R}_{i}^{2}\epsilon_{i}\epsilon_{i}^{\top}] & =\mathbb{E}[(A+B+C)^{2}\epsilon_{i}\epsilon_{i}^{\top}]\notag\\
   & =\mathbb{E}[(A^{2}+B^{2}+C^{2}+2AB+2BC+2AC)\epsilon_{i}\epsilon_{i}^{\top}]\notag\\
   & =\mathbb{E}[(A^{2}+B^{2}+C^{2})\epsilon_{i}\epsilon_{i}^{\top}]
  \end{align}
  
  Tackle $C$ term. 
  \begin{align}
  \mathbb{E}[C^{2}\epsilon_{i}\epsilon_{i}^{\top}] & =\mathbb{E}[C^{2}]\mathbb{E}[\epsilon_{i}\epsilon_{i}^{\top}]\notag\\
   & =\sigma_{\xi}^{2}I_{d}
  \end{align}
  
  Tackle \textbf{$A$} term
  
  \begin{align}
  \mathbb{E}[A^{2}\epsilon_{i}\epsilon_{i}^{\top}] & =\mathbb{E}[(-\sigma\,\theta_{0}^{\top}Q\epsilon_{i})^{2}\epsilon_{i}\epsilon_{i}^{\top}]\notag\\
   & =\sigma^{2}\mathbb{E}[(\mathbf{v}^{\top}\epsilon_{i})^{2}\epsilon_{i}\epsilon_{i}^{\top}]
  \end{align}
  
  using the same Wick theorem trick we have each element 
  \begin{align}
  \mathbb{E}[(\mathbf{v}^{\top}\epsilon_{i})^{2}\epsilon_{i}\epsilon_{i}^{\top}]_{cd} & =\mathbb{E}[\sum_{a,b}\mathbf{v}_{a}\mathbf{v}_{b}\epsilon_{a}\epsilon_{b}\epsilon_{c}\epsilon_{d}]\notag\\
   & =\sum_{a,b}\mathbf{v}_{a}\mathbf{v}_{b}\mathbb{E}[\epsilon_{a}\epsilon_{b}\epsilon_{c}\epsilon_{d}]\notag\\
   & =\sum_{a,b}\mathbf{v}_{a}\mathbf{v}_{b}(\delta_{ab}\delta_{cd}+\delta_{ac}\delta_{bd}+\delta_{ad}\delta_{cb})\notag\\
   & =\|\mathbf{v}\|^{2}\delta_{cd}+\mathbf{v}_{c}\mathbf{v}_{d}+\mathbf{v}_{d}\mathbf{v}_{c}
  \end{align}
  
  Thus, 
  \begin{equation}
  \mathbb{E}[A^{2}\epsilon_{i}\epsilon_{i}^{\top}]=\sigma^{2}(\|\mathbf{v}\|^{2}I_{d}+2\mathbf{v}\mathbf{v}^{\top})
  \end{equation}
  
  Finally, tackle $B$ term 
  \begin{align}
  \mathbb{E}[B^{2}\epsilon_{i}\epsilon_{i}^{\top}] & =\mathbb{E}[(\frac{\sigma^{2}}{2}(\epsilon_{i}^{\top}Q\epsilon_{i}-\mathrm{Tr}Q))^{2}\epsilon_{i}\epsilon_{i}^{\top}]\notag\\
   & =\frac{\sigma^{4}}{4}\mathbb{E}[\Big((\epsilon_{i}^{\top}Q\epsilon_{i})^{2}-2\mathrm{Tr}Q\cdot\epsilon_{i}^{\top}Q\epsilon_{i}+(\mathrm{Tr}Q)^{2}\Big)\epsilon_{i}\epsilon_{i}^{\top}]
  \end{align}
  
  where we encounter 2nd, 4th and 6th order terms. 
  
  2nd order term
  \begin{equation}
  \mathbb{E}[(\mathrm{Tr}Q)^{2}\epsilon_{i}\epsilon_{i}^{\top}]=(\mathrm{Tr}Q)^{2}I_{d}
  \end{equation}
  
  4th order term 
  \begin{align}
  \mathbb{E}[\epsilon_{i}^{\top}Q\epsilon_{i}\cdot\epsilon_{i}\epsilon_{i}^{\top}]_{cd} & =\mathbb{E}[\sum_{a,b}Q_{ab}\epsilon_{a}\epsilon_{b}\epsilon_{c}\epsilon_{d}]\notag\\
   & =\sum_{a,b}Q_{ab}(\delta_{ab}\delta_{cd}+\delta_{ac}\delta_{bd}+\delta_{ad}\delta_{cb})\notag\\
   & =\mathrm{Tr}Q\cdot\delta_{cd}+Q_{cd}+Q_{dc}
  \end{align}
  
  Thus
  \begin{align}
  \mathbb{E}[\epsilon_{i}^{\top}Q\epsilon_{i}\cdot\epsilon_{i}\epsilon_{i}^{\top}] & =\mathrm{Tr}Q\cdot I_{d}+Q+Q^{\top}\notag\\
   & =\mathrm{Tr}Q\cdot I_{d}+2Q
  \end{align}
  
  6th order term
  \begin{equation}
  \mathbb{E}[\Big((\epsilon_{i}^{\top}Q\epsilon_{i})^{2}\Big)\epsilon_{i}\epsilon_{i}^{\top}]_{ef}=\mathbb{E}[\sum_{a,b,c,d}Q_{ab}Q_{cd}\epsilon_{a}\epsilon_{b}\epsilon_{c}\epsilon_{d}\epsilon_{e}\epsilon_{f}]
  \end{equation}
  
  Note this has $C_{6}^{2}=15$ pairings, with some diagrammatic reasoning,
  we can find topologically we have three types 
  \begin{align}
  \mathbb{E}[\Big((\epsilon_{i}^{\top}Q\epsilon_{i})^{2}\Big)\epsilon_{i}\epsilon_{i}^{\top}]_{ef} & =\mathbb{E}[\sum_{a,b,c,d}Q_{ab}Q_{cd}\epsilon_{a}\epsilon_{b}\epsilon_{c}\epsilon_{d}\epsilon_{e}\epsilon_{f}]\notag\\
   & =\sum_{a,b,c,d}Q_{ab}Q_{cd}\Big(\underbrace{\delta_{ab}\delta_{cd}\delta_{ef}}_{1\text{case}}+\underbrace{\delta_{ac}\delta_{bd}\delta_{ef}+...}_{2\text{case}}+\notag\\
   & \quad\quad\underbrace{\delta_{ae}\delta_{bf}\delta_{cd}+...}_{4\text{case}}+\underbrace{\delta_{ae}\delta_{cf}\delta_{bd}+...}_{8\text{case}}\Big)\notag\\
   & =((\mathrm{Tr}Q)^{2}+2\mathrm{Tr}Q^{2})\delta_{ef}+4\mathrm{Tr}Q\cdot Q_{ef}+8[Q^{2}]_{ef}
  \end{align}
  
  Thus 
  \begin{equation}
  \mathbb{E}[\Big((\epsilon_{i}^{\top}Q\epsilon_{i})^{2}\Big)\epsilon_{i}\epsilon_{i}^{\top}]=((\mathrm{Tr}Q)^{2}+2\mathrm{Tr}Q^{2})I_{d}+4\mathrm{Tr}Q\cdot Q+8Q^{2}
  \end{equation}

  \begin{figure}[h]
      \centering
      \begin{tikzpicture}[scale=1.0, every node/.style={font=\small}]
        \newcommand{\Qedge}[3]{
          \draw[dashed] (#1) |- ($(#1)!0.5!(#2)+(0,0.4)$) -| (#2)
            node[midway,above] {$#3$};
        }
      
        \newcommand{\crossnode}[3]{
          \node[draw, cross out, inner sep=2pt, line width=0.8pt, label=below:$#2$] (#1) at #3 {};
        }
      
        \node at (0,3.2) {\textbf{Type I: $1$ diagram}};
      
        \node[circle,fill=black,inner sep=1.5pt,label=below:$a$] (a1) at (-1.2,2) {};
        \node[circle,fill=black,inner sep=1.5pt,label=below:$b$] (b1) at (-0.6,2) {};
        \node[circle,fill=black,inner sep=1.5pt,label=below:$c$] (c1) at (0.4,2) {};
        \node[circle,fill=black,inner sep=1.5pt,label=below:$d$] (d1) at (1.0,2) {};
        \crossnode{e1}{e}{(2.0,2)}
        \crossnode{f1}{f}{(2.8,2)}
      
        \Qedge{a1}{b1}{Q_{ab}}
        \Qedge{c1}{d1}{Q_{cd}}
      
        \draw (e1) -- (f1);
        \draw (a1) to[out=0,in=180] (b1);
        \draw (c1) to[out=0,in=180] (d1);
      
        \node at (0.8,1.35) {$\Rightarrow (\mathrm{Tr}Q)^2\,\delta_{ef}$};
      
        \node at (0,0.2) {\textbf{Type II: $2$ diagrams}};
      
        \node[circle,fill=black,inner sep=1.5pt,label=below:$a$] (a2) at (-1.2,-1) {};
        \node[circle,fill=black,inner sep=1.5pt,label=below:$b$] (b2) at (-0.6,-1) {};
        \node[circle,fill=black,inner sep=1.5pt,label=below:$c$] (c2) at (0.4,-1) {};
        \node[circle,fill=black,inner sep=1.5pt,label=below:$d$] (d2) at (1.0,-1) {};
        \crossnode{e2}{e}{(2.0,-1)}
        \crossnode{f2}{f}{(2.8,-1)}
      
        \Qedge{a2}{b2}{Q_{ab}}
        \Qedge{c2}{d2}{Q_{cd}}
      
        \draw (a2) to[out=20,in=160] (c2);
        \draw (b2) to[out=20,in=160] (d2);
        \draw (e2) -- (f2);
      
        \node at (0.8,-1.65) {$\Rightarrow 2\,\mathrm{Tr}(Q^2)\,\delta_{ef}$};
      
        \node at (7,3.2) {\textbf{Type III: $4$ diagrams}};
      
        \node[circle,fill=black,inner sep=1.5pt,label=below:$a$] (a3) at (5.5,2) {};
        \node[circle,fill=black,inner sep=1.5pt,label=below:$b$] (b3) at (6.1,2) {};
        \node[circle,fill=black,inner sep=1.5pt,label=below:$c$] (c3) at (7.1,2) {};
        \node[circle,fill=black,inner sep=1.5pt,label=below:$d$] (d3) at (7.7,2) {};
        \crossnode{e3}{e}{(8.7,2)}
        \crossnode{f3}{f}{(9.5,2)}
      
        \Qedge{a3}{b3}{Q_{ab}}
        \Qedge{c3}{d3}{Q_{cd}}
      
        \draw (c3) to[out=15,in=165] (e3);
        \draw (d3) to[out=15,in=165] (f3);
        \draw (a3) to[out=0,in=180] (b3);
      
        \node at (7.5,1.35) {$\Rightarrow 4\,\mathrm{Tr}(Q)\,Q_{ef}$};
      
        \node at (7,0.2) {\textbf{Type IV: $8$ diagrams}};
      
          \node[circle,fill=black,inner sep=1.5pt,label=below:$a$] (a4) at (5.5,-1) {};
          \node[circle,fill=black,inner sep=1.5pt,label=below:$b$] (b4) at (6.1,-1) {};
          \node[circle,fill=black,inner sep=1.5pt,label=below:$c$] (c4) at (7.1,-1) {};
          \node[circle,fill=black,inner sep=1.5pt,label=below:$d$] (d4) at (7.7,-1) {};
          \crossnode{e4}{e}{(9.0,-1)}
          \crossnode{f4}{f}{(10.0,-1)}
  
          \Qedge{a4}{b4}{Q_{ab}}
          \Qedge{c4}{d4}{Q_{cd}}
  
          \draw (a4) to[out=15,in=165] (e4);
          \draw (b4) to[out=0,in=180] (c4);
          \draw (d4) to[out=15,in=165] (f4);
      
        \node at (7.5,-1.65) {$\Rightarrow 8\,(Q^2)_{ef}$};
      
      \end{tikzpicture}
      \caption{Wick contraction classes for $\mathbb{E}[(\epsilon^\top Q\epsilon)^2\,\epsilon_e\epsilon_f]$, with external loose ends $e,f$ shown as crosses.}
  \end{figure}

  In summary
  \begin{align}
  \mathbb{E}[B^{2}\epsilon_{i}\epsilon_{i}^{\top}] & =\frac{\sigma^{4}}{4}\mathbb{E}[\Big((\epsilon_{i}^{\top}Q\epsilon_{i})^{2}-2\mathrm{Tr}Q\cdot\epsilon_{i}^{\top}Q\epsilon_{i}+(\mathrm{Tr}Q)^{2}\Big)\epsilon_{i}\epsilon_{i}^{\top}]\notag\\
   & =\frac{\sigma^{4}}{4}\Big[((\mathrm{Tr}Q)^{2}+2\mathrm{Tr}Q^{2})I_{d}+4\mathrm{Tr}Q\cdot Q+8Q^{2}\notag\\
   & \quad-2\mathrm{Tr}Q\cdot(\mathrm{Tr}Q\cdot I_{d}+2Q)\notag\\
   & \quad+(\mathrm{Tr}Q)^{2}I_{d}\Big]\notag\\
   & =\frac{\sigma^{4}}{4}\Big[(2\mathrm{Tr}Q^{2})I_{d}+8Q^{2}\Big]\notag\\
   & =\frac{\sigma^{4}}{2}\Big[(\mathrm{Tr}Q^{2})\cdot I_{d}+4Q^{2}\Big]
  \end{align}
  
  Thus
  
  \begin{align}
  \mathbb{E}[\tilde{R}_{i}^{2}\epsilon_{i}\epsilon_{i}^{\top}] & =\mathbb{E}[(A^{2}+B^{2}+C^{2})\epsilon_{i}\epsilon_{i}^{\top}]\notag\\
   & =\sigma^{2}(\|\mathbf{v}\|^{2}I_{d}+2\mathbf{v}\mathbf{v}^{\top})+\frac{\sigma^{4}}{2}\Big[(\mathrm{Tr}Q^{2})\cdot I_{d}+4Q^{2}\Big]+\sigma_{\xi}^{2}I_{d}
  \end{align}
  
  The covariance term reads 
  \begin{align}
  \mathrm{Cov}[\tilde{R}_{i}\epsilon_{i}] & =\mathbb{E}[\tilde{R}_{i}^{2}\epsilon_{i}\epsilon_{i}^{\top}]-\sigma^{2}\mathbf{v}\mathbf{v}^{\top}\notag\\
   & =\sigma^{2}(\|\mathbf{v}\|^{2}I_{d}+\mathbf{v}\mathbf{v}^{\top})+\frac{\sigma^{4}}{2}\Big[(\mathrm{Tr}Q^{2})\cdot I_{d}+4Q^{2}\Big]+\sigma_{\xi}^{2}I_{d}
  \end{align}
  
  Recall that $\mathrm{Var}[\tilde{R}_{i}]=\sigma^{2}\|\mathbf{v}\|^{2}+\frac{\sigma^{4}}{2}\mathrm{Tr}[Q^{2}]+\sigma_{\xi}^{2}$
  
  \begin{equation}
  \mathrm{Cov}[\tilde{R}_{i}\epsilon_{i}]=\sigma^{2}\mathbf{v}\mathbf{v}^{\top}+2\sigma^{4}Q^{2}+\sigma_{R}^{2}I_{d}
  \end{equation}
  
  Thus in summary we have mean and covariance of parameter update, $\mathbf{v}:=Q\theta_{0}$
  \begin{align}
  \mathbb{E}[\Delta\theta] & =\frac{\alpha}{\sigma_{R}}\mathbb{E}[\tilde{R}_{i}\epsilon_{i}]\notag\\
   & =\frac{-\alpha\sigma\mathbf{v}}{\sigma_{R}}
  \end{align}
  \begin{align}
  \mathrm{Cov}[\Delta\theta] & =\frac{\alpha^{2}}{N\sigma_{R}^{2}}\mathrm{Cov}[\tilde{R}_{i}\epsilon_{i}]\notag\\
   & =\frac{\alpha^{2}}{N\sigma_{R}^{2}}\Big(\sigma^{2}\mathbf{v}\mathbf{v}^{\top}+2\sigma^{4}Q^{2}+\sigma_{R}^{2}I_{d}\Big)\notag\\
   & =\frac{\alpha^{2}}{N}I_{d}+\frac{\alpha^{2}}{N\sigma_{R}^{2}}\Big(\sigma^{2}\mathbf{v}\mathbf{v}^{\top}+2\sigma^{4}Q^{2}\Big)
  \end{align}
  
  The covariance of weight update has three terms 
  \begin{itemize}
  \item An isotropic variance term from the Gaussian exploration per se. $\frac{\alpha^{2}}{N}I_{d}$
  \item A rank-1 term along current gradient direction $\mathbf{v}$, $\frac{\alpha^{2}\sigma^{2}}{N\sigma_{R}^{2}}\mathbf{v}\mathbf{v}^{\top}$
  \item A term corresponding to landscape curvature, i.e. higher variance
  along higher curvature directions, $\frac{2\sigma^{4}\alpha^{2}}{N\sigma_{R}^{2}}Q^{2}$
  \end{itemize}
  \end{proof}
  
  \paragraph{Order analysis of quadratic landscape}

  Consider the relative magnitudes of the covariance terms, using the first term as a reference:
  \begin{itemize}
    \item The second term (aligned with the gradient direction) is of order $\frac{\sigma^{2}}{\sigma_{R}^{2}}\|\mathbf{v}\|^{2}$.
    \item The third term (along each direction $\mathbf{u}$) is of order $\frac{2\sigma^{4}}{\sigma_{R}^{2}}\,\mathbf{u}^{\top}Q^{2}\mathbf{u}$.
  \end{itemize}

  Since in practice $2\sigma^{2}$ is very small (e.g., $4.5\times10^{-6}$), the quadratic term generally acts as a small perturbation. To determine when the third (quadratic) term becomes comparable to the second, compare the terms along a specific direction $\mathbf{u}$:

  \begin{align}
    \frac{2\sigma^{4}}{\sigma_{R}^{2}}\,\mathbf{u}^{\top}Q^{2}\mathbf{u} &\approx \frac{\sigma^{2}}{\sigma_{R}^{2}}\|\mathbf{v}\|^{2} \notag\\
    2\sigma^{2}\,\mathbf{u}^{\top}Q^{2}\mathbf{u} &\approx \theta_{0}^{\top}Q^{2}\theta_{0} \notag\\
    \mathbf{u}^{\top}Q^{2}\mathbf{u} &\approx \frac{1}{2\sigma^{2}}\,\theta_{0}^{\top}Q^{2}\theta_{0}
  \end{align}

  Assume $\mathbf{u}$ is an eigenvector of $Q$ with eigenvalue $\lambda_{k}$,
  \begin{align}
  \lambda_{k}^{2} & \approx\frac{1}{2\sigma^{2}}\theta_{0}^{\top}Q^{2}\theta_{0}\notag\\
  \lambda_{k} & \approx\frac{\|\mathbf{v}\|}{\sqrt{2}\sigma}
  \end{align}
  \begin{equation}
  \lambda_{k}^{2}\approx\frac{1}{2\sigma^{2}}\sum_{i=1}^{d}\lambda_{i}^{2}(\theta_{0}^{\top}\mathbf{u}_{i})^{2}
  \end{equation}
  
  This shows that the current gradient norm $\|\mathbf{v}\|$ needs
  to be much smaller (at least $\sigma$ times) and the eigenvalue /
  curvature of the probe direction needs to be sharp enough for the
  3rd term to be comparable to the 2nd term. Two scenarios this could
  happen 
  \begin{itemize}
  \item For really sharply curved directions $\lambda_{k}$ this is usually
  be true. 
  \item The current state $\theta_{0}$ is close to optimum $\theta^{*}=0$
  or lies in the vast valley of degenerate solution, i.e. $Q\theta_{0}=0$,
  $\theta_{0}$ in the null space of $Q$. 
  \end{itemize}
  \paragraph{Manifold Alignment of Update}

Define the projections onto the gradient direction $\mathbf{v} = Q\theta_0$ and its orthogonal complement:
\begin{align}
P^{\parallel} &:= \frac{\mathbf{v}\mathbf{v}^{\top}}{\|\mathbf{v}\|^{2}}, \qquad
P^{\perp} := I_{d} - P^{\parallel}.
\end{align}

The on-manifold expected squared norm decomposes as
\begin{align}
\mathbb{E}\|P^{\parallel}\Delta\theta\|^{2}
&= \|\mathbb{E}[P^{\parallel}\Delta\theta]\|^{2}
   + \mathrm{Tr}[\mathrm{Cov}[P^{\parallel}\Delta\theta]].
\end{align}

For the first term, since $\mathbb{E}[\Delta\theta] = \frac{-\alpha\sigma\,\mathbf{v}}{\sigma_{R}}$
is already along $\mathbf{v}$, we have $P^{\parallel}\mathbb{E}[\Delta\theta] = \mathbb{E}[\Delta\theta]$, so
\begin{align}
\|\mathbb{E}[P^{\parallel}\Delta\theta]\|^{2}
&= \|\mathbb{E}[\Delta\theta]\|^{2}
= \frac{\alpha^{2}\sigma^{2}\|\mathbf{v}\|^{2}}{\sigma_{R}^{2}}.
\end{align}

For the second term, we project the covariance \eqref{eq:es-quad-cov} onto $P^{\parallel}$:
\begin{align}
\mathrm{Cov}[P^{\parallel}\Delta\theta]
&= P^{\parallel}\,\mathrm{Cov}[\Delta\theta]\,P^{\parallel} \notag\\
&= \frac{\alpha^{2}}{N\sigma_{R}^{2}}
   P^{\parallel}\left(\sigma^{2}\,\mathbf{v}\mathbf{v}^{\top}
   + 2\sigma^{4}Q^{2}
   + \sigma_{R}^{2}\,I_{d}\right)P^{\parallel}.
\end{align}
Evaluating each term using $P^{\parallel}\mathbf{v}\mathbf{v}^{\top}P^{\parallel} = \|\mathbf{v}\|^{2}P^{\parallel}$,
$P^{\parallel}Q^{2}P^{\parallel} = (\hat{\mathbf{v}}^{\top}Q^{2}\hat{\mathbf{v}})\,P^{\parallel}$
where $\hat{\mathbf{v}} := \mathbf{v}/\|\mathbf{v}\|$,
and $P^{\parallel}I_{d}\,P^{\parallel} = P^{\parallel}$:
\begin{align}
\mathrm{Cov}[P^{\parallel}\Delta\theta]
&= \frac{\alpha^{2}}{N\sigma_{R}^{2}}
   \left(\sigma^{2}\|\mathbf{v}\|^{2}
         + 2\sigma^{4}\,\hat{\mathbf{v}}^{\top}Q^{2}\hat{\mathbf{v}}
         + \sigma_{R}^{2}\right)P^{\parallel}.
\end{align}
Taking the trace ($\mathrm{Tr}[P^{\parallel}] = 1$) and combining:
\begin{align}
\mathbb{E}\|P^{\parallel}\Delta\theta\|^{2}
&= \frac{\alpha^{2}\sigma^{2}\|\mathbf{v}\|^{2}}{\sigma_{R}^{2}}
   + \frac{\alpha^{2}}{N\sigma_{R}^{2}}
     \left(\sigma^{2}\|\mathbf{v}\|^{2}
           + 2\sigma^{4}\,\hat{\mathbf{v}}^{\top}Q^{2}\hat{\mathbf{v}}
           + \sigma_{R}^{2}\right).
\end{align}

The total expected squared norm follows analogously using $\mathrm{Tr}[\mathbf{v}\mathbf{v}^{\top}] = \|\mathbf{v}\|^{2}$,
$\mathrm{Tr}[Q^{2}] = \mathrm{Tr}[Q^{2}]$, and $\mathrm{Tr}[I_{d}] = d$:
\begin{align}
\mathbb{E}\|\Delta\theta\|^{2}
&= \frac{\alpha^{2}\sigma^{2}\|\mathbf{v}\|^{2}}{\sigma_{R}^{2}}
   + \frac{\alpha^{2}}{N\sigma_{R}^{2}}
     \left(\sigma^{2}\|\mathbf{v}\|^{2}
           + 2\sigma^{4}\,\mathrm{Tr}[Q^{2}]
           + d\,\sigma_{R}^{2}\right).
\end{align}

Factoring out $\alpha^{2}/(N\sigma_{R}^{2})$ from both expressions --- noting that the
mean-squared term $\frac{\alpha^{2}\sigma^{2}\|\mathbf{v}\|^{2}}{\sigma_{R}^{2}}$
contributes $N\,\sigma^{2}\|\mathbf{v}\|^{2}$ after extracting the common factor,
which combines with the variance contribution to yield $(N+1)\,\sigma^{2}\|\mathbf{v}\|^{2}$
--- the on-manifold fraction is
\begin{align}
\rho
:= \frac{\mathbb{E}\|P^{\parallel}\Delta\theta\|^{2}}
        {\mathbb{E}\|\Delta\theta\|^{2}}
= \frac{(N+1)\,\sigma^{2}\|\mathbf{v}\|^{2}
        + 2\sigma^{4}\,\hat{\mathbf{v}}^{\top}Q^{2}\hat{\mathbf{v}}
        + \sigma_{R}^{2}}
       {(N+1)\,\sigma^{2}\|\mathbf{v}\|^{2}
        + 2\sigma^{4}\,\mathrm{Tr}[Q^{2}]
        + d\,\sigma_{R}^{2}}.
\end{align}

Setting $Q^{2}=0$ and $\sigma_{R}^{2}=\sigma^{2}\|\mathbf{v}\|^{2}+\sigma_{\xi}^{2}$
recovers the linear-landscape ratio $\rho = (1+(N+1)s)/(d+(N+1)s)$
with $s = \sigma^{2}\|\mathbf{v}\|^{2}/\sigma_{R}^{2}$.

\clearpage
  \subsection{Multi-step ES analysis for quadratic landscape}\label{appsec:quad_landscape_dynamics}
  We could write down the ES iteration as 
  \begin{align}
  \theta_{1}-\theta_{0} & =\frac{-\alpha\sigma\mathbf{v}}{\sigma_{R}}+\eta_{0} \notag\\
   & =-\frac{\alpha\sigma}{\sigma_{R}}Q\theta_{0}+\eta_{0}
  \end{align}
  Thus
  \begin{equation}
  \theta_{t+1}=\Big(I-\frac{\alpha\sigma}{\sigma_{R}}Q\Big)\ \theta_{t}+\eta_{t}
  \end{equation}
  where, 
  \begin{align}
  \mathrm{Cov}[\eta_{t}] & =\frac{\alpha^{2}}{N\sigma_{R}^{2}}\Big(\sigma^{2}\mathbf{v}\mathbf{v}^{\top}+2\sigma^{4}Q^{2}+\sigma_{R}^{2}I_{d}\Big) \notag\\
   & =\frac{\alpha^{2}}{N}I_{d}+\frac{\alpha^{2}}{N\sigma_{R}^{2}}\Big(\sigma^{2}Q\theta_{t}\theta_{t}^{\top}Q+2\sigma^{4}Q^{2}\Big)
  \end{align} 
  and 
  \begin{align}
  \sigma_{R} & =\sqrt{\sigma^{2}\|\mathbf{v}\|^{2}+\frac{\sigma^{4}}{2}\mathrm{Tr}[Q^{2}]+\sigma_{\xi}^{2}} \notag\\
   & =\sqrt{\sigma^{2}\theta_{t}^{\top}Q^{2}\theta_{t}+\frac{\sigma^{4}}{2}\mathrm{Tr}[Q^{2}]+\sigma_{\xi}^{2}}
  \end{align}
  
  \paragraph{Simplified ES-optimization process }
  
  Consider the simplified ES optimization, where we assume the $\sigma_{R}$ is held constant. 
  To simplify it even more, assume $\eta_{t}$ is Gaussian and isotropic
  each step $\frac{\alpha^{2}}{N}I_{d}$, i.e. ignore the 2nd and 3rd
  term $\frac{\alpha^{2}}{N\sigma_{R}^{2}}\Big(\sigma^{2}Q\theta_{t}\theta_{t}^{\top}Q+2\sigma^{4}Q^{2}\Big)$
  
  Thus the optimization iteration reads, 
  \begin{equation}
  \theta_{t+1}=\Big(I-\frac{\alpha\sigma}{\sigma_{R}}Q\Big)\theta_{t}+\eta_{t},\quad\eta_{t}\sim\mathcal{N}(0,\frac{\alpha^{2}}{N}I_{d})
  \end{equation}
  
  This process is known as vector autoregression VAR(1) or discrete-time
  OU process.

  \begin{proof}[Proof of Proposition~\ref{prop:quad_landscape_dynamics}]
  We can write down the generative process of $\theta_{T}$
  recursively. 
  \begin{align}
  \theta_{t} & =H\theta_{t-1}+\eta_{t-1} \notag\\
   & =H(H\theta_{t-2}+\eta_{t-2})+\eta_{t-1} \notag\\
   & =H^{2}\theta_{t-2}+H\eta_{t-2}+\eta_{t-1} \notag\\
   & =H^{3}\theta_{t-3}+H^{2}\eta_{t-3}+H\eta_{t-2}+\eta_{t-1} \notag\\
   & =H^{t}\theta_{0}+\sum_{i=1}^{t}H^{i-1}\eta_{t-i}
  \end{align}
  
  Note let $\mathbf{u}_{k},\lambda_{k}$ be the eigenvector and eigenvalue
  of $Q$, then 
  \begin{align}
  H & =I-\frac{\alpha\sigma}{\sigma_{R}}Q \notag\\
   & =\sum_{k=1}^{d}(1-\frac{\alpha\sigma}{\sigma_{R}}\lambda_{k})\mathbf{u}_{k}\mathbf{u}_{k}^{\top}
  \end{align}

  In this simplification, the distribution comes from the independent
  Gaussian random vectors $\{\eta_{i}\}$, thus it's still Gaussian
  distributed. 
  
  Thus we can write down the expectation, which only depends on the
  first term. 
  \begin{align}
  \mathbb{E}[\theta_{t}] & =H^{t}\theta_{0} \notag\\
   & =\sum_{k=1}^{d}(1-\frac{\alpha\sigma}{\sigma_{R}}\lambda_{k})^{t}\mathbf{u}_{k}\mathbf{u}_{k}^{\top}\theta_{0}
  \end{align}
  
  Asymptotically, 
  \begin{align}
  \lim_{t\to\infty}\mathbb{E}[\theta_{t}] & =\lim_{t\to\infty}H^{t}\theta_{0} \notag\\
   & =\lim_{t\to\infty}\sum_{k=1}^{d}(1-\frac{\alpha\sigma}{\sigma_{R}}\lambda_{k})^{t}\mathbf{u}_{k}\mathbf{u}_{k}^{\top}\theta_{0}
  \end{align}
  
  \begin{itemize}
  \item For flat dimensions $\lambda_{k}=0$, then $\lim_{t\to\infty}(1-\frac{\alpha\sigma}{\sigma_{R}}\lambda_{k})^{t}=1$,
  thus on expectation the solution won't move on that direction $\lim_{t\to\infty}\mathbb{E}[\mathbf{u}_{k}^{\top}\theta_{t}]=\mathbf{u}_{k}^{\top}\theta_{0}$
  \item If $|1-\frac{\alpha\sigma}{\sigma_{R}}\lambda_{k}|>1$, this will
  diverge to infinity. $\lim_{t\to\infty}\mathbb{E}[\mathbf{u}_{k}^{\top}\theta_{t}]=\infty$ 
  \begin{itemize}
  \item Specifically, when $\lambda_{k}<0$ concave dimensions
  \item or for $\frac{\alpha\sigma}{\sigma_{R}}\lambda_{k}>2$, this will
  diverge. This could happen for very sharply curved directions where
  curvature exceeds $\lambda_{k}>2\frac{\sigma_{R}}{\alpha\sigma}$.
  Note for gradient descent this could also be caused by too high learning
  rate, but here $\alpha=\sigma/2$. 
  \end{itemize}
  \end{itemize}
  The covariance depends on the summation $\sum_{i=1}^{t}H^{i-1}\eta_{t-i}$,
  using the i.i.d. property of the noise $\{\eta_{i}\}$ and the Gaussian
  distribution of noise, 
  \begin{align}
  \mathrm{Cov}[\theta_{t}] & =\mathrm{Cov}[\sum_{i=1}^{t}H^{i-1}\eta_{t-i}] \notag\\
   & =\sum_{i=1}^{t}\mathrm{Cov}[H^{i-1}\eta_{t-i}] \notag\\
   & =\sum_{i=1}^{t}(H^{i-1})\mathrm{Cov}[\eta_{t-i}](H^{i-1})^{\top} \notag\\
   & =\sum_{i=1}^{t}(H^{i-1})^{2}\frac{\alpha^{2}}{N}I_{d} \notag\\
   & =\frac{\alpha^{2}}{N}\sum_{i=0}^{t-1}H^{2i}
  \end{align}
  
  Thus the variance along one eigen axis is the sum of this series.
  \begin{align}
  \mathrm{Cov}[\mathbf{u}_{k}^{\top}\theta_{t}] & =\frac{\alpha^{2}}{N}\mathbf{u}_{k}^{\top}\sum_{i=0}^{t-1}H^{2i}\mathbf{u}_{k} \notag\\
   & =\frac{\alpha^{2}}{N}\sum_{i=0}^{t-1}(1-\frac{\alpha\sigma}{\sigma_{R}}\lambda_{k})^{2i}
  \end{align}
  
  Using the formula for finite term sum of geometric series, 
  \begin{equation}
  \sum_{k=0}^{t-1}x^{k}=\begin{cases}
  \dfrac{1-x^{t}}{1-x}, & x\neq1,\\[6pt]
  t, & x=1.
  \end{cases}
  \end{equation}
  note that for finite terms this identity holds for $|x|>1$. 
  
  
  
  The key quantity here is $\gamma^{2}:=(1-\frac{\alpha\sigma}{\sigma_{R}}\lambda_{k})^{2}$
  and its relation to $1$. $(1-\frac{\alpha\sigma}{\sigma_{R}}\lambda_{k})^{2}=1$
  entails $\lambda_{k}=0$ or $\lambda_{k}=\frac{2\sigma_{R}}{\alpha\sigma}$Thus,
  along an eigen direction, the variance reads, 
  
  \begin{equation}
  \mathrm{Cov}[\mathbf{u}_{k}^{\top}\theta_{t}]=\frac{\alpha^{2}}{N}\begin{cases}
  \frac{1-(1-\frac{\alpha\sigma}{\sigma_{R}}\lambda_{k})^{2t}}{1-(1-\frac{\alpha\sigma}{\sigma_{R}}\lambda_{k})^{2}}, & \text{otherwise}\\[6pt]
  t, & \lambda_{k}\in\{0,\frac{2\sigma_{R}}{\alpha\sigma}\}
  \end{cases}
  \end{equation}
  
  \begin{equation}\label{eq:simplify_ES_theory_norm_scaling}
  \mathrm{Cov}[\mathbf{u}_{k}^{\top}\theta_{t}]=\frac{\alpha^{2}}{N}\begin{cases}
  \frac{1-\gamma^{2t}}{1-\gamma^{2}}, & \text{otherwise}\\[6pt]
  t, & \gamma^{2}=0,\text{i.e. }\lambda_{k}\in\{0,\frac{2\sigma_{R}}{\alpha\sigma}\}
  \end{cases}
  \end{equation}
  \end{proof}

  \paragraph{Asymptotic behavior}
  \begin{itemize}
  \item In the flat direction, $\lambda_{k}=0$, then $\mathrm{Cov}[\mathbf{u}_{k}^{\top}\theta_{t}]=\frac{\alpha^{2}}{N}t$,
  recovering the linear scaling of random walk. Similarly for $\lambda_{k}=\frac{2\sigma_{R}}{\alpha\sigma}$,
  it increase linearly. 
  \item In the concave directions, i.e.$\lambda_{k}<0$, $(1-\frac{\alpha\sigma}{\sigma_{R}}\lambda_{k})^{2}>1$,
  then this series will diverge towards infinity. $\lim_{t\to\infty}\mathrm{Cov}[\mathbf{u}_{k}^{\top}\theta_{t}]=\infty$
  \item For $\lambda_{k}\geq\frac{2\sigma_{R}}{\alpha\sigma}$, convex direction
  with too high curvature, then $|1-\frac{\alpha\sigma}{\sigma_{R}}\lambda_{k}|\geq1$,
  then this series will diverge $\lim_{t\to\infty}\mathrm{Cov}[\mathbf{u}_{k}^{\top}\theta_{t}]=\infty$.
  \item For $0<\lambda_{k}<\frac{2\sigma_{R}}{\alpha\sigma}$, then $|1-\frac{\alpha\sigma}{\sigma_{R}}\lambda_{k}|<1$,
  then this series will converge, this is the desirable case for optimization.
  Asymptotically, 
  \begin{align}
  \lim_{t\to\infty}\mathrm{Cov}[\mathbf{u}_{k}^{\top}\theta_{t}] & =\frac{\alpha^{2}}{N}\frac{1}{1-(1-\frac{\alpha\sigma}{\sigma_{R}}\lambda_{k})^{2}} \notag\\
   & =\frac{\alpha^{2}}{N}\frac{1}{1-\gamma^{2}}
  \end{align}
  
  \begin{itemize}
  \item The convergence speed is inversely proportional to $\gamma^{2}:=(1-\frac{\alpha\sigma}{\sigma_{R}}\lambda_{k})^{2}$.
  When $\gamma^{2}$ close to $0$, the series $\gamma^{2t}$ converges
  faster; when $\gamma^{2}\sim1$, the convergence time gets longer
  and longer. 
  \item Convergence timescale can be computed $1-\gamma^{2\tau}=1/2$, then
  $\log\frac{1}{2}=\tau\log\gamma^{2}$, $\tau=-\frac{\log2}{\log\gamma^{2}}=-\frac{\log2}{\log(1-\frac{\alpha\sigma}{\sigma_{R}}\lambda_{k})^{2}}$
  is the characteristic time scale. 
  \item The converged level is also controlled by $\gamma^{2}$, specifically
  $\frac{\alpha^{2}}{N}\frac{1}{1-\gamma^{2}}$. When $\gamma^{2}$
  close to $0$, converged level is lower; when $\gamma^{2}\sim1$ the
  converged level grow towards infinity
  \item Thus directions where $\lambda_{k}\sim\frac{\sigma_{R}}{\alpha\sigma}$
  converge the fastest and the variance are the tightest $\lim_{t\to\infty}\mathrm{Cov}[\mathbf{u}_{k}^{\top}\theta_{t}]\sim\frac{\alpha^{2}}{N}$.
  Showing that these directions are the most well controlled by the
  landscape asymptotically. Specifically when $\lambda_{k}=\frac{\sigma_{R}}{\alpha\sigma}$,
  then this direction converges in one iteration $H\theta_{0}$. 
  \end{itemize}
  \end{itemize}

\clearpage
  \subsection{Gradient ascent on quadratic landscape}\label{appsec:gd_quad_landscape}

  For comparison with the ES dynamics of Section~\ref{appsec:quad_landscape_dynamics},
  we analyze standard (noiseless) gradient ascent\footnote{Since we are maximizing the reward, it's gradient ascent, but still we call it GD for simplicity}
  on the same quadratic landscape $R(\theta) = -\frac{1}{2}\theta^\top Q\theta$. 
  
  The negative gradient at $\theta_t$ is $-\nabla R(\theta_t) = Q\theta_t = \mathbf{v}_t$.
  Gradient ascent with learning rate $\beta$ updates as
  \begin{equation}
  \theta_{t+1} = \theta_t + \beta\nabla R(\theta_t) = \theta_t - \beta\,Q\theta_t = (I - \beta Q)\,\theta_t.
  \label{eq:gd_update}
  \end{equation}
  
  Define $H_{\mathrm{GD}} := I - \beta Q$, with eigendecomposition
  \begin{equation}
  H_{\mathrm{GD}} = \sum_{k=1}^{d}(1 - \beta\lambda_k)\,\mathbf{u}_k\mathbf{u}_k^\top,
  \label{eq:gd_eigendecomp}
  \end{equation}
  where $\lambda_k$ and $\mathbf{u}_k$ are the eigenvalues and eigenvectors of $Q$.
  Let $\gamma_k^{\mathrm{GD}} := 1 - \beta\lambda_k$ denote the contraction factor along
  eigendirection $\mathbf{u}_k$.
  
  \paragraph{Trajectory}
  
  Since the iteration is deterministic, recursion gives
  \begin{equation}
  \theta_t = H_{\mathrm{GD}}^t\,\theta_0
  = \sum_{k=1}^{d}(\gamma_k^{\mathrm{GD}})^t\,(\mathbf{u}_k^\top\theta_0)\,\mathbf{u}_k.
  \label{eq:gd_trajectory}
  \end{equation}
  
  Each eigendirection evolves independently:
  \begin{equation}
  \mathbf{u}_k^\top\theta_t = (1 - \beta\lambda_k)^t\,\mathbf{u}_k^\top\theta_0.
  \label{eq:gd_eigen_proj}
  \end{equation}
  
  The squared distance to the optimum $\theta^* = 0$ is
  \begin{equation}
  \|\theta_t\|^2 = \sum_{k=1}^{d}(1-\beta\lambda_k)^{2t}\,(\mathbf{u}_k^\top\theta_0)^2.
  \label{eq:gd_squared_dist}
  \end{equation}
  
  \paragraph{Convergence conditions}
  
  Direction $\mathbf{u}_k$ converges if and only if $|\gamma_k^{\mathrm{GD}}| < 1$, i.e.,
  \begin{equation}
  0 < \lambda_k < \frac{2}{\beta}.
  \label{eq:gd_stability}
  \end{equation}
  The three non-convergent modes mirror those of the ES dynamics:
  \begin{itemize}
  \item \textbf{Flat directions} ($\lambda_k = 0$): $\gamma_k^{\mathrm{GD}} = 1$,
    so $\mathbf{u}_k^\top\theta_t = \mathbf{u}_k^\top\theta_0$ for all $t$, 
    i.e. along flat directions, the gradient ascent freezes at the initial projection.
  \subitem This is dramatically different from ES, where along flat directions, 
    the iterate undergoes unbounded diffusion with variance growing as $\frac{\alpha^2 t}{N}$.
  \item \textbf{Concave directions} ($\lambda_k < 0$): $\gamma_k^{\mathrm{GD}} > 1$,
    causing exponential divergence $|\mathbf{u}_k^\top\theta_t| \to \infty$.
  \item \textbf{Overly sharp convex directions} ($\lambda_k \geq 2/\beta$):
    $|\gamma_k^{\mathrm{GD}}| \geq 1$, and the iterate overshoots, oscillating
    with growing (or non-decaying) amplitude. 
    This explosion can happen due to too large learning rate, or due to too high landscape curvature.
  \end{itemize}
  
  \paragraph{Convergence rate}
  
  Within the stable regime $0 < \lambda_k < 2/\beta$, the convergence timescale along
  $\mathbf{u}_k$ is
  \begin{equation}
  \tau_k^{\mathrm{GD}} = \frac{-\log 2}{\log(\gamma_k^{\mathrm{GD}})^2}
  = \frac{-\log 2}{2\log|1 - \beta\lambda_k|}.
  \label{eq:gd_timescale}
  \end{equation}
  
  The optimal eigenvalue is $\lambda_k = 1/\beta$, which achieves
  $\gamma_k^{\mathrm{GD}} = 0$ and convergence in a single iteration
  ($\tau_k^{\mathrm{GD}} = 0$, since $\mathbf{u}_k^\top\theta_1 = 0$ exactly).
  The overall convergence rate is governed by the slowest direction,
  \begin{equation}
  \tau_{\max}^{\mathrm{GD}} = \max_{k:\,\lambda_k>0}\;\frac{-\log 2}{2\log|1 - \beta\lambda_k|},
  \label{eq:gd_max_timescale}
  \end{equation}
  which for small learning rate $\beta\lambda_k \ll 1$ gives
  $\tau_k^{\mathrm{GD}} \approx \frac{1}{2\beta\lambda_k}$, inversely proportional
  to the curvature. The condition number $\kappa := \lambda_{\max}/\lambda_{\min}$
  (over positive eigenvalues) then controls the ratio of fastest to slowest timescales.
  
  \paragraph{Comparison with ES dynamics}
  
  The simplified ES iteration (Proposition~\ref{prop:quad_landscape_dynamics}) takes the form
  \begin{equation}
  \theta_{t+1} = H_{\mathrm{ES}}\,\theta_t + \eta_t, \qquad
  H_{\mathrm{ES}} = I - \frac{\alpha\sigma}{\sigma_R}Q,
  \label{eq:es_update}
  \end{equation}
  with $\gamma_k^{\mathrm{ES}} = 1 - \frac{\alpha\sigma}{\sigma_R}\lambda_k$.
  The two iterations are structurally identical, with the correspondence
  $\beta \leftrightarrow \frac{\alpha\sigma}{\sigma_R}$, up to the additive
  noise $\eta_t$ in the ES case. This yields three key differences:
  
  \begin{enumerate}
  \item \textbf{Effective learning rate.}
    The ES effective learning rate $\frac{\alpha\sigma}{\sigma_R}$ is
    self-normalized by the reward standard deviation $\sigma_R$, which
    depends on the current iterate through $\|\mathbf{v}\|^2 = \|Q\theta_t\|^2$
    (in the non-simplified version).
    In contrast, the GD learning rate $\beta$ is a fixed hyperparameter.
    This self-normalization provides ES with a degree of automatic step-size
    adaptation absent in vanilla GD.
  
  \item \textbf{Stability threshold.}
    GD diverges when $\lambda_k > 2/\beta$; ES diverges when
    $\lambda_k > 2\sigma_R/(\alpha\sigma)$.
    The ES threshold is typically more permissive since $\sigma_R \geq \sigma_\xi$
    inflates the denominator, though this comes at the cost of a smaller
    effective step size.
  
  \item \textbf{Asymptotic behavior.}
    GD converges to $\theta^* = 0$ exactly along stable directions, with
    zero residual error. ES converges only in distribution, with asymptotic
    variance $\frac{\alpha^2}{N(1-(\gamma_k^{\mathrm{ES}})^2)}$ per direction ---
    a noise floor set by the exploration--exploitation tradeoff that cannot
    be eliminated without increasing $N$.
    On flat directions ($\lambda_k = 0$), GD stays frozen at the initial value,
    while ES undergoes unbounded diffusion with variance growing as
    $\frac{\alpha^2 t}{N}$.
  \end{enumerate}
  
  These comparisons are summarized by:
  \begin{equation}
  \renewcommand{\arraystretch}{1.4}
  \begin{array}{lcc}
  \hline
   & \textbf{Gradient Descent} & \textbf{Evolution Strategy} \\ \hline
  \text{Eff.\ learning rate} & \beta & \alpha\sigma/\sigma_R \\
  \text{Contraction factor } \gamma_k & 1 - \beta\lambda_k & 1 - \frac{\alpha\sigma}{\sigma_R}\lambda_k \\
  \text{Stability bound} & \lambda_k < 2/\beta & \lambda_k < 2\sigma_R/(\alpha\sigma) \\
  \text{Optimal curvature} & \lambda_k = 1/\beta & \lambda_k = \sigma_R/(\alpha\sigma) \\
  \text{Residual error (stable)} & 0 & \frac{\alpha^2}{N(1-\gamma_k^2)} \\
  \text{Flat direction ($\lambda_k=0$)} & \text{frozen} & \text{diffusion } \sim \frac{\alpha^2 t}{N} \\
  \hline
  \end{array}
  \label{eq:gd_vs_es_comparison}
  \end{equation}

  \subsection{Solution geometry: GD vs.\ ES on quadratic landscape}\label{appsec:solution_geometry}

  We formalize the geometric claims of the main text by comparing the GD and ES
  solutions after $T$ steps on the same quadratic landscape
  $R(\theta) = -\frac{1}{2}\theta^\top Q\theta$, starting from the same initialization $\theta_0$.
  Let $Q = \sum_k \lambda_k \mathbf{u}_k\mathbf{u}_k^\top$ with $\lambda_k \geq 0$.
  
  \paragraph{GD solution}
  
  From Section~\ref{appsec:gd_quad_landscape}, the GD iterate is deterministic:
  \begin{equation}
  \theta_{\mathrm{GD}} := \theta_T^{\mathrm{GD}}
  = \sum_{k=1}^{d}(1-\beta\lambda_k)^T\,(\mathbf{u}_k^\top\theta_0)\,\mathbf{u}_k.
  \label{eq:gd_solution}
  \end{equation}
  
  The displacement from initialization is
  \begin{equation}
  \theta_{\mathrm{GD}} - \theta_0
  = \sum_{k=1}^{d}\bigl[(1-\beta\lambda_k)^T - 1\bigr](\mathbf{u}_k^\top\theta_0)\,\mathbf{u}_k.
  \label{eq:gd_disp}
  \end{equation}
  
  Crucially, directions with $\lambda_k = 0$ contribute nothing:
  $(1 - \beta\cdot 0)^T - 1 = 0$. Thus $\theta_{\mathrm{GD}} - \theta_0$ lies
  entirely within the column space of $Q$ (the ``task-relevant'' subspace). If $Q$ has
  rank $r \ll d$, the GD displacement is confined to an $r$-dimensional subspace,
  regardless of the number of steps $T$.
  
  The squared displacement along each active direction is
  \begin{equation}
  |\mathbf{u}_k^\top(\theta_{\mathrm{GD}} - \theta_0)|^2
  = \bigl[1 - (1-\beta\lambda_k)^T\bigr]^2\,(\mathbf{u}_k^\top\theta_0)^2,
  \label{eq:gd_disp_squared}
  \end{equation}
  and for converged directions ($|\gamma_k^{\mathrm{GD}}| \ll 1$, large $T$) this
  saturates to $(\mathbf{u}_k^\top\theta_0)^2$, i.e., the iterate reaches the
  optimum $\mathbf{u}_k^\top\theta^* = 0$.
  
  \paragraph{ES solution}
  
  From Section~\ref{appsec:quad_landscape_dynamics}, the simplified ES iterate is
  \begin{equation}
  \theta_{\mathrm{ES}} := \theta_T^{\mathrm{ES}}
  = H_{\mathrm{ES}}^T\,\theta_0 + \sum_{t=1}^{T}H_{\mathrm{ES}}^{t-1}\,\eta_{T-t},
  \label{eq:es_solution}
  \end{equation}
  with $H_{\mathrm{ES}} = I - \frac{\alpha\sigma}{\sigma_R}Q$ and
  $\eta_t \sim \mathcal{N}(0, \frac{\alpha^2}{N}I_d)$.
  
  The displacement decomposes as
  \begin{equation}
  \theta_{\mathrm{ES}} - \theta_0
  = \underbrace{(H_{\mathrm{ES}}^T - I)\theta_0}_{\text{signal (deterministic)}}
  + \underbrace{\sum_{t=1}^{T}H_{\mathrm{ES}}^{t-1}\,\eta_{T-t}}_{\text{noise (stochastic)}}.
  \label{eq:es_disp_decomp}
  \end{equation}
  
  The signal term has exactly the same structure as the GD displacement, with
  $\beta \leftrightarrow \alpha\sigma/\sigma_R$: it lies within the column space of $Q$
  and vanishes on flat directions. The noise term, however, is isotropic in origin ---
  each $\eta_t$ is drawn from $\mathcal{N}(0, \frac{\alpha^2}{N}I_d)$ --- and thus
  has support in all $d$ dimensions.
  
  Along each eigendirection, the mean and variance of the ES displacement are:
  \begin{align}
  \mathbb{E}[\mathbf{u}_k^\top(\theta_{\mathrm{ES}} - \theta_0)]
  &= \bigl[(\gamma_k^{\mathrm{ES}})^T - 1\bigr]\,\mathbf{u}_k^\top\theta_0, \label{eq:es_disp_mean} \\
  \mathrm{Var}[\mathbf{u}_k^\top(\theta_{\mathrm{ES}} - \theta_0)]
  &= \frac{\alpha^2}{N}\cdot
  \begin{cases}
  \dfrac{1 - (\gamma_k^{\mathrm{ES}})^{2T}}{1 - (\gamma_k^{\mathrm{ES}})^2}, & (\gamma_k^{\mathrm{ES}})^2 \neq 1, \\[6pt]
  T, & (\gamma_k^{\mathrm{ES}})^2 = 1.
  \end{cases}
  \label{eq:es_disp_var}
  \end{align}
  
  For flat directions ($\lambda_k = 0$, $\gamma_k^{\mathrm{ES}} = 1$), the mean displacement
  is zero but the variance grows as $\frac{\alpha^2 T}{N}$ --- a random walk.
  For active directions ($\lambda_k > 0$, $|\gamma_k^{\mathrm{ES}}| < 1$), the mean displacement
  converges to $-\mathbf{u}_k^\top\theta_0$ (same target as GD) while the variance
  saturates to $\frac{\alpha^2}{N(1 - (\gamma_k^{\mathrm{ES}})^2)}$.
  
  \paragraph{Decomposition of squared displacements}
  
  The total expected squared displacement of ES from initialization decomposes as
  \begin{align}
  \mathbb{E}\|\theta_{\mathrm{ES}} - \theta_0\|^2
  &= \sum_{k:\,\lambda_k > 0}
     \Bigl(\bigl[1 - (\gamma_k^{\mathrm{ES}})^T\bigr]^2(\mathbf{u}_k^\top\theta_0)^2
           + \frac{\alpha^2}{N}\cdot\frac{1 - (\gamma_k^{\mathrm{ES}})^{2T}}{1 - (\gamma_k^{\mathrm{ES}})^2}\Bigr) \notag \\
  &\quad+ \sum_{k:\,\lambda_k = 0}
     \frac{\alpha^2 T}{N}.
  \label{eq:es_disp_sq_decomp}
  \end{align}
  
  The first sum runs over the $r$ active directions and contains both signal and noise;
  the second sum runs over the $d - r$ flat directions and is pure noise. For
  $r \ll d$ (low-rank task structure), the flat-direction contribution dominates:
  \begin{equation}
  \mathbb{E}\|\theta_{\mathrm{ES}} - \theta_0\|^2
  \approx \underbrace{\|\theta_{\mathrm{GD}} - \theta_0\|^2}_{\mathcal{O}(r)}
  + \underbrace{\frac{\alpha^2 T d}{N}}_{\text{off-manifold diffusion, }\mathcal{O}(d)}.
  \label{eq:es_disp_sq_approx}
  \end{equation}
  
  By contrast, GD has $\|\theta_{\mathrm{GD}} - \theta_0\|^2 = \mathcal{O}(r)$ with
  no off-manifold contribution.
  
  \paragraph{Geometry of $\theta_{\mathrm{ES}} - \theta_{\mathrm{GD}}$}
  
  The difference between the two solutions is
  \begin{equation}
  \theta_{\mathrm{ES}} - \theta_{\mathrm{GD}}
  = \underbrace{(H_{\mathrm{ES}}^T - H_{\mathrm{GD}}^T)\theta_0}_{\text{signal mismatch}}
  + \underbrace{\sum_{t=1}^{T}H_{\mathrm{ES}}^{t-1}\,\eta_{T-t}}_{\text{ES noise}}.
  \label{eq:es_gd_difference}
  \end{equation}
  
  The signal mismatch term lies within the column space of $Q$ and reflects the
  difference in effective learning rates ($\alpha\sigma/\sigma_R$ vs.\ $\beta$).
  For converged directions where both methods have reached the optimum, this
  term vanishes: $(\gamma_k^{\mathrm{ES}})^T \approx (\gamma_k^{\mathrm{GD}})^T \approx 0$.
  The noise term is identical to the ES noise analyzed above.
  
  Thus, in the large-$T$ regime where both methods have converged on the active directions:
  \begin{align}
  \mathbb{E}\|\theta_{\mathrm{ES}} - \theta_{\mathrm{GD}}\|^2
  &\approx \sum_{k:\,\lambda_k > 0}
     \frac{\alpha^2}{N(1 - (\gamma_k^{\mathrm{ES}})^2)}
     + \sum_{k:\,\lambda_k = 0}\frac{\alpha^2 T}{N} \notag \\
  &\approx \frac{\alpha^2 T(d-r)}{N} + \mathcal{O}(r),
  \label{eq:es_gd_diff_sq}
  \end{align}
  which is dominated by the off-manifold random walk for $r \ll d$.
  
  \paragraph{Summary of geometric structure}
  
  The three pairwise distances satisfy the approximate hierarchy (for $r \ll d$, large $T$):
  \begin{align}
  \|\theta_{\mathrm{GD}} - \theta_0\|^2 &= \mathcal{O}(r), \notag \\
  \mathbb{E}\|\theta_{\mathrm{ES}} - \theta_0\|^2 &= \mathcal{O}(r) + \mathcal{O}(d), \notag \\
  \mathbb{E}\|\theta_{\mathrm{ES}} - \theta_{\mathrm{GD}}\|^2 &= \mathcal{O}(d).
  \label{eq:hierarchy_geom}
  \end{align}
  
  The first displacement is low-dimensional and lies within the column space of $Q$.
  The second is the sum of a comparable on-manifold component and a much larger
  isotropic off-manifold diffusion. The third is dominated by the off-manifold
  component alone, since both methods converge to the same optimum within the
  active subspace. This triangle of displacements accounts for the empirical
  observation that $\theta_{\mathrm{ES}}$ achieves similar task performance to
  $\theta_{\mathrm{GD}}$ (shared on-manifold component) while being far from it
  in parameter space (off-manifold diffusion).
  
  \begin{remark}[Cosine similarity]
  The cosine similarity between the two displacement vectors is
  \begin{equation}
  \cos\angle(\theta_{\mathrm{ES}} - \theta_0,\;\theta_{\mathrm{GD}} - \theta_0)
  = \frac{(\theta_{\mathrm{ES}} - \theta_0)^\top(\theta_{\mathrm{GD}} - \theta_0)}
         {\|\theta_{\mathrm{ES}} - \theta_0\|\,\|\theta_{\mathrm{GD}} - \theta_0\|}.
  \label{eq:cosine_sim}
  \end{equation}
  Since $\theta_{\mathrm{GD}} - \theta_0$ lies within the column space of $Q$
  and the ES noise is zero-mean, the expected inner product picks up only the
  on-manifold signal of ES:
  \begin{equation}
  \mathbb{E}[(\theta_{\mathrm{ES}} - \theta_0)^\top(\theta_{\mathrm{GD}} - \theta_0)]
  = (\theta_{\mathrm{GD}}' - \theta_0)^\top(\theta_{\mathrm{GD}} - \theta_0),
  \label{eq:cosine_sim_inner}
  \end{equation}
  where $\theta_{\mathrm{GD}}'$ denotes GD with learning rate $\alpha\sigma/\sigma_R$
  instead of $\beta$. For large $T$ with both converged on active directions, this
  approaches $\|\theta_{\mathrm{GD}} - \theta_0\|^2 = \mathcal{O}(r)$. Meanwhile,
  $\|\theta_{\mathrm{ES}} - \theta_0\| = \mathcal{O}(\sqrt{d})$, so
  \begin{equation}
  \mathbb{E}\left[\cos\angle(\theta_{\mathrm{ES}} - \theta_0,\;\theta_{\mathrm{GD}} - \theta_0)\right]
  \sim \mathcal{O}\!\left(\sqrt{r/d}\right),
  \label{eq:cosine_sim_expect}
  \end{equation}
  which is small for $r \ll d$: the two displacement vectors are nearly orthogonal
  despite leading to similar task performance.
  \end{remark}

%% file: appendix_weight_drift_RW_stats.tex

\subsection{Testing Random Walk Scaling of Weight Drift in ES and GRPO}
\label{appsec:weight_drift_RW_stats}

Proposition~\ref{prop:flat_landscape_scaling} predicts that, on a flat landscape, the
squared norm of parameter drift grows linearly in the step index $t$:
\[
    \mathbb{E}\!\left[\|\Delta\theta_t\|^2\right] = \frac{\alpha^2\, t\, d}{N}.
\]
This is a parameter-free prediction: $d$ is the number of parameters and $\alpha, N$ are
ES hyperparameters, so the predicted trajectory is fixed before any data are seen. It is
also the null model one would obtain if the landscape were entirely flat and every
parameter irrelevant to the task; deviations from it are therefore informative about
which parts of the network the task actually constrains.

The proposition applies to any set of coordinates evolving under the same update, so we
test it at two granularities: for the parameter vector as a whole, and independently for
each named parameter tensor.

\paragraph{Setup}
We fine-tune two language models of different families and scales on a single reasoning
task:
\begin{itemize}
    \item \textbf{Qwen3-4B-Instruct-2507}, $d = 4{,}022{,}468{,}096$ parameters,
          $36$ transformer blocks, $d_{\text{model}} = 2560$;
    \item \textbf{Llama-3.2-3B}, $d = 3{,}212{,}749{,}824$ parameters,
          $28$ transformer blocks, $d_{\text{model}} = 3072$.
\end{itemize}
Each is trained with two optimizers:
\begin{itemize}
    \item \textbf{Evolution Strategy (ES):} z-scored ES with perturbation standard
          deviation $\sigma = 0.0015$, step size $\alpha = \sigma/2 = 7.5\times10^{-4}$,
          and population size $N = 30$, run for $T = 500$ steps. The run was extended
          from $300$ to $500$ steps for the analysis reported here.
    \item \textbf{GRPO:} Group Relative Policy Optimization, a gradient-based RL
          fine-tuning method, run for $T = 79$ steps.
\end{itemize}
At each step $t$ we record, for every named parameter $\theta^{(k)}$,
\[
    \|\Delta\theta^{(k)}_t\|^2 := \|\theta^{(k)}_t - \theta^{(k)}_0\|^2,
\]
the squared Euclidean distance from initialization.

\paragraph{Whole-model drift vs.\ theory (Figure~\ref{fig:weight_drift})}
Figure~\ref{fig:weight_drift} plots the empirical $\sum_k \|\Delta\theta^{(k)}_t\|^2$
against the predicted curve. ES tracks it closely in both models, with
$R^2 = 0.9943$ (Qwen) and $0.9978$ (Llama) measured against the parameter-free
prediction itself rather than against a fitted line.

The residual is a small error in the \emph{slope}, not a departure from linearity.
Refitting the slope freely --- still through the origin, as the law has no intercept ---
raises $R^2$ to $0.9999$ and $1.0000$ respectively, so the growth is linear to four or
five decimal places. Writing $\rho := d_{\mathrm{eff}}/d$ for the ratio of the fitted to
the predicted slope:

\begin{center}
\begin{tabular}{lccccc}
\toprule
Model & $d$ & predicted slope & fitted slope & $\rho$ & $R^2$ (free fit) \\
\midrule
Qwen  & $4.02\times10^{9}$ & $75.42$ & $72.74$ & $0.964$ & $0.9999$ \\
Llama & $3.21\times10^{9}$ & $60.24$ & $61.66$ & $1.024$ & $1.0000$ \\
\bottomrule
\end{tabular}
\end{center}

The two models bracket the prediction: Qwen's aggregate drift runs $3.6\%$ below it and
Llama's $2.4\%$ above. Both mechanisms available in the theory are consistent with a
deviation of this size --- curvature along task-relevant directions suppresses diffusion
and drives $\rho$ below unity (Prop.~\ref{prop:quad_landscape_dynamics}), while a
systematic gradient-like component adds displacement and drives it above
(Prop.~\ref{prop:quad_landscape_step}) --- so we do not read a single mechanism into the
aggregate scaling residual. What the aggregate does establish is that neither effect is large:
across two architectures the total drift rate is within $4\%$ of a prediction with no fitted
parameters.

Note that these aggregate values are exactly the parameter-count-weighted mean of the per-tensor
$\rho$ reported below, agreeing to four decimal places. 

GRPO behaves qualitatively differently. Its drift saturates within roughly $50$ steps
rather than growing linearly, reaching a peak $\|\Delta\theta\|^2$ of $1.12$ (Qwen) and
$1.72$ (Llama). Compared at the last step both optimizers reached, $T = 79$, ES has
already moved the weights $5.2\times10^{3}$ (Qwen) and $2.9\times10^{3}$ (Llama) times
farther. This is consistent with gradient-based fine-tuning making targeted updates
within a low-dimensional subspace rather than diffusing through parameter space.

\begin{figure}[h]
    \centering
    \includegraphics[width=0.9\linewidth]{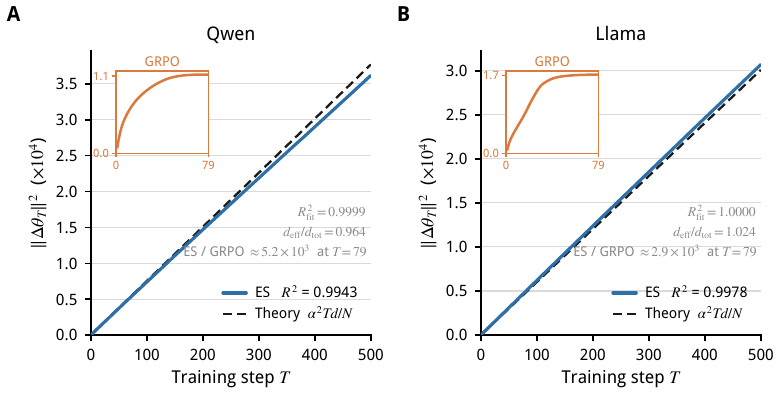}
    \caption{\textbf{Total weight drift under ES follows the random-walk law; under GRPO
        it does not.} Squared displacement of the full parameter vector
        $\|\Delta\theta_T\|^{2}$ against training step $T$ for Qwen (\textbf{A},
        $d = 4.02\times10^{9}$) and Llama (\textbf{B}, $d = 3.21\times10^{9}$). Solid
        blue, ES; dashed black, the prediction $\alpha^{2} T d / N$ with
        $\alpha = \sigma/2$, $\sigma = 0.0015$ and $N = 30$. The $R^{2}$ in the legend is
        measured against that prediction, which has no fitted parameter; the annotation
        gives instead the $R^{2}$ of a freely fitted slope and the effective dimension
        $\rho = d_{\mathrm{eff}}/d_{\mathrm{total}}$ that slope implies. Insets show GRPO
        on its own axes over its full $79$ steps, annotated with the ES/GRPO displacement
        ratio at the matched step $T = 79$. Values and interpretation in the text.}
    \label{fig:weight_drift}
\end{figure}

\paragraph{Per-parameter effective dimension}
\label{sec:per_param}
To characterize which parameters drive the drift, we fit a no-intercept linear regression
\[
    \|\Delta\theta^{(k)}_t\|^2 \approx s^{(k)}\cdot t
\]
for each parameter $k$ independently, obtaining a slope $s^{(k)}$ and a goodness of fit
$R^2_k$. Because the fit has no intercept, $R^2_k$ is not bounded below by zero: a tensor
whose drift is poorly described by \emph{any} line through the origin will receive a
negative value, and several families below do. Inverting
Proposition~\ref{prop:flat_landscape_scaling} gives the effective random walk dimension
\[
    d^{(k)}_{\mathrm{eff}} := \frac{s^{(k)}\, N}{\alpha^2} = \frac{4\, s^{(k)}\, N}{\sigma^2},
    \qquad
    \rho^{(k)} := \frac{d^{(k)}_{\mathrm{eff}}}{d^{(k)}},
\]
where $d^{(k)}$ is the number of elements in parameter $k$. A value
$\rho^{(k)} \approx 1$ indicates the parameter diffuses isotropically through its full
space as predicted; $\rho^{(k)} \ll 1$ indicates constrained diffusion; $\rho^{(k)} > 1$
indicates excess displacement relative to dimension, potentially from a gradient-like
drift component (Prop.~\ref{prop:quad_landscape_step}).

One tensor is excluded throughout. Qwen's
\texttt{model.layers.34.input\_layernorm.weight} has exactly zero displacement at every
recorded step, leaving $s^{(k)}$, $\rho^{(k)}$ and $R^2_k$ undefined; it is dropped from
all statistics and figures ($n = 35$ rather than $36$ for that family).

\paragraph{Large weight matrices}
The matrix-shaped parameters --- embedding, \texttt{qkv\_proj}, \texttt{o\_proj},
\texttt{gate\_up\_proj}, \texttt{down\_proj}; $145$ tensors in Qwen and $113$ in Llama ---
behave as near-isotropic random walkers in both models:
\begin{center}
\begin{tabular}{lccc}
\toprule
Model & $n$ & $\rho$ (mean $\pm$ SD) & $R^2$ (mean $\pm$ SD) \\
\midrule
Qwen  & $145$ & $0.968 \pm 0.014$ & $0.999888 \pm 0.000014$ \\
Llama & $113$ & $1.024 \pm 0.031$ & $0.999968 \pm 0.000022$ \\
\bottomrule
\end{tabular}
\end{center}
Every such tensor in either model lies within $10\%$ of $\rho = 1$ and above
$R^2 = 0.9998$, directly confirming the flat-landscape regime of
Proposition~\ref{prop:flat_landscape_scaling} for the bulk of the parameter count.

The small residual is nonetheless systematic rather than noisy, and it is what
distinguishes the two models: $75\%$ of Llama's large tensors exceed $\rho = 1$, reaching
$1.07$ for \texttt{o\_proj}, whereas every one of Qwen's $145$ falls below it, with a
maximum of $0.999$. Llama also fits the law about four times more precisely
(median $R^2 = 0.999975$ against $0.999890$). We report this as an observation; the
figures do not isolate its cause.

\paragraph{Norm and layer-norm parameters}
The remaining parameters are RMSNorm gain vectors ($144$ tensors in Qwen, $57$ in Llama).
Both models place one before the attention sublayer of each block
(\texttt{input\_layernorm}) and one before the MLP sublayer --- the latter named
\texttt{post\_attention\_layernorm} for its position in the block rather than because it
normalizes the attention output --- plus a single output norm before the language-model
head. Each is a per-channel gain of length $d_{\text{model}}$, with no bias and no
centering. Qwen adds QK-norms, applied within each attention head to queries and keys;
their gain is of length $d_{\text{head}} = 128$ and is shared across all heads, so these
are the smallest tensors in either model by three to five orders of magnitude.

These parameters depart from the law substantially, and --- unlike the matrices --- they
depart differently in the two architectures:
\begin{center}
\begin{tabular}{llccc}
\toprule
Family & Model & $n$ & $\rho$ (mean $\pm$ SD) & $R^2$ (mean $\pm$ SD) \\
\midrule
\texttt{input\_layernorm} & Qwen  & $35$ & $1.00 \pm 0.93$ & $-0.38 \pm 1.31$ \\
                          & Llama & $28$ & $0.37 \pm 0.15$ & $-2.69 \pm 1.22$ \\
\texttt{post\_attention\_layernorm} & Qwen  & $36$ & $1.55 \pm 0.79$ & $-0.98 \pm 1.24$ \\
                          & Llama & $28$ & $0.34 \pm 0.25$ & $-3.11 \pm 1.10$ \\
\texttt{q\_norm}          & Qwen  & $36$ & $2.46 \pm 0.33$ & $\phantom{-}0.81 \pm 0.12$ \\
\texttt{k\_norm}          & Qwen  & $36$ & $2.00 \pm 0.50$ & $\phantom{-}0.70 \pm 0.23$ \\
\texttt{model.norm}       & Qwen  & $1$  & $0.10$ & $\phantom{-}0.83$ \\
                          & Llama & $1$  & $2.99$ & $\phantom{-}0.89$ \\
\bottomrule
\end{tabular}
\end{center}

Three observations. First, the strongly negative $R^2$ values for the per-block
LayerNorms are not scatter around a line but evidence of the wrong functional form:
Figure~\ref{fig:norm_curve_per_param} shows the underlying trajectories, which rise
steeply and then saturate. Llama's pass $90\%$ of their final displacement by
$T \approx 130$, against $T = 450$ for a straight line, and end well below the predicted
trajectory. A saturating curve, not a random walk, is the appropriate description -- which is consistent with the theory describing an OU process under quadratic potential (Eq.~\ref{eq:simplify_ES_theory_norm_scaling}). 

Second, the sign of the deviation is architecture-dependent. Llama's per-block
LayerNorms diffuse in roughly a third of their nominal dimensions ($\rho \approx 0.35$),
while Qwen's sit at or above unity ($\rho = 1.00$ and $1.55$), with correspondingly less
negative $R^2$. The single output norm flips outright between models, $\rho = 2.99$ in
Llama against $0.10$ in Qwen.

Third, Qwen's QK-norms show a consistent excess, $\rho \approx 2.0$--$2.5$ with
$R^2 \approx 0.7$--$0.8$. At $d=128$ entries they are also where the concentration
underlying $\alpha^2 T d_p / N$ is weakest. Their excess drift may reflect a genuinely
higher signal-to-noise ratio in these low-dimensional subspaces, but that caveat should
be kept in view.

\begin{figure}[h]
    \centering
    \includegraphics[width=0.95\linewidth]{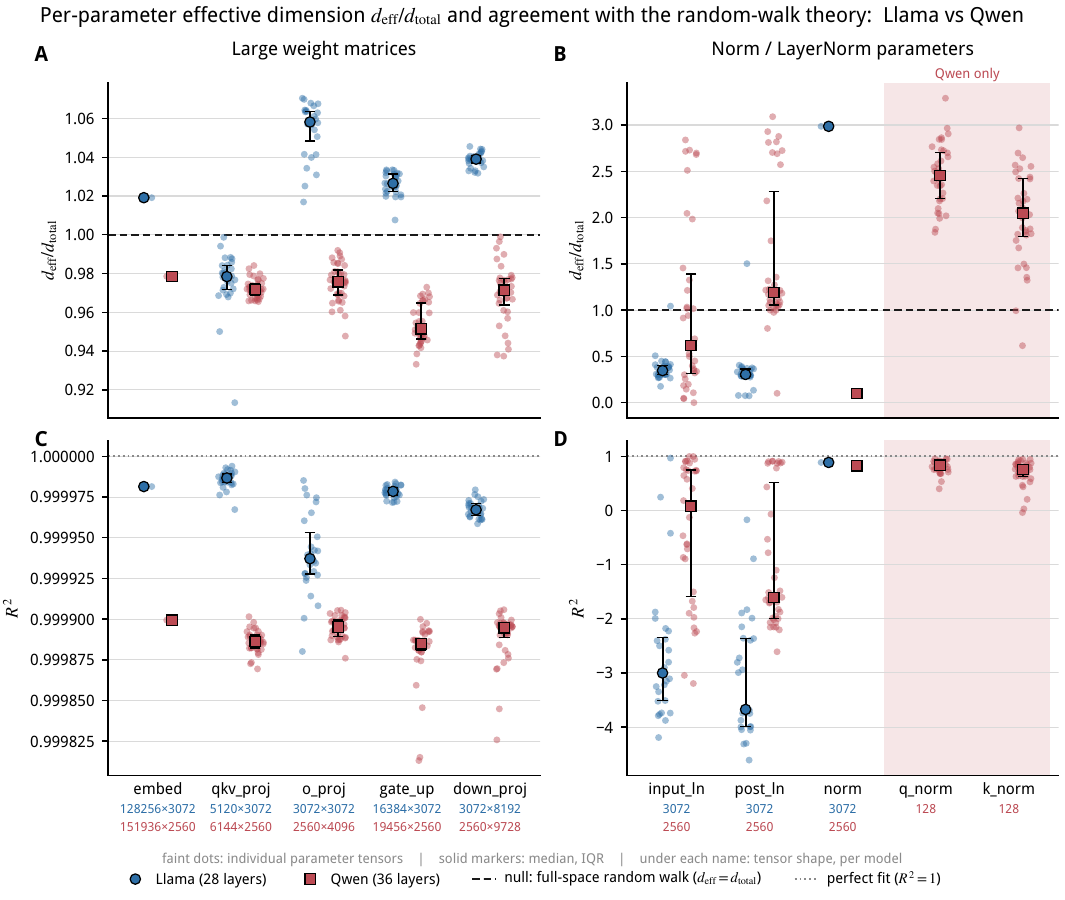}
    \caption{\textbf{Per-parameter effective dimension and agreement with the random-walk
        theory, Llama vs Qwen.} For every parameter tensor we fit
        $\|\Delta\theta_T\|^{2} = \alpha^{2} T d_{\mathrm{eff}} / N$ through the origin
        over $500$ ES steps and report $\rho = d_{\mathrm{eff}}/d_{\mathrm{total}}$
        (\textbf{A}, \textbf{B}) and the $R^{2}$ of that fit (\textbf{C}, \textbf{D}),
        for the large weight matrices (left) and the RMSNorm gain vectors (right). Faint
        dots are individual tensors --- $28$ per per-layer family in Llama, $36$ in Qwen
        --- and solid markers give the median with the interquartile range. Two reference
        levels are drawn, distinguished by colour: the dashed black line is the null
        model in which a tensor random-walks through its full parameter space,
        $\rho = 1$; the dotted grey line is a perfect fit, $R^{2} = 1$. Under each name
        is the tensor's shape in each model (blue, Llama; red, Qwen); shading marks the
        QK-norm families, present only in Qwen. The fit has no intercept, so $R^{2}$ is
        free to be negative, as it is throughout much of \textbf{D}; the axis in
        \textbf{C} is correspondingly expanded, spanning only the top $2\times10^{-4}$
        of the $R^{2}$ range. The large matrices cluster tightly at $\rho \approx 1$
        with $R^{2}$ indistinguishable from unity at ordinary resolution, while the norm
        parameters scatter across the full range of both quantities; the weight-growth
        curves underlying the latter are shown in
        Figure~\ref{fig:norm_curve_per_param}. Qwen's
        \texttt{layers.34.input\_layernorm} is excluded throughout ($n = 35$ for that
        family; see text).}
    \label{fig:d_ratio_per_param}
\end{figure}

\begin{figure}[h]
    \centering
    \includegraphics[width=0.95\linewidth]{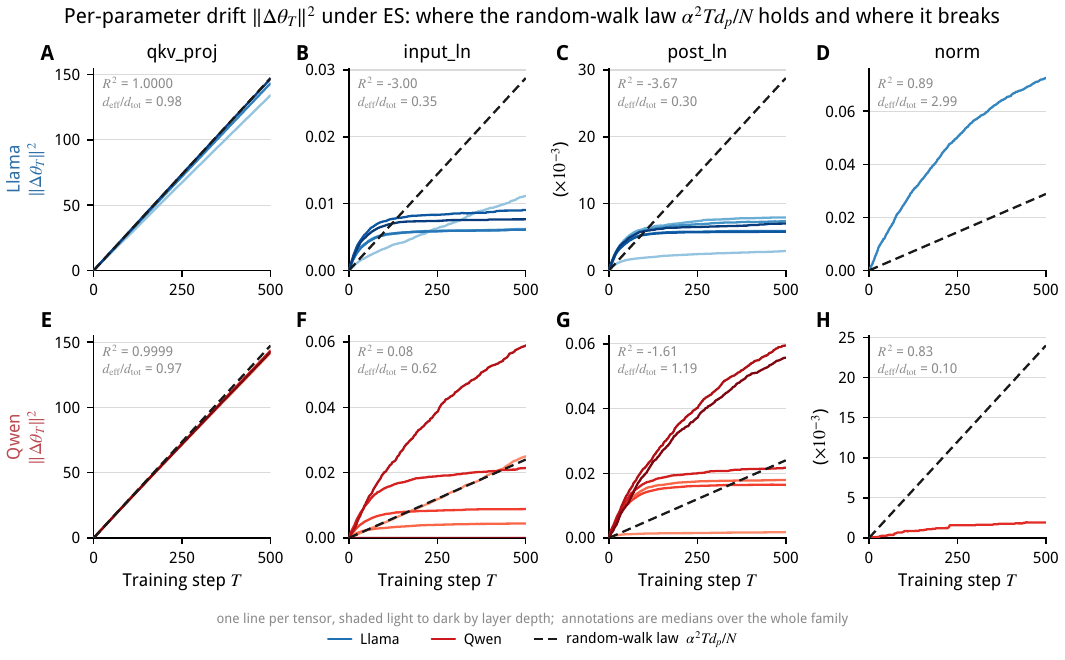}
    \caption{\textbf{Per-parameter weight drift under ES: where the random-walk law holds
        and where it breaks.} Squared displacement $\|\Delta\theta_T\|^2$ against ES step
        $T$ for individual parameter tensors of Llama (\textbf{A}--\textbf{D}, blue) and
        Qwen (\textbf{E}--\textbf{H}, red), one parameter family per column. Dashed lines
        give the random-walk prediction $\alpha^{2} T d_p / N$, where $d_p$ is the number
        of entries in the tensor ($\alpha = \sigma/2$, $\sigma = 0.0015$, $N = 30$); it is
        common to every curve in a panel, since all tensors in a family share a shape.
        Curves are individual tensors shaded light to dark by layer depth --- three
        evenly spaced layers for \texttt{qkv\_proj}, six for \texttt{input\_ln} and
        \texttt{post\_ln}, and the single final norm for \texttt{norm}. Annotations give
        the median $R^{2}$ and median $\rho = d_{\mathrm{eff}}/d_{\mathrm{total}}$ over
        the \emph{entire} family, not only the curves drawn; $R^{2}$ is that of a fit
        through the origin, as the law has no intercept, and may therefore be negative.
        The large weight matrices (\textbf{A}, \textbf{E}) lie on the predicted line for
        the whole run, while the RMSNorm gains rise steeply and then saturate --- the
        departure that the negative $R^{2}$ of a through-origin fit records. Values in
        the text.}
    \label{fig:norm_curve_per_param}
\end{figure}

\paragraph{Summary}
The random-walk law describes ES weight drift accurately where the theory's assumptions
are best met and fails in an interpretable way where they are not. At the whole-model
level it predicts total displacement across two architectures to within $4\%$ with no
fitted parameters. At the tensor level it holds to within $10\%$ in $\rho$ and
$2\times10^{-4}$ in $R^2$ for every large weight matrix --- which is to say, for
essentially the entire parameter count. It fails for the RMSNorm gain vectors, which are
three to five orders of magnitude smaller, and the failed fit is due to weight change  saturating: those coordinates stop moving, which is what one expects when a quadratic-like landscape constrains the random walk (Sec.~\ref{appsec:quad_landscape_dynamics}).

%% file: refs.bib
@article{antognini2018pca_highdim_random_walk,
  title={PCA of high dimensional random walks with comparison to neural network training},
  author={Antognini, Joseph and Sohl-Dickstein, Jascha},
  journal={Advances in Neural Information Processing Systems},
  volume={31},
  year={2018}
}

@article{wang2022tuninglandscape,
  title={Tuning landscapes of the ventral stream},
  author={Wang, Binxu and Ponce, Carlos R},
  journal={Cell reports},
  volume={41},
  number={6},
  year={2022},
  publisher={Elsevier}
}

@article{wang2022highperf_ES,
  title={High-performance evolutionary algorithms for online neuron control},
  author={Wang, Binxu and Ponce, Carlos R},
  journal={Proceedings of the genetic and evolutionary computation conference},
  pages={1308--1316},
  year={2022}
}

@article{chen2025onpolicy,
  title={Retaining by doing: The role of on-policy data in mitigating forgetting},
  author={Chen, Howard and Razin, Noam and Narasimhan, Karthik and Chen, Danqi},
  journal={arXiv preprint arXiv:2510.18874},
  year={2025}
}

@article{abdi2026es_forgetting,
  title={Evolutionary Strategies lead to Catastrophic Forgetting in LLMs},
  author={Abdi, Immanuel and Gupta, Akshat and Mok, Micah and Lu, Alexander and Lee, Nicholas and Anumanchipalli, Gopala},
  journal={arXiv preprint arXiv:2601.20861},
  year={2026}
}

@article{aghajanyan2021intrinsic,
  title={Intrinsic dimensionality explains the effectiveness of language model fine-tuning},
  author={Aghajanyan, Armen and Gupta, Sonal and Zettlemoyer, Luke},
  booktitle={Proceedings of the 59th annual meeting of the association for computational linguistics and the 11th international joint conference on natural language processing (volume 1: long papers)},
  pages={7319--7328},
  year={2021}
}

@article{deepseek2025r1,
  title={Deepseek-r1: Incentivizing reasoning capability in llms via reinforcement learning},
  author={Guo, Daya and Yang, Dejian and Zhang, Haowei and Song, Junxiao and Wang, Peiyi and Zhu, Qihao and Xu, Runxin and Zhang, Ruoyu and Ma, Shirong and Bi, Xiao and others},
  journal={arXiv preprint arXiv:2501.12948},
  year={2025}
}

@article{korotyshova2025essa,
  title={ESSA: Evolutionary Strategies for Scalable Alignment},
  author={Korotyshova, Daria and Shaposhnikov, Boris and Malakhov, Alexey and Khokhulin, Alexey and Surnachev, Nikita and Ovcharenko, Kirill and Bredis, George and Gorbatovski, Alexey and Sinii, Viacheslav and Gavrilov, Daniil},
  journal={arXiv preprint arXiv:2507.04453},
  year={2025}
}

@inproceedings{lehman2018es,
  title={ES is more than just a traditional finite-difference approximator},
  author={Lehman, Joel and Chen, Jay and Clune, Jeff and Stanley, Kenneth O},
  booktitle={Proceedings of the genetic and evolutionary computation conference},
  pages={450--457},
  year={2018}
}

@article{liang2026blessing,
  title={The Blessing of Dimensionality in LLM Fine-tuning: A Variance-Curvature Perspective},
  author={Liang, Qiyao and Song, Jinyeop and Liu, Yizhou and Gore, Jeff and Fiete, Ila and Miikkulainen, Risto and Qiu, Xin},
  journal={arXiv preprint arXiv:2602.00170},
  year={2026}
}

@article{qiu2025es,
  title={Evolution strategies at scale: Llm fine-tuning beyond reinforcement learning},
  author={Qiu, Xin and Gan, Yulu and Hayes, Conor F and Liang, Qiyao and Xu, Yinggan and Dailey, Roberto and Meyerson, Elliot and Hodjat, Babak and Miikkulainen, Risto},
  journal={arXiv preprint arXiv:2509.24372},
  year={2025}
}

@article{salimans2017es,
  title={Evolution strategies as a scalable alternative to reinforcement learning},
  author={Salimans, Tim and Ho, Jonathan and Chen, Xi and Sidor, Szymon and Sutskever, Ilya},
  journal={arXiv preprint arXiv:1703.03864},
  year={2017}
}

@article{sarkar2025hyperscale,
  title={Evolution strategies at the hyperscale},
  author={Sarkar, Bidipta and Fellows, Mattie and Duque, Juan Agustin and Letcher, Alistair and Villares, Antonio Le{\'o}n and Sims, Anya and Wibault, Clarisse and Samsonov, Dmitry and Cope, Dylan and Liesen, Jarek and others},
  journal={arXiv preprint arXiv:2511.16652},
  year={2025}
}

@article{shao2024grpo,
  title={Deepseekmath: Pushing the limits of mathematical reasoning in open language models},
  author={Shao, Zhihong and Wang, Peiyi and Zhu, Qihao and Xu, Runxin and Song, Junxiao and Bi, Xiao and Zhang, Haowei and Zhang, Mingchuan and Li, YK and others},
  journal={arXiv preprint arXiv:2402.03300},
  year={2024}
}

@article{shenfeld2025rl,
  title={Rl's razor: Why online reinforcement learning forgets less},
  author={Shenfeld, Idan and Pari, Jyothish and Agrawal, Pulkit},
  journal={arXiv preprint arXiv:2509.04259},
  year={2025}
}

@article{Moore2018HighEvolution,
    title = {{High dimensional random walks can appear low dimensional: Application to influenza H3N2 evolution}},
    year = {2018},
    journal = {Journal of Theoretical Biology},
    author = {Moore, James and Ahmed, Hasan and Antia, Rustom},
    month = {6},
    pages = {56--64},
    volume = {447},
    publisher = {Academic Press},
    url = {https://www.sciencedirect.com/science/article/pii/S002251931830136X},
    doi = {10.1016/J.JTBI.2018.03.022},
    issn = {0022-5193}
}

@article{isserlis1918formula,
  author  = {Isserlis, L.},
  title   = {On a formula for the product-moment coefficient of any order
             of a normal frequency distribution},
  journal = {Biometrika},
  year    = {1918},
  volume  = {12},
  number  = {1--2},
  pages   = {134--139},
  doi     = {10.1093/biomet/12.1-2.134}
}

@article{wick1950evaluation,
  author  = {Wick, G. C.},
  title   = {The Evaluation of the Collision Matrix},
  journal = {Physical Review},
  year    = {1950},
  volume  = {80},
  number  = {2},
  pages   = {268--272},
  doi     = {10.1103/PhysRev.80.268}
}
